%% file: main.tex
\PassOptionsToPackage{table}{xcolor}

\documentclass[acmsmall, screen]{acmart}

\pdfoutput=1

\AtBeginDocument{%
  \providecommand\BibTeX{{%
    \normalfont B\kern-0.5em{\scshape i\kern-0.25em b}\kern-0.8em\TeX}}}

\acmJournal{TECS}

\input{packages_tecs_compatible}
\input{macros_tecs_compatible}

\begin{document}

\title{Verifying Quantized Neural Networks using SMT-Based Model Checking}

\author{Luiz Sena}
\affiliation{%
    \institution{Federal University of Amazonas}
    \city{Manaus}
    \state{AM}
    \country{Brazil}
}
\email{lhcs@icomp.ufam.edu.br}

\author{Xidan Song}
\affiliation{%
    \institution{University of Manchester}
    \city{Manchester}
    \country{United Kingdom}
}
\email{xidan.song@postgrad.manchester.ac.uk}

\author{Erickson Alves}
\affiliation{%
    \institution{Federal University of Amazonas}
    \city{Manaus}
    \state{AM}
    \country{Brazil}
}
\email{erickson@icomp.ufam.edu.br}

\author{Iury Bessa}
\affiliation{%
    \institution{Federal University of Amazonas}
    \city{Manaus}
    \state{AM}
    \country{Brazil}
}
\email{iurybessa@ufam.edu.br}

\author{Edoardo Manino}
\affiliation{%
    \institution{University of Manchester}
    \city{Manchester}
    \country{United Kingdom}
}
\email{edoardo.manino@manchester.ac.uk}

\author{Lucas Cordeiro}
\affiliation{%
    \institution{University of Manchester}
    \city{Manchester}
    \country{United Kingdom}
}
\author{Eddie de Lima Filho}
\affiliation{%
    \institution{Federal University of Amazonas and TPV Technology}
    \city{Manaus}
    \country{Brazil}
}
\email{eddie.filho@tpv-tech.com}

\begin{abstract}
Artificial Neural Networks (ANNs) are being deployed for an increasing number of safety-critical applications, including autonomous cars and medical diagnosis. However, concerns about their reliability have been raised due to their black-box nature and apparent fragility to adversarial attacks. These concerns are amplified when ANNs are deployed on restricted system, which limit the precision of mathematical operations and thus introduce additional quantization errors. Here, we develop and evaluate a novel symbolic verification framework using software model checking (SMC) and satisfiability modulo theories (SMT) to check for vulnerabilities in ANNs. More specifically, we propose several ANN-related optimizations for SMC, including invariant inference via interval analysis, slicing, expression simplifications, and discretization of non-linear activation functions. With this verification framework, we can provide formal guarantees on the safe behavior of ANNs implemented both in floating- and fixed-point arithmetic. In this regard, our verification approach was able to verify and produce adversarial examples for $52$ test cases spanning image classification and general machine learning applications. Furthermore, for small- to medium-sized ANN, our approach completes most of its verification runs in minutes. Moreover, in contrast to most state-of-the-art methods, our approach is not restricted to specific choices regarding activation functions and non-quantized representations. Our experiments show that our approach can analyze larger ANN implementations and substantially reduce the verification time compared to state-of-the-art techniques that use SMT solving.
\end{abstract}


\maketitle

\section{Introduction}
\label{sec:intro}

Artificial neural networks (ANNs) are soft computing models usually employed for regression, machine learning, decision-making, and pattern recognition problems~\cite{bishop2006PRML}, which 
have been recently used to perform various safety-critical tasks. For instance,  \ann{}s are employed for Covid-19 diagnosis~\cite{Nour2020}, and for performing steering commands in self-driving cars~\cite{Wu2021}. Unfortunately, in such contexts, incorrect classifications can cause serious problems. Indeed, adversarial disturbances can make ANNs misclassify objects\Ignore{simple patterns}, thus causing severe damage to users of safety-critical systems. For instance, Eykholt~\etal{}~\cite{DBLP:conf/cvpr/EykholtEF0RXPKS18} showed that noise and disturbances, such as graffiti on traffic signals, could result in target misclassification during the operation of computer vision systems. Moreover, given that ANNs are notorious for being difficult to interpret and debug, the whole scenario becomes even more problematic~\cite{LundbergL17}, which then claims for techniques able to assess their structures and verify results and behaviors. For this reason, there is a growing interest in verification methods for ensuring safety, accuracy, and robustness for neural networks. The approaches for ANN verification may be divided into three groups: optimization~\cite{Fazlyab2020,Rossig2020,Venzke2021,Huang2020}, reachability~\cite{Huang2019,Xiang2018,Tran2020,Ivanov2021,Katz2019,Xiang2018,Wang2018}, and satisfiability~\cite{Narodytska2018,huang2017safety,katz2017reluplex,Henzinger2020}. 

On the one hand, optimization-based algorithms pose the safety verification problem as an optimization one, in which safety properties are usually treated as constraints, as described by Tjeng {\it et al.}~\cite{Tjeng2018}. The main difficulty of optimization methods, such as mixed-integer linear programming~\cite{Botoeva2020,Tjeng2018,Venzke2021}, branch and bound~\cite{Rossig2020}, and semi-definite programming~\cite{Fazlyab2020}, is to deal with constraints that are non-linear and non-convex due to a network's complex structure and its activation functions. Indeed, it is still possible to employ dual optimization for simplifying those constraints and then obtain a convex problem~\cite{Dvijotham2018}; however, completeness tends to be lost due to relaxations. On the other hand, reachability-based approaches aim at computing the reachable set of an ANN by propagating input sets through it, layer-by-layer, while checking whether some unsafe state (violation) belongs or not to that same reachable set. The main advantage of those methods is that they are usually sound, i.e., if the algorithm indicates that a network is unsafe, its safety property is violated. However, the computational cost to compute exact reachable sets becomes unreasonable for more complex ANNs and more extensive input spaces. In order to avoid such a problem, a reachable set is over-approximated by using symbolic~\cite{Katz2019,katz2017reluplex,Wang2018} and/or set-theoretic methods~\cite{Xiang2018,Tran2020}. Although those tools effectively reduce the computational cost of reachability sets, it is still challenging to over-approximate ANN's non-linear elements, particularly their activation functions. There are some symbolic techniques suitable for dealing with over-approximation of activation functions~\cite{Ivanov2021,Huang2019}; however, most of the approaches available in literature are only able to approximate piecewise-linear and rectified linear unit (ReLU) activation functions.

Finally, satisfiability modulo theories (SMT) encode both ANN and desired safety property into a single logic formula using a decidable fragment of first-order logic, and then check whether a counterexample exists. In this regard, only binarized neural networks~\cite{Rastegari2016,Hubara2016} can be encoded into boolean logic and verified with existing SAT solvers~\cite{Cheng2018,Narodytska2018}. More complex ANNs, whether implemented in floating- or fixed-point~\cite{Kim2015,Lin2016}, the latter aiming at efficiency and simplicity, require the use of first-order logic instead of propositional logic to exploit more abstract and less expensive techniques to solve the problem at hand. For example, SMT solvers often integrate a simplifier, which applies standard algebraic reduction rules and contextual simplification to simplify the logical formula. Regarding these, several SMT-based approaches have been proposed~\cite{pulina2012challenging,huang2017safety,katz2017reluplex,Katz2019,Baranowski2020,Giacobbe2020,Henzinger2020}. While SMT background theories allow those approaches to model the semantic of neural operations exactly using word-level theories, the resulting verification problem is challenging to solve~\cite{pulina2012challenging}. In this respect, quantization, i.e., a representation with a lower number of bits, has been proven to make this problem even computationally harder~\cite{Henzinger2020}. As a consequence, most existing approaches specialize in simple piecewise-linear activation functions~\cite{katz2017reluplex,Baranowski2020,Giacobbe2020}, focus on the floating-point scenario only~\cite{Katz2019}, or require domain-specific abstractions~\cite{huang2017safety}.

Against this background, we propose a novel approach to verify both fixed- and floating-point ANN implementations. Our main idea is to look at the source code of an ANN rather than the abstract mathematical model behind it. By doing so, we can then leverage many recent advances in software verification that can dramatically increase the computational efficiency of verification processes, as observed in our experimental evaluation. More specifically, in this paper, we make the following original contributions:
\begin{itemize}
    \item We cast the ANN verification problem into a software verification one. On the one hand, we propose a method to represent ANN safety properties as pairs of \texttt{assume} and \texttt{assert} instructions. On the other hand, we explain how to represent fixed- and floating-point operations in a quantized ANN, using direct implementations of their behavior, i.e., representations that consider a target precision.
    \item We introduce several pre-processing steps to increase the efficiency of downstream software verification tools. Namely, we give a principled method to discretize non-linear activation functions and replace them with lookup tables. Furthermore, we show how to bound the feasible range of each variable with interval analysis and how to represent those bounds with additional \texttt{assume} instructions.
    \item We detail which existing techniques for search-space reduction can be borrowed from the software verification literature, and we empirically evaluate their individual and cumulative effects.
    \item We evaluate our approach on fixed- and floating-point ANNs and give empirical evidence on its computational efficiency. In particular, we show that we can verify ANNs with hundreds of neurons in less than an hour.
    \item We compare our approach with state-of-the-art (SOTA) techniques, including quantized and floating-point tools. According to the comparison, since our method applies various optimization techniques before invoking the SMT solver, we have better performance than other SMT-based verification tools.
\end{itemize}

\textit{Outline}. In Section~\ref{sec:preliminaries}, we introduce the ANN verification problem and present existing satisfiability modulo theories. In Section~\ref{sec:verification}, we detail all the steps involved in our code-level verification approach for ANNs. In Section~\ref{sec:exp}, we empirically test our approach on ANN classifiers trained on the classic Iris dataset and an image recognition dataset. In Section~\ref{sec:related_work}, we give a broader review of the recent trends in verifying ANNs. In Section~\ref{sec:conclusion}, we conclude and outline possible future work.

\section{Preliminaries}
\label{sec:preliminaries}

Before introducing the details of our verification approach, let us review some important concepts related to the verification of artificial neural networks.

\subsection{Artificial Neural Networks (ANNs)}
\label{ssec:ann}

Modern ANNs are universal function approximators built by composing multiple copies of the same basic building block, called neuron~\cite{bishop2006PRML}. In other words, they provide a way of constructing system models with a set of sample observations, in such a way that the joint behavior of existing neurons is correctly adjusted. In their most common form, each neuron $k$ is itself the composition of two functions, as illustrated in Fig.~\ref{fig:neuron}. The first one is an affine projection of the $m$ local inputs, often referred to as the \textit{activation potential} $u_k$. The second one is a non-linear transformation of the resulting potential, often referred to as \textit{activation function} $\mathcal{N}_k$. Together, they define the following mapping $n_k:\mathbbm{R}^m\to\mathbbm{R}$:

\begin{equation}
\label{eq:anncalc_active}
    y_k = \mathcal{N}_k(u_k),
\end{equation}
\noindent where
\begin{equation}
\label{eq:anncalc_aff}
    u_k\left(x\right) = \sum_{j=1}^{m}{w_{j,k} x_{j}} + b_k.
\end{equation}

\noindent Finally, $b_{k}$ provides a way of directly shifting a given activation function.

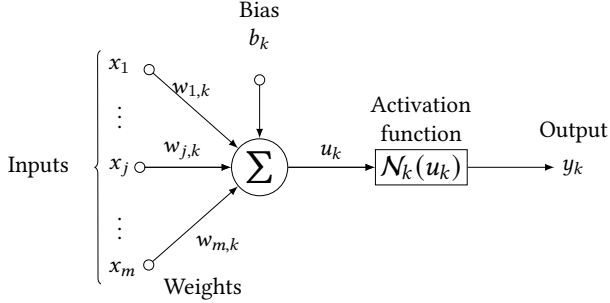
\begin{figure}[htb]
	\resizebox{0.6\linewidth}{!}{
		\begin{tikzpicture}[
		init/.style={
			draw,
			circle,
			inner sep=2pt,
			font=\Huge,
			join = by -latex
		},
		squa/.style={
			draw,
			inner sep=2pt,
			font=\Large,
			join = by -latex
		},
		start chain=2,node distance=13mm
		]
		\node[on chain=2] 
		(x2) {$x_j~$};
		\node[on chain=2,init] (sigma) 
		{$\displaystyle\Sigma$};
		\node[on chain=2,squa,label=above:{\parbox{2cm}{\centering Activation \\ function}}] (actfun)  
		{$\mathcal{N}_k(u_k)$};
		\node[on chain=2,label=above:Output,join=by -latex] 
		{$y_k$};
		\begin{scope}[start chain=1]
		\node[on chain=1] at (0,1.5cm) 
		(x1) {$x_1~$};
		\end{scope}
		\begin{scope}[start chain=3]
		\node[on chain=3] at (0,-1.5cm) 
		(x3) {$x_m$};
		\end{scope}
		
		\node[label=above:\parbox{2cm}{\centering Bias \\ $b_k$}] at (sigma|-x1) (b) {};
		\node[above of = x2, node distance = 0.85cm] (dots) {$\vdots~$};
		\node[below of = x2, node distance = 0.8cm] (dots) {$\vdots~$};
		\draw[o-latex] (x1.east) -- node[above] {$w_{1,k}$} (sigma);
		\draw[o-latex] (x2.east) + (-0.15,0) -- node[above] {$w_{j,k}$} (sigma);
		\draw[o-latex] (x3.east) -- node[below,xshift=0.4cm] {$w_{m,k}$} (sigma);
		\draw[o-latex] (b) -- (sigma);
		\draw[-latex] (sigma) -- node[above] {$u_k$} (actfun);
		\node[below of = x3, xshift = 1.2cm, yshift = 1cm] (weights) {Weights}; 
		\draw[decorate,decoration={brace,mirror}] (x1.north west) -- node[left=10pt] {Inputs} (x3.south west);
		\end{tikzpicture}
	}
	\caption{The detailed view of a single neuron $n_k$.}
	\label{fig:neuron}
\end{figure}

The behavior of the basic neuron in Fig.~\ref{fig:neuron} depends on the values of its weights $w_k$ and also on the chosen activation function $\mathcal{N}_k$. In this regard, researchers have experimented with a wide range of functions, including non-monotonic~\cite{Parascandolo2016,Sitzmann2020}, non-continuous~\cite{bishop2006PRML}, and unbounded ones~\cite{Nair2010,Hendrycks2020}. In our experiments, which are available in Section \ref{sec:exp}, we cover the most popular activation functions: namely, ReLU, sigmoid (Sigm), and the re-scaled version of the latter known as hyperbolic tangent (TanH):
\begin{align}
\label{eq:neuronout_relu}
    \mathcal{N}_{\mathrm{ReLU}}(u_k) &= \max(0,u_k)
    \\
\label{eq:neuronout}
    \mathcal{N}_{\mathrm{Sigm}}(u_k) &= \big(\mathrm{1} + e^{-u_k}\big)^{-1}
    \\
\label{eq:neuronout_tanh}
    \mathcal{N}_{\mathrm{TanH}}(u_k) &= 2\mathcal{N}_{\mathrm{Sigm}}(2u_k)-1.
\end{align}

At the same time, one may notice that many state-of-the-art verification tools for ANNs are only compatible with ReLU and similar piece-wise linear activation functions~\cite{katz2017reluplex,Baranowski2020,Giacobbe2020}. Moreover, those that do support more activation functions~\cite{Ivanov2021,Huang2019,Katz2019} often incur a significant performance hit, when solving the resulting non-linear verification problem. In contrast, the discretization technique we propose in Section \ref{ssec:discretise} allows us to efficiently verify ANNs with any form of activation function.

Besides, our verification methodology is general enough to be applied to a large variety of ANN architectures. Specifically, we support any feedforward, convolutional~\cite{Lecun1998}, recurrent~\cite{Graves2012}, and graph neural network~\cite{WuZonghan2021} that is built from the composition of the basic neuron model in Fig.~\ref{fig:neuron}. Similar to what has been reported in existing ANN verification studies~\cite{Bastani2016,huang2017safety,katz2017reluplex}, the primary factor influencing our verification time is the number of non-linearities, in a neural network, rather than its architecture (see Section \ref{sec:exp}).

\subsection{Quantized Neural Networks  (QNNs)}
\label{ssec:quantisedann}

As the deployment of ANNs in software applications becomes widespread, concerns about power consumption and complexity of large models increase. In this light, one of the main techniques to reduce energy requirements related to ANN inference is \textit{quantization}~\cite{Kim2015}, which further restrict operations required to compute the output of each neuron (see \ref{eq:anncalc_active} and \ref{eq:anncalc_aff}) to integer~\cite{Lin2016} or even binary representations~\cite{Rastegari2016,Hubara2016}. State-of-the-art methods to perform such a transformation significantly improve the low-power feature of ANNs while retaining the original predictive accuracy~\cite{Guo2018}.

At the same time, the discretized nature of quantized neural networks (QNN) generates unique challenges regarding their verification~\cite{Henzinger2020}. More specifically, the output and intermediate computations performed by a network may differ from their floating-point counterparts. Thus, verification tools that operate on non-quantized ANN may return incorrect results.

We demonstrate this with the following motivating example. Assume that we want to verify the neural network in Fig.~\ref{fig:motivating-example}, which relies on the activation function ReLU and whose output can be directly computed as:
\begin{equation}\label{eq:motivating-example}
f(x_1,x_2) = A + B = \mathrm{ReLU}(2x_1 - 3x_2) + \mathrm{ReLU}(x_1 + 4x_2).
\end{equation}
Furthermore, assume that, in our example application, the output of this ANN must never fall below $f(x_1,x_2)\geq2.7$, and that we want to verify whether this is true for the input $(x_1,x_2)=(0.749,0.498)$.
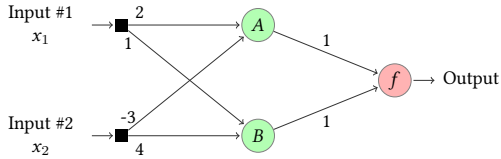
\begin{figure}[htb]
	\centering
	\def\layersep{2.5cm}
	\resizebox{0.5\linewidth}{!}{
		\begin{tikzpicture}[shorten >=1pt,->,draw=black!80, node distance=\layersep]
		\tikzstyle{every pin edge}=[<-,shorten <=1pt]
		\tikzstyle{neuron}=[circle,draw=black!50,fill=black!25,minimum size=17pt,inner sep=0pt];
		\tikzstyle{input neuron}=[neuron, fill=green!30];
		\tikzstyle{output neuron}=[neuron, fill=red!30];
		\tikzstyle{hidden neuron}=[neuron, fill=blue!30];
		\tikzstyle{annot} = [text width=4em, text centered];
		\tikzstyle{input}=[rectangle, fill=black];
		\node[input, pin=left:{\begin{tabular}{c} Input {\#}1 \\ $x_1$\end{tabular}}] (IP-1) at (-\layersep,-2) {};
		\node[input, pin=left:{\begin{tabular}{c} Input {\#}2 \\ $x_2$\end{tabular}}] (IP-2) at (-\layersep,-4) {};
		\node[input neuron] (I-1) at (0,-2) {$A$};
		\node[input neuron] (I-2) at (0,-4) {$B$};
		\node[output neuron,pin={[pin edge={->}]right:Output}, right of=I-1, yshift = -1cm] (O) {$f$};
		\path (IP-1) edge node[above,at start, anchor= south west] {2} (I-1);
		\path (IP-2) edge node[above,at start, anchor= south] {-3} (I-1);
		\path (IP-1) edge node[above,at start, anchor= north] {1} (I-2);
		\path (IP-2) edge node[above,at start, anchor= north west] {4} (I-2);
		\path (I-1) edge node[above] {1} (O);
		\path (I-2) edge node[below] {1} (O);
		\end{tikzpicture}
	}
	%
	\caption{A simple fully-connected neural network with ReLU activations and biases set to zero (not shown).}
	\label{fig:motivating-example}
\end{figure}

Now, if we run an experiment with real numbers $\mathbb{R}$ (from the mathematical domain), the result is $f(0.749,0.498)=2.745$, which satisfies our safety property $f(x_1,x_2)\geq2.7$. However, if the same ANN is quantized to a lower precision, this is not the case anymore. Indeed, for a QNN with 4-bit integer and 6-bit fractional precision, its output becomes $\hat{f}(0.749,0.498)=2.6867$, which violates our property. It is worth mentioning that such discrepancies can be even worse when larger ANNs are employed, due to cumulative error in long computation chains. Thus, in our verification approach, we make sure the actual implementation model used in an ANN implementation is captured (see Section \ref{sec:fixedpointOM}).

Besides, we could formulate another research question of interest: what is the deepest quantization that can be applied to a given ANN so that it makes correct decisions? This way, for instance, we would be able to target heavily restricted devices while still keeping the implementation correctness based on formal guarantees. Although that is not the focus of the present work, it provides the first step towards that goal. Moreover, it paves the way for a complete verification framework suitable to ANN implementations in embedded devices.

\subsection{Safety properties for ANNs and QNNs}
\label{ssec:safetyprop}

Let us now formalize the concept of safety property we briefly mentioned in the previous Section \ref{ssec:quantisedann}. In general, a safety property defines the set of states that a system is designed to reach safely. In software verification, such properties are usually defined according to a user's domain knowledge, which allows him to state which program behaviors are safe~\cite{Alpern1987}. In ANN verification, the black-box nature of their associated computation means that safety properties are usually defined on the inputs and outputs alone~\cite{Huang2020survey,Liu2021}. In this paper, we often refer to safety properties in the following form:
\begin{equation}
\label{eq:safetyimplication}
    \mathbf{x}\in\mathcal{H}\implies f(\mathbf{x})\in\mathcal{G},
\end{equation}
where $\mathbf{x}$ is an input vector, $\mathcal{H}$ is an input region, $f(\mathbf{x})$ is the corresponding output, and $\mathcal{G}$ is an output region. However, one may notice that our verification method supports any safety property that can be expressed in first-order logic (see Section~\ref{sec:safetypropertiesincode}).

A powerful and general way to define an input region $\mathcal{H}$ is choosing a center point $\mathbf{x}\in\mathcal{D}$ in the input domain $\mathcal{D}$, and letting the set $\mathcal{H}(\mathbf{x},d_{in})$ cover the whole neighborhood of points around it that are within a given distance $d_{in}(\mathbf{x},\mathbf{x}')\leq 1$~\cite{Huang2020survey,Liu2021}. As an example, in the field of image classification, robustness properties are defined in this way~\cite{Szegedy2014}. For continuous input domains $\mathcal{D}\equiv\mathbbm{R}^m$, such a distance is often defined in terms of the family of $p$-norms as follows:
\begin{equation}
\label{eq:genpnorm}
    d_p(\mathbf{x},\mathbf{x}') = ||\mathbf{x},\mathbf{x}'||_p=\Big(\sum_{i=1}^m|x_i-x'_i|^p\Big)^{\frac{1}{p}},\qquad\textit{with }p\in[1,\infty),
\end{equation}
\noindent where $p=1$ is the Manhattan distance and $p=2$ is the Euclidean distance. Furthermore, this definition can be extended to $p=\infty$ by introducing the so-called infinity or maximum norm $d_{\infty}(x,x')=\max_i(|x_i-x'_i|)$. Note that input regions defined through $p_{\infty}$ can be described by a set of linear constraints, a fact that makes them attractive to the verification community for efficiency reasons~\cite{katz2017reluplex,Wang2018,Singh2018deepZ}. Also, input vectors can be re-scaled using a diagonal matrix $Z$, allowing us to define hyper-ellipsoids (if $p=2$) and hyper-rectangles (if $p=\infty$) in the input space:
\begin{equation}
\label{eq:gengnorm}
    \hat{d}_p(\mathbf{x},\mathbf{x}',Z) = d_p(Z\mathbf{x},Z\mathbf{x}').
\end{equation}

\noindent Moreover, further attention is required if the input domain $\mathcal{M}$ is discrete in nature, for instance, in natural language processing (NLP) applications. However, a mapping to a continuous space is often available~\cite{Jia2019}.

Once we establish a definition for the input set $\mathcal{H}$ in \eqref{eq:safetyimplication}, we can complete the definition of our safety property by choosing the corresponding output set $\mathcal{G}$~\cite{Liu2021}. For regression tasks, we can again define a safe neighborhood around an output point $f(\mathbf{x})$ within a given distance $d_{out}(f(\mathbf{x}),f(\mathbf{x}'))\leq 1,\forall\mathbf{x}'\in\mathcal{H}(\mathbf{x},d_{in})$. For classification tasks, the output set $\mathcal{G}$ often comprises all points that assign the highest score to the desired class, e.g., $\mathcal{G}\equiv \{\mathbf{y}|(\mathbf{y}=f(\mathbf{x}),\forall\mathbf{x}\in\mathcal{D})\land (y_i>y_j,\forall j\neq i)\}$ for output class $i$. In Section \ref{sec:safetypropertiesincode}, we show how to define this kind of safety properties inside our verification tool.

\subsection{Satisfiability Modulo Theories (SMT)}
\label{ssec:SMT}

Once we have defined a safety property $P$, according to \eqref{eq:safetyimplication}, we need to verify that it always holds for our (quantized) neural network. As we mentioned in Section \ref{sec:intro}, there exist many approximate techniques to do so. However, in this paper, we focus on bit-precise verification via \textit{satisfiability modulo theories} (SMT) solvers~\cite{Barrett2018}.

Similar to Boolean Satisfiability (SAT) solving~\cite{Vizel2015}, the SMT approach to verification works by converting a verification problem at hand into a logic formula and then checking whether it is satisfiable. However, SMT extends SAT beyond boolean logic and allows us to model a verification problem as a decidable subset of first-order logic. At the same time, the interpretation of these models is restricted to a combination of \textit{background theories}, which are written in first-order logic with equality. More formally, given a first-order formula $F$, encoding a verification problem, and a background theory $T$, we say that $F$ is $T$-satisfiable if and only if there exists an assignment such that the union $F\cup \{T\}$ is satisfiable.

The modeling power of SMT comes from the variety of background theories $T$ that we can use. Those theories model the semantic of common mathematical objects like real, floating-point, and integer numbers, arrays, lists, bit vectors, and the operations defined on them for computational problems~\cite{Barrett2010}. While modeling capabilities of SMT are still being extended to new domains (e.g., the work of de Salvo Braz~\cite{DeSalvo2016}), mainstream SMT solvers (e.g., Z3~\cite{de2008z3}, CVC4~\cite{barrett2011cvc4}, and Boolector~\cite{brummayer2009boolector}) already offer native support for all theories above.

\subsection{Existing SMT approaches for ANNs and QNNs}
\label{ssec:existingfwl}

SMT approaches have been applied to an extensive range of verification problems~\cite{Barrett2018}. In this section, we review existing approaches for ANNs and QNNs. One may notice that due to the SMT paradigm flexibility, such approaches vary in the abstraction level at which they tackle a verification problem.

Early research applied existing SMT solvers to the verification of real-valued ANNs and showed some difficulties in scaling beyond toy examples~\cite{pulina2012challenging}. More recently, Katz \textit{et al.} proposed to extend the background theory of real numbers and include an extra predicate for the ReLU activation function~\cite{katz2017reluplex}. Since each ReLU doubles the number of verification formulas, they introduced a dedicated lazy solver, called Reluplex, which only visits a relevant subset of formulas. Their algorithm has been subsequently extended to arbitrary piecewise-linear activation functions~\cite{Katz2019}. An alternative approach by Huang \textit{et al.} asks a user to define a problem-dependent set of micro-manipulations that the SMT solver can chain to search a state space~\cite{huang2017safety}. With this approach, they can scale to medium-sized ANNs for image classification. Furthermore, verification approaches based on real number computation can be easily extended to cover floating-point implementations of ANNs~\cite{katz2017reluplex,Singh2018deepZ}.

In contrast, SMT methods to verify QNNs have to contend with a more challenging computational problem, from the theoretical perspective~\cite{Henzinger2020}. In this respect, Giacobbe \textit{et al.} chose to represent QNN operations with the bit-vector background theory and showed that the associated verification results can be very different from their real and floating-point counterparts~\cite{Giacobbe2020}. Similarly, Baranowski \textit{et al.} proposed a new fixed-point background theory and tested it on some small QNNs~\cite{Baranowski2020}. In general, low-level optimizations in SMT encoding of QNNs are shown to speed up verification processes considerably~\cite{Giacobbe2020,Henzinger2020}. In the extreme case of binarized neural networks, where quantization only allows two binary states for each variable, a verification problem can be reduced to SAT solving~\cite{Narodytska2018}. In addition, hardware-level optimizations are crucial for efficiency too~\cite{Cheng2018}.

In summary, our methodology is a generalization of the previous work by Sena \textit{et al.} focused on SMT verification of CUDA implementations of ANNs~\cite{Sena20}. As we expound in Section \ref{sec:verification}, we take advantage of existing techniques in software verification to model both ANNs and QNNs as SMT formulas. Our novelty lies in the encoding of fixed-point operations and the efficient treatment of non-linear activation functions, which allows us to verify networks beyond the simple ReLU function.

\section{A methodology for Verifying Quantized Neural Networks}
\label{sec:verification}


While we usually think of neural networks as mathematical models, their implementation is actually written in source code, in a given language. Thus, in this respect, neural networks can be treated like any other piece of software. The advantage of this strategy is twofold. First, we can readily adapt many existing software verification techniques to ANNs and QNNs. Second, we give a user access to these highly technical verification tools in a familiar coding framework.

This section lists the sequence of steps required to verify ANNs in such a way. To this end, we assume that an ANN is given as input in the form of a piece of single-threaded C code (see Section \ref{ssec:codegen}). Furthermore, we explain how to represent a quantized ANN by calling our finite-word length (FWL) implementation models, which are discussed in Section \ref{sec:fixedpointOM}. Likewise, we shown how to discretize each activation function with the algorithm in Section \ref{ssec:discretise}.

Once the code has been prepared in this way, the user can specify the desired safety property with \texttt{assume} and \texttt{assert} statements, as detailed in Section \ref{sec:safetypropertiesincode}. Then, we compute a reachable set of values for each variable, using the invariant inference techniques in Section \ref{sssec:invariant}. Finally, we verify the safety property via SMT model checking, as explained in Section \ref{sec:modelchecking}. All the techniques we use to reduce the search space of the SMT solver are listed in Sections \ref{sec:IncrementalVerificationusingInvariantInference} and \ref{sec:SearchSpaceReduction}.

Our whole verification methodology is summarized in Fig. \ref{fig:fp-nn-methodology}. Furthermore, we conclude in Section \ref{ssec:imagerobustnessexample} with a complete walk-through example of our workflow.

\begin{figure*}[htb]
\centering
\includegraphics[width=0.98\textwidth]{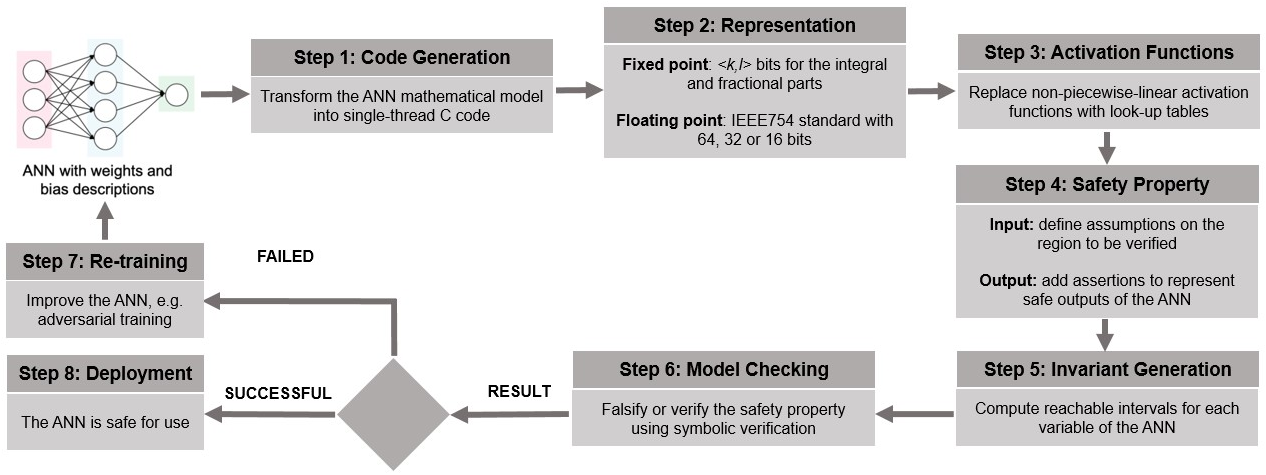}
\caption{The proposed verification workflow for fixed- and floating-point ANNs.}
\label{fig:fp-nn-methodology}
\end{figure*}

\subsection{ANN code generation}
\label{ssec:codegen}

The current mainstream approach to ANN development uses high-level machine learning libraries such as TensorFlow and PyTorch to define the architecture of the neural network and train its weights. Once the development phase is over, a final implementation of the ANN is produced, targeting a specific computer architecture, e.g. AMDx64, CUDA-enabled GPUs~\cite{OH20041311}, embedded systems running on FPGAs~\cite{10.1007/978-3-540-45234-8_120}, etc. Depending on the application, these implementations are optimized with several objectives in mind, ranging from speed of inference to energy consumption and memory required~\cite{10.1145/3309551,7551399}.

In this paper, we use the C language as an abstraction of all these possible system realizations, with the addition of implementation models to represent fixed-point arithmetic (see Section \ref{sec:fixedpointOM}). Furthermore, we limit our scope to sequential code, and leave the verification of concurrent implementations of ANNs (e.g. CUDA) for future work.

At the same time, given the mathematical model of a specific ANN (see Section \ref{ssec:ann}), there exist multiple possible sequential implementations of it. This is because neural networks are highly parallel, since the output of all neurons in a single layer can be computed independently. Furthermore, the activation potential $u_k$ of each neuron (see Equation \ref{eq:anncalc_aff}) is the result of a sequence of multiply-and-accumulate (MAC) operations, whose order can be changed arbitrarily.

In our experiments in Section \ref{sec:code-balancing}, we show that our verification framework is insensitive to changes in the order of the basic operations performed by the ANN. In other words, all equivalent implementations of the same ANN will yield the same verification performance in terms of time, memory usage and outcome. Thus, for the remainder of this section, we work under the assumption that a specific implementation is given, and detail the sequence of processing steps required for its verification.

\subsection{Implementation models for fixed-point ANN implementations}
\label{sec:fixedpointOM}

In this section, we discuss how our implementation models work to support fixed-point verification of neural network implementations.

Generally, there are two ways of supporting fixed-point neural network implementations~\cite{Giacobbe2020}: ($1$) to converting inputs into fixed-point and perform all the underlying steps, e.g., training and validation, in fixed-point; or ($2$) converting trained models and neural network operations, {e.g.}, realization, from floating-point representation into fixed-point, which is then followed by a check of the desired properties. The former is likely to produce better representations, but the latter is likely to be more practical~\cite{Hubara2017}, mainly because datasets are usually provided in floating-point representation. In the present work, we have chosen the latter. Moreover, such a method, also known as network compression or quantization, is also the usual way of deploying neural networks on restricted devices, which reinforces its use. 

Our goal is to transform an existing model (and its constraints) defined in the C programming language into a fixed-point representation. Here, a fixed-point format is specified as $\langle k, l \rangle$, where $k$ denotes the number of bits to encode its sign and integral part, resulting in a representation $I$, and $l$ indicates the number of bits to encode its fractional part, resulting in $F$. Furthermore, given a rational number, we can represent it in fixed-point by using $k + l$ bits, which is interpreted as $I + \frac{F}{2^l}$. Such representation allows us to take a hardware platform's limitations, where a specific model will be executed, into account, in such a way that a more suitable implementation is provided. Moreover, in the present context, two's complement is used for value representation and arithmetic operations, due to some advantages, such as the wrap-around effect \cite{Chaves2019}. For instance, if we want to encode number $+3.25$ into format $\langle 5, 3 \rangle$, it will give rise to the following representation in memory: $\{00011|010\}$, with the most significant bit (i.e., $0$) indicating the sign "$+$", $I = 3$, and $F=2$. 

In order to model the quantization effect on an ANN's computation steps, we need to convert each arithmetic operation (addition, subtraction, multiplication, or division) from floating-point to their respective fixed-point counterparts. In particular, these operations and conversions must take into account the parameters $k$ and $l$, along with the sign bit. We achieve this goal with the implementation models proposed by Chaves {\it et al.}~\cite{Chaves2019}, which have been extensively validated in the digital controller domain. Indeed, they replace the mentioned arithmetic operations (i.e., "$+$", "$-$", "$*$", and "$/$") and then return results according to a specific precision. Furthermore, these implementation models formally define a set of methods and values that precisely represent fixed-point operations' behavior.

In Fig.~\ref{fig:fxp-sample-conversion-code} we show an example of how to convert a piece of floating-point source-code into a fixed-point representation with the proposed implementation models. Here, we have a code snippet that computes the activation potential of a single neuron, one of the basic operations in ANNs. One may notice how the types and operations have been changed in the fixed-point version. In particular, \texttt{fxp\_float\_to\_fxp} transforms a type \texttt{float} into a  type \texttt{fxp\_t} (fixed point), and both \texttt{fxp\_add} and \texttt{fxp\_mult} make sure that the addition and multiplication arithmetic operations are performed in fixed-point and take into account the previously defined desired precision.

\begin{figure}[ht]
\centering
\begin{subfigure}{0.4\textwidth}
\begin{lstlisting}[numbers=left]
float potential(float *w,
                unsigned int w_len,
                float *x,
                unsigned int x_len,
                float b) {
  
  if (w_len != x_len) {
    return 0;
  }
  
  float result = 0;
  
  for (unsigned int i = 0; i < w_len; ++i) {
    result += w[i] * x[i];
  }
  
  result += b;
  
  return result;
}
\end{lstlisting}
\caption{}
\end{subfigure}
\hspace{0.05\textwidth}
\begin{subfigure}{0.4\textwidth}
\begin{lstlisting}[numbers=left]  
fxp_t potential(float *w,
                unsigned int w_len,
                float *x,
                unsigned int x_len,
                float b) {
  if (w_len != x_len) {
    return 0;
  }
  fxp_t result = 0;
  for (unsigned int i = 0; i < w_len; ++i) {
    fxp_t w_fxp = fxp_float_to_fxp(w[i]);
    fxp_t x_fxp = fxp_float_to_fxp(x[i]);
    result = fxp_add(result, fxp_mult(w_fxp, x_fxp));
  }
  fxp_t b_fxp = fxp_float_to_fxp(b);
  result = fxp_add(result, b_fxp);
  return result;
}
\end{lstlisting}
\caption{}
\end{subfigure}
\caption{A method to compute the activation potential of neurons implemented in C: (a) floating-point and (b) fixed-point versions.}
\label{fig:fxp-sample-conversion-code}
\end{figure}

In summary, the fixed-point version of an ANN's code references the appropriate implementation models, thus ensuring that the behavior of each fixed-point arithmetic operation is carried out correctly. Our experiments, in Sections \ref{sec:verification-of-anns-with-fwl-implementation-v2} and \ref{sec:adversarial-cases-verification-in-fwl-ann-implementations-v2}, show the impact of different levels of quantization granularity in ANNs.

Finally, another aspect is worth mentioning: for an entirely correct implementation, when a fixed-point format is chosen, one should still represent the dynamic range associated with the target data. If that is not done, overflow occurs, which introduces errors that can jeopardize an ANN's decision.

In other words, if a given variable holds values that range from $-15.5$ to $15.5$, for instance, a format $\langle 2, 2 \rangle$ should not be used because that would lead to frequent overflow events. Specifically, values above $3.75$ would not be represented. Consequently, in this specific case, a format $\langle 5, 2 \rangle$ (note the dynamic range provided by the integer part), for instance, would be suitable, then keeping correct computation in all associated operations.

\subsection{Discretization of non-linear activation functions}
\label{ssec:discretise}

As mentioned in Section \ref{ssec:ann}, the choice of an activation function can have a considerable impact on verification times. While piece-wise linear functions can be readily represented as a (sequence of) if-then-else instructions, non-linear activation functions require careful adjustments to avoid severe performance degradation. This section presents an approach to convert such non-linear functions into look-up tables, thus significantly speeding up verification processes.

Assume that the non-linear activation function $\mathcal{N}:\mathcal{U} \mapsto \mathbb{R}$ is a piece-wise Lipschitz continuous function \cite{searcoid2006}, thus there is a finite set of $a$ locally Lipschitz continuous functions $\mathcal{N}_{i}:\mathcal{U}_{i} \mapsto \mathbb{R}$ for $i\in \mathbb{N}_{\leq a}$, the so-called selection functions, such that the sets $\mathcal{U}_{i}\subset \mathbb{R}$ are disjoint intervals, $\mathcal{N}(u) \in \left\lbrace \mathcal{N}_{1}(u), \ldots ,\mathcal{N}_{a}(u)  \right\rbrace$ holds for all $u \in \mathbb{U}$, $\mathcal{U} = \bigcup_{i\in \mathbb{N}_{\leq a}}{\mathcal{U}_{i}}$,  and
\begin{equation}
\label{eq:lip}
    \| \mathcal{N}_{i}(u_{1}) - \mathcal{N}_{i}(u_{2})   \| \leq \lambda_{i} \| u_{1} - u_{2} \|, \quad \forall u_{1},u_{2} \in \mathcal{U}_{i},
\end{equation}

\noindent where $\lambda_{i}$ denotes the Lipschitz constant of $\mathcal{N}_{i}$. 

The proposed discretisation approach is applied to each subset $\mathcal{U}_{i}$. The general idea consists in discretising the $\mathcal{U}_{i}$ by obtaining the discrete and countable set $\tilde{\mathcal{U}}_{i} \subset \mathcal{U}_{i}$. Then, we build a lookup table for rounding the evaluation of $\mathcal{N}_{i}(u)$ to $\tilde{\mathcal{N}}_{i}(u):\mathcal{U}_{i}\mapsto \mathcal{R}$, and consequently rounding $\mathcal{N}(u)$ to  $\tilde{\mathcal{N}}(u)\in \left\lbrace \tilde{\mathcal{N}}_{1}(u), \ldots ,\tilde{\mathcal{N}}_{a}(u)  \right\rbrace$. This lookup table contains uniformly distributed $N_{i}$ samples within $\mathcal{U}_{i}$, including interval limits, to ensure the accuracy $\| \tilde{\mathcal{N}}_{i}(u) -  \mathcal{N}_{i}(u)\| \leq \epsilon $. 
Let $L_{i}$ be defined as the length of the interval $\mathcal{U}_{i}$, i.e., 
\begin{equation}
\label{eq:int_leng}
    L_{i} \triangleq \sup_{u\in \mathcal{U}_{i}}{u} - \inf_{u\in \mathcal{U}_{i}}{u}.
\end{equation}

\noindent This way, the following Theorem can be used to choose the number of samples $N_{i}$ to ensure the desired accuracy $\epsilon$.

\begin{theorem}
\label{th:approx}
Let the non-linear activation function $\mathcal{N}:\mathcal{U} \mapsto \mathbb{R}$, $\mathcal{N}\in \left\lbrace \mathcal{N}_{1}(u), \ldots ,\mathcal{N}_{a}(u)  \right\rbrace$, be piecewise Lipschitz continuous such that each selection function $\mathcal{N}_{i}(u):\mathcal{U}_{i}\mapsto \mathcal{R}$ presents the Lipschitz constant $\lambda_{i}$, and consider the discrete approximation $\tilde{\mathcal{N}}(u)\in \left\lbrace \tilde{\mathcal{N}}_{1}(u), \ldots ,\tilde{\mathcal{N}}_{a}(u)  \right\rbrace$, where each selection function $\tilde{\mathcal{N}}_{i}:\mathcal{U}_{i}\mapsto \mathbb{R}$, for $i \in \mathbb{N}_{\leq a}$ is obtained with $\mathcal{U}_{i}\subset \mathcal{U}$ containing $N_{i}$ samples. The approximation error is bounded as
\begin{equation}
    \label{eq:accu}
    \| \tilde{\mathcal{N}}(u) -  \mathcal{N}(u)\| \leq \epsilon ,
\end{equation}
 \noindent for a given $\epsilon$, if 

\begin{equation}
    \label{eq:Nlimit}
    N_{i} \geq 1 + \frac{L_{i}\lambda_{i}}{\epsilon}, \forall i \in \mathbb{N}_{\leq a}
\end{equation}

\noindent holds.
\end{theorem}

\begin{proof}
Given that the length of each interval $\mathcal{U}_{i}$ is $L_{i}$ (cf. \eqref{eq:int_leng}), the length of each sub-interval, obtained by uniformly dividing $\mathcal{U}_{i}$ at the $N_{i}$ samples, is $\frac{L_{i}}{N_{i}-1}$. Considering the Lipschitz continuity in~\eqref{eq:lip}, the rounding error for $\tilde{\mathcal{N}}_{i}(u)$ is bounded as
\begin{equation}
\label{eq:ineq_err_1}
    \| \tilde{\mathcal{N}}_{i}(u) -  \mathcal{N}_{i}(u)\| \leq \frac{L_{i}}{N_{i}-1}\lambda_{i}.
\end{equation}

\noindent If~\eqref{eq:accu} holds for all $i \in \mathbb{N}_{\leq a}$, the inequality 
\begin{equation}
\label{eq:ineq_err_2}
     \frac{L_{i}}{N_{i}-1}\lambda_{i} \leq \epsilon
\end{equation}

\noindent and, consequently,~\eqref{eq:Nlimit} also hold. Moreover, from \eqref{eq:ineq_err_1} and \eqref{eq:ineq_err_2}, $\| \tilde{\mathcal{N}}_{i}(u) -  \mathcal{N}_{i}(u)\| \leq \epsilon$ for all $i \in \mathbb{N}_{\leq a}$. 
\end{proof}

Based on Theorem~\ref{th:approx}, the number of samples used in the discretization of nonlinear activation functions, such as the $\mathcal{N}_{\mathrm{TanH}}$ and $\mathcal{N}_{\mathrm{Sigm}}$, described respectively in \eqref{eq:neuronout} and~\eqref{eq:neuronout_tanh}, can be computed to ensure some desired accuracy. Without loss of generality, the approximation $\tilde{\mathcal{N}}_{i}(u)$ can be defined as
\begin{equation}
    \label{eq:ntilde}
    \tilde{\mathcal{N}}_{i}(u) = \mathcal{N}_{i}\left(\mathcal{A}_{i}(u)\right),
\end{equation}
\noindent where $\mathcal{A}_{i}:\mathcal{U}_{i}\mapsto\tilde{\mathcal{U}}_{i}$ is an arbitrary approximation operator, e.g., rounding and quantization.

For instance, consider that we want to obtain the function $\Tilde{\mathcal{N}}_{\mathrm{Sigm}}$, which approximates $\mathcal{N}_{\mathrm{Sigm}}$ based on a discrete domain $\Tilde{\mathcal{U}}$, with target accuracy $\epsilon=0.01$. It is clear that $\mathcal{N}_{\mathrm{Sigm}}$ is globally Lipschitz continuous with constant $\lambda_{\mathrm{Sigm}}=0.25$ since the $
\sup_{u\in \mathbb{U}}{\vert \mathcal{N}_{\mathrm{Sigm}}(u) \vert } = 0.25$, and $\mathcal{U}=\mathbb{R}$. Moreover, let us choose the following three intervals to define the approximation $\Tilde{\mathcal{N}}_{\mathrm{Sigm}}(u)$:
\begin{equation}
    \mathcal{U}_{1}=\left(-\infty,-20\right],
\end{equation}
\begin{equation}
\mathcal{U}_{2}=\left(-20,20\right),
\end{equation}
\noindent and
\begin{equation}
\mathcal{U}_{3}=\left[20,\infty\right),
\end{equation}
\noindent since the derivative of $\mathcal{N}_{\mathrm{Sigm}}(u)$ is negligible for $u\in \mathcal{U}_{1}\cup \mathcal{U}_{3}$, i.e., $\lambda_{1}\approx 0$ and $\lambda_{3}\approx 0$, while the constant $\lambda_{2}$ is equivalent to the global Lipschitz constant, i.e., $\lambda_{2}=\lambda_{\mathrm{Sigm}}=0.25$. Now, we can use \eqref{eq:Nlimit} to compute the number of samples in each interval necessary to ensure the desired accuracy $\epsilon=0.01$. Accordingly, the numbers of samples are $N_{1}=N_{3}=1$ and $N_{2}=1001$ since $L_{2}=40$ (cf.~\eqref{eq:int_leng}). Notice that the approximators $\mathcal{A}_{i}$ can be arbitrarily chosen. For this example, it is suggested to choose $\mathcal{A}_{1}=-20$ and $\mathcal{A}_{3}=20$, because it is not necessary to have more samples than the limits of the intervals for $\mathcal{U}_{1}$ and $\mathcal{U}_{3}$. Finally, $\mathcal{A}_{2}$ can be chosen as the half-towards-zero rounding with 3 decimal digits for floating-point and real ANNs, and as the underlying quantization function for fixed-point ANNs. 

Fig.~\ref{fig:disc} illustrates the effect of the discretization when evaluating the sigmoid function. Note that the approximation fits well for $\epsilon=0.01$, and it becomes poor when $\epsilon$ increases. It is worth mentioning that a look-up table is fundamentally a trade-off between speed and memory. If the latter is not a restriction, verification processes may benefit from such a strategy. Another interesting point is that such a discretization strategy should be under the final desired fixed-point format so that a safety-property verification is not compromised. This way, $\epsilon$ should be arbitrarily small and also much lower than the quantization step incurred by a fixed-point format.

\begin{figure}[!ht]
	\centering
	\begin{subfigure}[t]{0.48\linewidth}
		\centering
        \includegraphics[width=\linewidth]{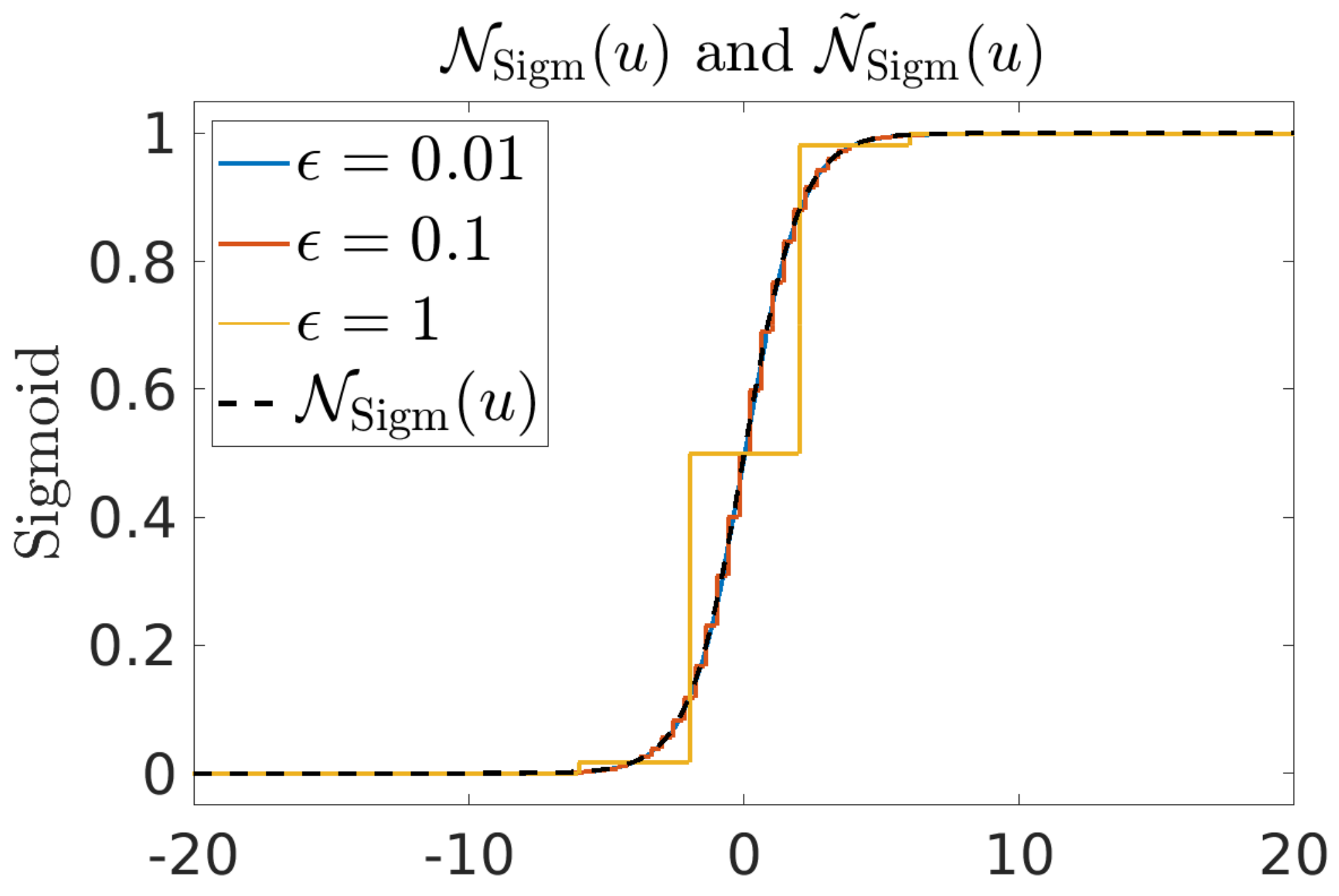}
		\caption{\label{fig:disc_a}}
	\end{subfigure}
	\begin{subfigure}[t]{0.48\linewidth}
		\centering
        \includegraphics[width=\linewidth]{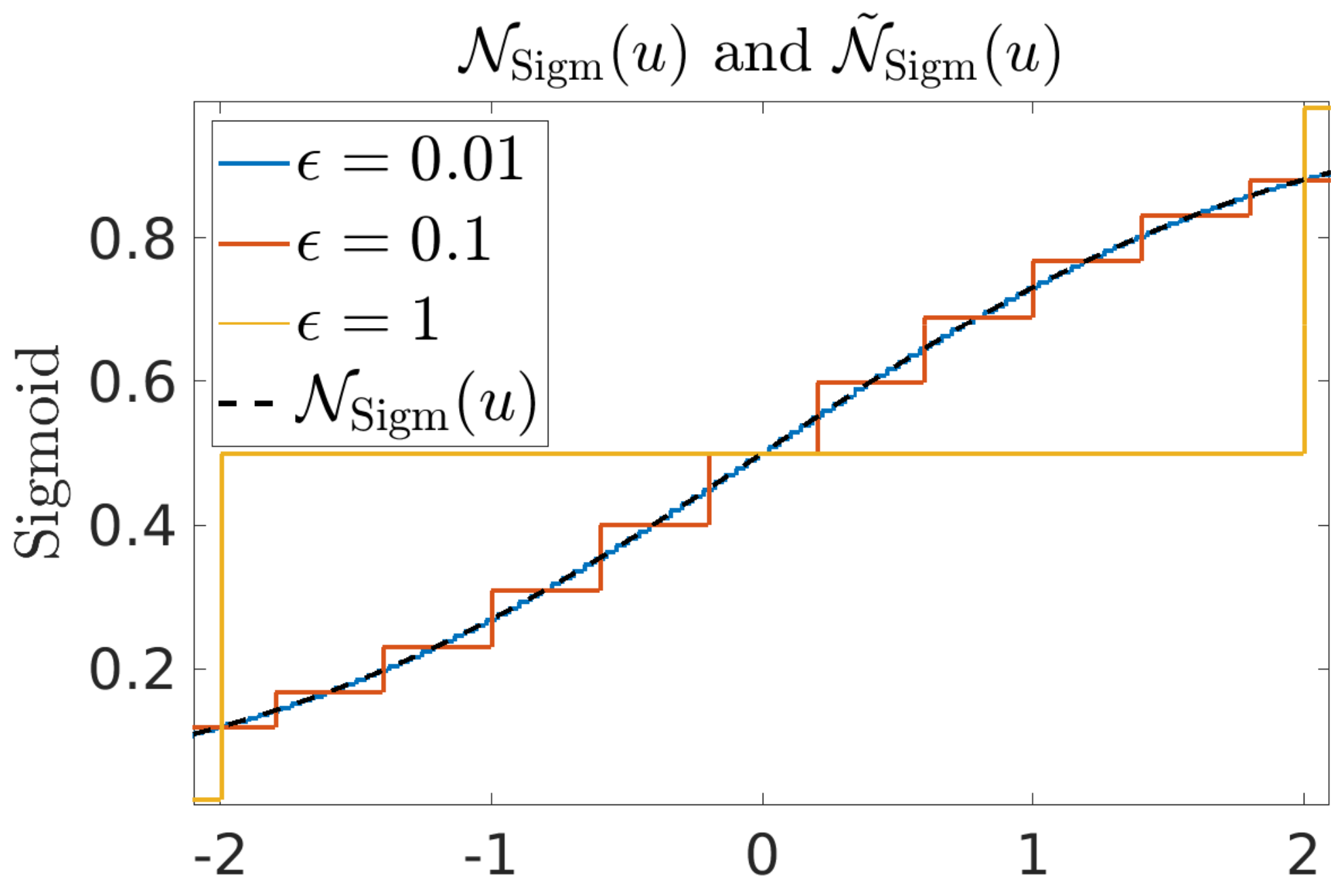}
		\caption{\label{fig:disc_b}}
	\end{subfigure}
	\caption{Comparison between the real sigmoid $\mathcal{N}_{\mathrm{Sigm}}$ and its discretizations $\Tilde{\mathcal{N}}_{\mathrm{Sigm}}$ for $\epsilon=0.01$ ($N_{2}=1001$), $\epsilon=0.1$ ($N_{2}=101$), $\epsilon=1$ ($N_{2}=11$):  (a) sigmoid activation function together with its approximations within the range $\left[ -20,20 \right]$ and (b) a zoom in to show the interval $\left[ -2,2 \right]$. \label{fig:disc}} 
\end{figure}

\subsection{Introducing safety properties in ANN code}
\label{sec:safetypropertiesincode}

As we explain in Section \ref{ssec:safetyprop}, verifying an ANN means proving that a given safety property holds. Such a safety property is a falsifiable mathematical relation defined on the values of an ANN's variables. Since we are considering software implementation of ANNs here, in this section, we then show how to annotate ANN code and also how to specify a desired safety property.

As a preliminary step, we annotate code by replacing the concrete input to the neural network with a general non-deterministic input. We do so by assigning a non-deterministic value to each input variable as in the following example (another example is shown in Figure \ref{figure:neural-net-verification}):
\begin{align}\begin{split}
    \texttt{float x\_1 = nondet\_float()}\\
    \texttt{float x\_2 = nondet\_float()}
\end{split}\end{align}
where we use the notation \texttt{nondet\_float()} prescribed by our underlying verification tool ESBMC~\cite{GadelhaMMC0N18,GadelhaMCN19}. With this, our verification tool knows to expect any possible input, and it is then the role of the safety property (see Equation \ref{eq:safety_prop_input_assume} below) to restrict the input space to the sub-domain of interest.

Let us consider safety properties in the general form $\mathbf{x}\in\mathcal{H}\implies \mathbf{y}\in\mathcal{G}$, where knowing an input vector $\mathbf{x}$, belonging to $\mathcal{H}$, guarantees that the output vector $\mathbf{y}=f(\mathbf{x})$ belongs to $\mathcal{G}$. Consequently, we encode the premise of this implication with a pre-condition instruction \texttt{assume}, specifying the set of values $\mathcal{H}$ that each $x_i\in \mathbf{x}$ can take. For example, a rectangular domain for the input variables $x_1\in[0,2]$ and $x_2\in[-\frac{1}{2},+\frac{1}{2})$ can be encoded as
\begin{equation}
\label{eq:safety_prop_input_assume}
    \texttt{assume(x\_1\;>=\;0\;\&\&\;x\_1\;<=\;2)}\\
\end{equation}
\noindent and
\begin{equation}
    \texttt{assume(x\_2\;>=\;-0.5\;\&\&\;x\_2\;<\;0.5)}.
\end{equation}
\noindent This notation instructs the subsequent SMT model checking to search only the inputs that satisfy the conditions specified in the \texttt{assume} instruction, as it ignores an execution when being false (e.g., see \texttt{\_\_ESBMC\_assume} \cite{AlbuquerqueSBMF17}), thus making sure that the premise of the safety property $\mathbf{x}\in\mathcal{H}$ is satisfied. Note also that the instruction \texttt{assume} is general and supports any boolean condition as its argument.\footnote{When working with quantized representations, whether fixed or floating point, extra care should be taken in checking that the specified constants are rounded in a way that does not break the desired condition.} This way, any form of input region $\mathcal{H}$ can be specified, as long as it is valid C code syntax. At the same time, hyper-rectangular input domains tend to lead to faster verification times, as mentioned in Section \ref{ssec:safetyprop}.

In contrast, we encode the conclusion of the implication with the post-condition instruction \texttt{assert}, specifying the set of values $\mathcal{G}$ that each variable $y_i\in \mathbf{y}$ can safely range in. For instance, if we have a binary classification network with two outputs $y_1$ and $y_2$ indicating the score of each class, we can encode the conclusion of a robustness safety property for the second class as
\begin{equation}
    \texttt{assert(y\_2\;>\;y\_1)}.
\end{equation}
\noindent Consequently, it requires that when that premise is satisfied, our binary network always predicts the second class. As for the input region $\mathcal{H}$, the \texttt{assert} instruction can be used to specify a variety of output regions $\mathcal{G}$, but now making an assessment of what is expected.

\subsection{Invariant inference via interval analysis}
\label{sssec:invariant}

Once the safety property has been specified, as we explain in Section \ref{sec:safetypropertiesincode}, we can inject further \texttt{assume} instructions in the code and reduce the model checker's search space. Indeed, given the sequential nature of ANN computation, the set $\mathcal{H}$ of values allowed by the premise of a safety property also constrains the range of the following intermediate computation steps. Thus, if we can explicitly derive and unfold these additional constraints onto intermediate variables, in such a way that we propagate constraints and benefit from them on subsequent operations, we can more succinctly tell a model checker where to look for counterexamples.

In general, deriving additional (over-approximated) constraints on intermediate computation steps falls under the umbrella of \textit{invariant inference}~\cite{RochaRIC017}. It is based on the discovery of an assertion that holds during the execution a given piece of code, which can then be used in verification procedures. For neural network code, which does not contain loops or dynamic memory allocation, we find that an interval invariant analysis suffices~\cite{Moore2009}. Such a method of invariant analysis computes lower and upper bounds on the values of each program variable (e.g., $a \leq x \leq b$, where $a$, $b$ are constants and $x$ is a variable), by propagating the initial set $\mathcal{H}$ through an ANN with interval arithmetic rules. One may notice that more complex constraint propagation methods (e.g., zonotopes and polyhedra) exist in the literature~\cite{Tran2020}, but whether reduction in search space justifies the additional computational cost is an open problem. Moreover, given that neural network quantization, as tackled here, is already used for integrating this kind of system into restricted devices, low complexity is desired, at least initially.

On the more practical side, there are many tools to perform interval analysis of C code. In our experiments, in Section \ref{sec:exp}, we have used the evolved value analysis (EVA) plugin of the open-source tool FRAMA-C \cite{blanchard2018}. We then inject intervals into ANN code as additional pre-condition instructions \texttt{assume} on intermediate variables, thus covering the entire processing chain. Finally, we have compared this method with the native interval analysis support provided by the state-of-the-art verification tool ESBMC \cite{morse2014esbmc}, in Section \ref{sec:fine-tuning-bmc-parameters-v2}, and found that combining them (both enabled) yields the best results.

\subsection{Model checking ANN implementations}
\label{sec:modelchecking}

Given the annotated C code from Sections \ref{sec:safetypropertiesincode} to \ref{sssec:invariant}, we are now tasked with answering the following verification question: do \textit{all} inputs that satisfy pre-conditions \texttt{assume} also satisfy associated \texttt{assert} post-conditions, in a specific ANN implementation? In other terms, are we able to find at least \texttt{one} specific input that violates a safety property, given an ANN implemented with a specific precision? In this section, we explain how to answer this question with state-of-the-art \textit{symbolic model checking} techniques.

In general, model checking is concerned with verifying whether a given property $\phi$ holds for a finite state transition system $M$, which is typically represented by a triple $(S,I,T)$~\cite{Clarke2018}. More formally, these mathematical objects are defined as follows:
\begin{itemize}
    \item $S$ is the set of states a system can be in, where each state consists of the value of the program counter (\textit{PC}), local and global variables;
    \item $I:S\to\{0,1\}$ is an indicator function for a set of initial states;
    \item $T:s_{i}\to s_{j}$, with $s_{i}, s_{j} \in S$, is a transition function describing a system's evolution, i.e., pairs of states specifying how a system can move from state to state;
    \item $\phi:S\to\{0,1\}$ is an indicator function for safe states.
\end{itemize}

In our case, the annotated ANN code defines these objects implicitly. $S$ represents all possible value assignments to a set of program variables, including the PC. $I$ indicates all assignments that satisfy existing \texttt{assume} pre-conditions. $T$ holds the semantic of each instruction in code, defining how to go from one state to another, which allows checking for reachability (cf. Definition~\ref{reachable-state} below). Finally, $\phi$ represents a safety property encoded with existing \texttt{assert} post-conditions.

\begin{definition}\label{reachable-state}
\textit{Let $M$ be a transition system. A state $s_{r} \in S$ is called a reachable state in $M$ if there exists a finite sequence of $N$ state transitions starting from an initial state $s_{0}$ and ending in state $s_{r}$, i.e., $s_{0} \stackrel{T_{0}}{\rightarrow} s_{1} \stackrel{T_{1}}{\rightarrow} \ldots \stackrel{T_{N}}{\rightarrow} s_{N+1} = s_{r}$, where $s_{n} \stackrel{T_{n}}{\rightarrow} s_{n+1}$ denotes a state transition when applying $T_{n}$.}
\end{definition}

In practice, several state-of-the-art model checkers accept C code as input~\cite{Beyer2011,Kroening2014,GadelhaMMC0N18,GadelhaMCN19}. Frequently, input code is readily converted into static single assignment (SSA) form before further processing \cite{Cytron1991}, which has the advantage of making underlying finite-state transition systems more explicit. We show an example, in Fig. \ref{figure:neural-net-verification}, parts (a) and (b), of such a conversion procedure.

Note that, in all experiments in Section \ref{sec:exp}, we use ESBMC for this model checking step~\cite{GadelhaMMC0N18,GadelhaMCN19}. Like any other state-of-the-art model checker, ESBMC has been heavily optimized to reduce verification times. However, not all of these optimization techniques apply to feed-forward neural network code, which does not contain loops and recursions. In the following Sections \ref{sec:IncrementalVerificationusingInvariantInference} and \ref{sec:SearchSpaceReduction}, we clarify which techniques do apply to ANN code.

\subsection{Incremental verification using lemma learning via SMT}
\label{sec:IncrementalVerificationusingInvariantInference}

The SMTLIB logic format introduced an assertion stack concept and the ability to push and pop assertions of it~\cite{BarFT-SMTLIB}. In particular, some SMTLIB compliant SMT solvers have an internal stack of assertions, which we can add new assertions to or remove old ones from. The main idea here is to enable assertion retraction and lemma learning incrementally. The former allows one to add assertions to a formula, evaluate the individual result, and then return the same formula to its original form. The latter happens when the SMT solver stores facts (in the form of lemmas over a formula's variables). In summary, it has already determined a formula, which may prove helpful in future checks. 

Here, we enable the underlying SMT solver to use lemmas determined during previous checks for future ones, thereby optimizing search procedures and potentially eliminating a large amount of formula state-space to be searched. Note that previous studies report encouraging results using incremental (bounded) model checking for software, increasing the search depth without leading to the overhead of restarting a verification process from scratch~\cite{GuntherW14}. This way, we apply incremental SMT solving to verify neural net implementations, where a formula is built up in stages, and lemmas are learned, along the way, about that same formula.

In particular, this incremental verification is beneficial to exploring neural net implementations by ESBMC since they contain various \textit{ite} operators (e.g., to represent ReLU activation functions). The existing operation of the SMT solver follows directly from ESBMC. Indeed, once we build the directed acyclic graph (DAG) and produce an SSA program by symbolic execution, from a neural net's implementation, that program is converted to a fragment of first-order logic and translated into a form acceptable for the SMT solver. Then, after checking the satisfiability of a given formula, the latter is discarded. Here, many \textit{ite} operations will be converted, solved, and discarded during a neural net's verification procedure. Since each variable in an \textit{ite} operation is assigned only once along each path in SSA form, this requires a case split to evaluate the activation function, e.g., $z \, = \, g \, \, ? \, \, x \, : \, y \,$. As a result, we call the SMT solver during a symbolic execution to check the satisfiability of the guard $g$ and then determine the value of variable $z$. Using \textit{ite} retraction to build and deconstruct a formula has the potential to reduce SMT-conversion overhead, and lemma learning could lead to swifter verification times. The SMT solvers supported by ESBMC (i.e., Z3~\cite{de2008z3}, Yices~\cite{dutertre2014yices}) claim lemma learning as a feature, thereby allowing us to evaluate its impact for verifying neural-net implementations.

To use incremental SMT, during neural net verification, we must identify ways to reuse an SMT formula by pushing and popping \textit{ite} operations into the solver. In particular, we retain the formula produced for an \textit{ite} operator, identify the common prefix between it and the next \textit{ite} operator produced, and retract all the \textit{ite} operations that can be evaluated. Then, we place the \textit{ite} operators that could not be evaluated on top of the remaining formula. Fig.~\ref{figure:neural-net-verification} illustrates this approach. In particular, in Fig.~\ref{figure:neural-net-verification}(a), we have two inputs $x$ and $y$ in lines $4$ and $5$, respectively; three assignments in lines $6$, $8$, and $10$; three \textit{ite} operators, which represent ReLU activation functions, in lines $7$, $9$, and $11$; and one assertion representing a safety property, in line $12$. Fig.~\ref{figure:neural-net-verification}(b) illustrates the program of Fig.~\ref{figure:neural-net-verification}(a) converted into SSA form (i.e., each variable is assigned exactly once), which is the format we use for incremental learning. 

During the symbolic execution of this neural net implementation, based on Fig.~\ref{figure:neural-net-verification}(a), we check the satisfiability of guard ``a < 0'', in line 7, and conclude that it could either be evaluated as ``true'' or ``false'' since ``a'' can assume values between $-3$ (lowest) and $2$ (highest). As a result, we cannot simplify this expression before checking the safety property in line $12$ of Fig.~\ref{figure:neural-net-verification}(a). However, we can learn from this assignment and place its \textit{ite} operation on top of the remaining formula, which can then be used to check the mentioned safety property. After that, we check the satisfiability of guard ``b < 0'', in line $9$, and conclude that it always evaluates to ``false'' since ``b'' can assume only positive numbers between $0$ (lowest) and $5$ (highest). So, we are thus able to remove this expression and the respective assertion. Similarly, we check the satisfiability of guard ``f < 0'', in line $11$, and also conclude that it always evaluates to ``false'' since ``f'' can assume only positive values between $0$ (lowest) and $4$ (highest). 

We show the simplified neural net implementation using our incremental verification via lemma learning in Figure~\ref{figure:neural-net-verification}(c). One may notice that we have safely removed two ReLU activation functions represented by the variables $b2$ and $f2$, initially present in Fig.~\ref{figure:neural-net-verification}(b), which thus reduce the formula's size to be checked by the underlying SMT solver. Note further that we have learned that variable ``a'' can assume values between $-3$ (lowest) and $2$ (highest), which can be used to check the assert statement specified in line $9$ of Fig.~\ref{figure:neural-net-verification}(b). Consequently, that same assert can not be identified in Fig.~\ref{figure:neural-net-verification}(c) anymore because the knowledge of its range allowed such a simplification. The assertions $b1 \leq 5$ and $f1 \leq 4$ were also removed since we previously learned the intervals for the variable $b$ and $f$. Lastly, we can observe the ability to perform a query at any neuron using incremental verification, which can help prune neural net implementation before deploying it to an embedded device with time, memory, and energy constraints.

\begin{figure}[ht]
\centering
\begin{subfigure}{0.4\textwidth}
 \begin{lstlisting}[numbers=left]
int main() {
  _Bool x, y;
  int a, b, f;
  x = nondet_bool();
  y = nondet_bool();
  a = ((2*x) - (3*y)); 
  a = a < 0 ? 0 : a; 
  b = (x + (4*y));
  b = b < 0 ? 0 : b; 
  f = ((3*x) + y);
  f = f < 0 ? 0 : f; 
  assert(a <= 2 && b <= 5 && f <= 4); 
  return 0;
}
\end{lstlisting}
\caption{}
\end{subfigure}
\hspace{0.05\textwidth}
\begin{subfigure}{0.4\textwidth}
 \begin{lstlisting}[numbers=left]  
x1 == nondet_symbol(nondet0)
y1 == nondet_symbol(nondet1)
a1 == 2 * (int)x1 - 3 * (int)y1
a2 == (a1 < 0 ? 0 : a1)
b1 == (int)x1 + 4 * (int)y1
b2 == (b1 < 0 ? 0 : b1)
f1 == 3 * (int)x1 + (int)y1
f2 == (f1 < 0 ? 0 : f1)
(assert) a2 <= 2
(assert) b2 <= 5
(assert) f2 <= 4
\end{lstlisting}
\caption{}
\end{subfigure}
\hspace{0.05\textwidth}
\begin{subfigure}{0.4\textwidth}
 \begin{lstlisting}[numbers=left]  
x1 == nondet_symbol(nondet0)
y1 == nondet_symbol(nondet1)
a1 == 2 * (int)x1 - 3 * (int)y1
a2 == (a1 < 0 ? 0 : a1)
\end{lstlisting}
\caption{}
\end{subfigure}
\caption{(a) A simple neural net implemented in C, where variables ``a'', ``b'', and ``c'' range from $-3$ to $2$, $0$ to $5$, and $0$ to $4$, respectively. (b) The initial neural-net C program converted into SSA form. (c) A simplified version of the SSA form using incremental learning.} 
\label{figure:neural-net-verification}
\end{figure}

\subsection{Constant folding, slicing and expression balancing for search-space reduction}
\label{sec:SearchSpaceReduction}

Our employed verification engine implements general code optimizations, when converting a neural net implementation to SMT. These include \textit{constant folding}, \textit{slicing} and expression balancing~\cite{Cordeiro11}, which we briefly introduce here.

Constant folding evaluates constants, including nondeterministic symbols, and propagates them throughout the resulting formula, during encoding. In particular, we exploit the constant propagation technique to reduce the number of expressions associated with specific neuron computation procedures and activation function. 
Thus, we simplify the SSA representation, using local and 
recursive transformations, to remove functionally redundant expressions (for neuron computation procedures and activation functions) 
and redundant literals (for safety properties), as
\[\begin{array}{ll}
\label{eqnarray:transformations}
a \wedge \mathit{true} = a  & a \wedge \mathit{false} = \mathit{false} \\
a \vee \mathit{false} = a   & a \vee \mathit{true} = \mathit{true} \\
a \oplus \mathit{false} = a & a \oplus \mathit{true} = \neg a \\
ite\left(\mathit{true}, a, b\right) = a & ite\left(\mathit{false}, a, b\right) = b \\
ite\left(f, a, a\right) = a & ite\left(f, f \wedge a, b\right) = ite\left(f, a, b\right).
\end{array}\]
\noindent We apply such simplifications to reduce the size of the resulting formula and consequently achieve simplification within each time step and across time steps, during the encoding procedure of a neural net's implementation. In our experimental evaluation, in Section \ref{sec:fine-tuning-bmc-parameters-v2}, we have noticed substantial improvements using these simplifications in formulas, but we have not identified improvements using the constant propagation approach itself. It happens because neural net inputs are typically symbolic ones and not constants, as can be noticed in the illustrative example in Fig.~\ref{figure:neural-net-verification}, where incremental learning removed the activation functions for neuron $b$ and output $f$.

Slicing removes expressions that do not contribute to the checking procedure of a given safety property. It is an essential step to improve a program's verification procedure, considerably, in some cases~\cite{DBLP:phd/ethos/Morse15}. Our verification engine implements two slicing strategies in combination. First, it removes all instructions after the last assert in the set of SSA. Second, it collects all symbols (and their dependent symbols) in assertions and removes instructions that do not contribute to them. When used in combination, both slicing strategies ensure that unnecessary instructions are ignored during SMT encoding. As an example, the code in Fig.~\ref{figure:neural-net-verification}(a) can be considered. If we are interested in checking that neural net's output only, we could rewrite the final assert statement, in line 12, as $f<=4$. Consequently, such a modification do indicate that everything not involving $f$ does not cause an impact on the conclusion of the intended safety property. Based on such a scenario, the resulting SSA for the code in Fig.~\ref{figure:neural-net-verification}(a) would be sliced as 
\begin{flalign*} 
&x1 == nondet\_symbol(nondet0) \wedge y1 == nondet\_symbol(nondet1) \wedge \nonumber \\
&f1 == 3 * (int)x1 + (int)y1 \wedge f2 == (f1 < 0 \, ? \, 0 : f1) \wedge f2 <= 4,
\end{flalign*} 
\noindent where there is no presence of information (states) regarding neurons $a$ and $b$. In our experimental evaluation, in Section \ref{sec:fine-tuning-bmc-parameters-v2}, we have observed that \textit{slicing} can significantly reduce the resulting SMT solving time.

\begin{figure}[htb]
\centering
\begin{subfigure}[t]{.4\textwidth}
    \begin{tikzpicture}[<-, >=stealth', auto, semithick,
     level/.style={sibling distance=15mm}
    ]
    
    \node [circle,draw] (z){+}
      child {node (a) {$w_n x_n$}
      }
      child {node [circle,draw] (j) {+}
        child {node (k) {$w_{n-1} x_{n-1}$}
        }
      child {node (l) {$\vdots$}
          child {node (o) {$w_1 x_1$}}
          child {node (p) {$w_2 x_2$}}
      }
    };
      child [grow=down] {
        edge from parent[draw=none]
      };
    \draw [->] (z) -- ++(0,1cm);
    \end{tikzpicture}
\caption{Linear layout.}
\label{fig:linearlayout}
\end{subfigure}
\begin{subfigure}[t]{.4\textwidth}
    \begin{tikzpicture}[<-, >=stealth', auto, semithick,
     level/.style={sibling distance=15mm}
    ]
    
    \node [circle,draw] (z){+}
        child {node [circle,draw] (b) {+}
            child {node [circle,draw] {+}
                child {node (d) {$w_1 x_1$}}
                child {node (e) {$w_2 x_2$}}
                } 
            child {node {$\vdots$}}
        }
        child {node [circle,draw] (l) {+}
            child {node {$\vdots$}}
            child {node [circle,draw] (c){+}
                child {node (o) {$w_{n-1} x_{n-1}$}}
                child {node (p) {$w_n x_n$}
            }
        }
      };
    
    \draw [->] (z) -- ++(0,1cm);
    \end{tikzpicture}
\caption{Balanced layout.}
\label{fig:balancedlayout}
\end{subfigure}
\end{figure}

Expression balancing reduces the size of SMT formulae by reordering long chains of operations with the associative rule. This technique has been recently applied to neural networks by Giacobbe {\it et al.}~\cite{Giacobbe2020}, but has been used in compilers for decades. In brief, the computation of neuron potentials in ANNs requires a linear combination of the neuron inputs (see Equation \ref{eq:anncalc_aff}). Depending on the specific implementation, the resulting sequence of multiply-and-accumulate operations (MAC) in the code is translated to SMT formulae of different sizes. In the worst case, which is portrayed in Fig.~\ref{fig:linearlayout}, the formula size is linear in the number of MAC operations. Expression balancing ensures that the SMT formulae are always reordered as in the best case scenario shown in Fig.~\ref{fig:balancedlayout}. That is, the sequence of MAC operations is split over multiple accumulators in a divide-and-conquer fashion, yielding a set of semantically equivalent, but smaller SMT formulae. In Section \ref{sec:fine-tuning-bmc-parameters-v2}, we show that this associative balancing step is crucial in making ANN verification viable. Note that this result is consistent with those presented in~\cite{Giacobbe2020}. Furthermore, in Section \ref{sec:code-balancing} we show that, thanks to this balancing step, the performance of our verification methodology is stable across different implementations of the same ANN.


\subsection{Illustrative example: robustness to adversarial images}
\label{ssec:imagerobustnessexample}

We conclude this section with an illustrative example of our verification methodology. We do so in order to clarify the user's side of the workflow illustrated in Fig. \ref{fig:fp-nn-methodology}. Later, in Section \ref{sec:exp}, we report more details on the range of ANNs and safety properties that can be verified with our methodology, as well as the efficiency of doing so.

The present example, illustrated in Fig. \ref{fig:robustA}, shows how to verify a character recognition ANN. First, given a network's architecture and weights, in a high-level representation, as in Fig. \ref{fig:robustA_a}, such elements should be converted into single-threaded C code. This task can be achieved through the popular machine learning libraries PyTorch ~\cite{paszke2019pytorch} and Tensorflow ~\cite{abadi2016tensorflow}, or, like in many of our experiments, in Section \ref{sec:exp}, by converting from the mid-level representation NNet\footnote{\url{github.com/sisl/NNet}}. In this example, we use the neural network from our Vocalic benchmark (see Section \ref{sec:benchmark-vocalic}) quantized to a fixed-point representation with 8 integer (including sign) and 8 fractional bits.

\begin{figure}[htb]
	\centering
	\begin{subfigure}[b]{0.30\linewidth}
		\centering
		\includegraphics[height=0.5\linewidth]{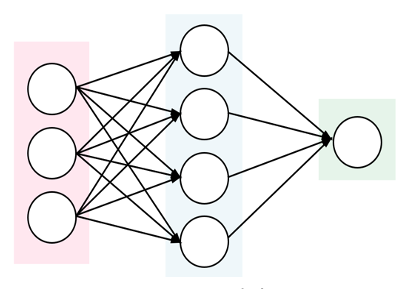}
				\caption{ANN structure and weights.}
		\label{fig:robustA_a}
	\end{subfigure}
	\begin{subfigure}[b]{0.15\linewidth}
		\centering
		\includegraphics[height=\linewidth]{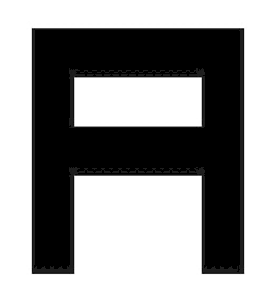}
				\caption{$\mathcal{H}$ center.}
		\label{fig:robustA_b}
	\end{subfigure}
	\begin{subfigure}[b]{0.15\linewidth}
		\centering
		\includegraphics[height=\linewidth]{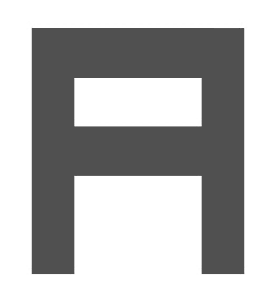}
				\caption{$\mathcal{H}$ lower.}
		\label{fig:robustA_c}
	\end{subfigure} 
	\begin{subfigure}[b]{0.15\linewidth}
		\centering
		\includegraphics[height=\linewidth]{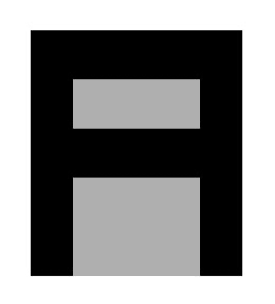}
		    \caption{$\mathcal{H}$ upper.}
		\label{fig:robustA_d}
	\end{subfigure}
	\begin{subfigure}[b]{0.15\linewidth}
		\centering
		\includegraphics[height=\linewidth]{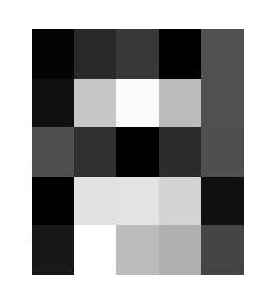}
		    \caption{Counterex.}
		\label{fig:robustA_e}
	\end{subfigure}
	\caption{Inputs and outputs of our verification approach: (a) ANN implementation; safety property, where the (b) center, (c) lower, and (d) upper extremes of the input region are shown; and (e) a counterexample that violates the safety property.}
	\label{fig:robustA}
\end{figure}

Second, the ANN source code undergoes a further sequence of transformations. Initially, we replace all floating-point arithmetic operations with the corresponding fixed-point implementation models (see Section \ref{sec:fixedpointOM}), given that our ANN is quantized. Then, we also replace any sigmoid, hyperbolic tangent, or piecewise-linear activation function with its corresponding discretized look-up table  (see Section \ref{ssec:discretise}).

Third, a safety property is encoded by adding the corresponding pair of \texttt{assume} and \texttt{assert} instructions. In the present example, we check for robustness around a specific input image, which we show in Fig. \ref{fig:robustA_b}. More formally, we define the input region of our safety property (premise) as a set $\mathcal{H}=\{x:|x-x^d|_{\infty}\leq c\}$, where the centre point $x^d$ corresponds to the $5\times 5$ pixel values in the image of the ideal character ``A'', i.e., without deviation, in Fig. \ref{fig:robustA_b}, and $c=80$. For reference, we report the lower and upper bounds of $\mathcal{H}$ in the gray image pixel domain, in Figs. \ref{fig:robustA_c} and \ref{fig:robustA_d}, respectively.

Likewise, we set the output region of the safety property (conclusion) as the set of all outputs that assign a higher score $y_A>y_k,\forall k\neq A$ to class ``A'' than to any other output classes. Note that the final softmax layer, typically included in classification ANNs, can be omitted for our purposes since it is a monotonic function of the score of each class~\cite{bishop2006PRML}. After this, a static analysis tool such as FRAMA-C \cite{blanchard2018} propagates the input region $\mathcal{H}$ through the associated ANN code and annotates it with additional \texttt{assume} instructions, representing the reachable values-interval of each intermediate variable (see Section \ref{sssec:invariant}).

Fourth, annotated C code goes through a model checker that tries to falsify a safety property. In our experiments (see Section \ref{sec:exp}), we have used ESBMC to do so, as it is a good representative of state-of-the-art SMT model checkers~~\cite{GadelhaMMC0N18,GadelhaMCN19}. If a given safety property can not be verified, ESBMC returns a counterexample that falsifies it, which represents a potential adversarial attack on a neural network. In the present example, ESBMC does indeed report such a counterexample, which we show in Fig. \ref{fig:robustA_e}. More adversarial examples can be seen in Figs. \ref{fig:imagequality8bits}, \ref{fig:imagequality16bits}, and \ref{fig:imagequality32bits}, for a wide range of safety properties and quantization granularities of our character recognition ANN.

\section{Experimental Evaluation}
\label{sec:exp}

In this section, we test the performance of the verification approach we introduced in Section \ref{sec:verification}. In this regard, we are mainly interested in the following research questions:

\begin{enumerate}
\item[\textbf{RQ1}] \textbf{- Ablation study -} Is it possible to establish the role of each of the enhancement techniques introduced in Section \ref{sec:verification} and also define an optimal setup, both regarding total verification time and performance?
\item[\textbf{RQ2}] \textbf{- Quantization effects -} How does a quantization choice influence our verification process and the safety of a neural network?
\item[\textbf{RQ3}] \textbf{- Comparison with SOTA techniques -} What is the performance of our verification approach when compared to the existing literature?
\end{enumerate}


Regarding \textbf{RQ1}, since those techniques were first introduced for software verification in general, we are interested, in particular, in finding their optimal configuration to verify ANNs, including contribution and general setup. In addition, \textbf{RQ2} is related to quantization of ANNs, which is in the core of the present work and have the potential to provide a methodology regarding integration into target platforms. Moreover, if we were to verify the same property for different quantization levels, would we observe any difference in verification time or outcome? Finally, regarding \textbf{RQ3}, it is always of paramount importance to position a given approach among the existing scientific knowledge.

We present our answers to those questions in the following way. In Section \ref{sec:quantization-aspects}, we discuss a configuration step regarding quantization and also general data processing to provide adaptation and avoid overflow in ANN operations. In Section \ref{sec:description-of-benchmarks-v2}, we describe the datasets and ANNs that constitute our verification benchmarks, including the necessary minimum number of bits for correct data-range representation. In Section \ref{sec:ablation-study-v2}, we isolate the contribution of each component of our verification approach and propose the configuration that yields the best results performance-wise, which answers \textbf{RQ1}. In Section \ref{sec:verification-of-anns-with-fwl-implementation-v2}, we compare the performance and output of our verification approach across different quantization levels of the same problem, which addresses \textbf{RQ2}, while analyzing important aspects and general behavior and also providing guidance on integration into restricted platforms. In Section \ref{sec:comparison-with-sota-v2}, we compare our verification framework with the most popular SOTA approaches, which fulfills \textbf{RQ3}. Finally, in Section \ref{sec:limitations-v2}, we list the remaining limitations towards large-scale verification of fixed-point ANNs. All benchmarks, tools, and results associated with the current evaluation are available for download at \url{https://tinyurl.com/6y7e49vk}.

\subsection{Quantization aspects and data adaptation}
\label{sec:quantization-aspects}
As mentioned at the end of Section \ref{sec:fixedpointOM}, when correctness comes into play, not every quantization format can be used. Indeed, if a format that is not suitable to the target ANN is chosen, overflow will likely occur, compromising operation results and general ANN output. Nonetheless, a designer can also incur severe quantization and suppose that errors due to wrong operations are an acceptable side effect (even under frequent overflow). Still, our goal is to provide compression that results in quantization error only, then preserving an ANN's associated dynamic range and correct computation of operations in neurons.

Another aspect is that input data may present a broad diversity of dynamic ranges. As a consequence, they are usually processed in scaled format. In our framework, input data is first normalized to the range $[0,1]$ and then fed to a given ANN (also for training). This way, the initial (input) dynamic range is always known.

Consequently, it is essential to analyze neurons in a given ANN and then identify the minimum and maximum associated values resulting from their processing, given input date in the range $[0,1]$, which will define the minimum number of bits for the integer part of a given representation. It does not specify maximum compression because it only intends to represent the existing dynamic range and avoid overflow correctly. Besides, we should also check the number of bits for the fractional part to provide the desired accuracy.

Note that the discovery of the minimum number of bits for the integer part is made by using Eq.~\eqref{eq:genpnorm}, with $p=1$, and taking into account all weights of each neuron to find the maximum magnitude. Alternatively, FRAMA-C \cite{blanchard2018} can also be used, as it reveals intervals associated with variables in ANN code.

\subsection{Description of the benchmarks}
\label{sec:description-of-benchmarks-v2}

In our evaluation, we consider ANNs trained on two datasets: the UCI Iris dataset~\cite{Dua:2019} and a vocalic character recognition dataset~\cite{Sena20}. This section gives the details regarding the employed datasets, the neural networks we trained on top of them, the safety properties that we used to test our verification approach, and, finally, our general experimental setup.

\subsubsection{Iris benchmark}
\label{sec:benchmark-iris}

The Iris dataset~\cite{Dua:2019} consists of $50$ samples from each of three species of Iris (Iris setosa, Iris virginica, and Iris versicolor). This dataset contains both the length and width of the sepals and petals in centimeters (our inputs) and the iris specie label (our output). Here, we use TensorFlow version 1.4~\cite{abadi2016tensorflow} and keras~\cite{gulli2017deep} to train a feedforward neural network with layers of $4\times7\times3$ neurons, hyperbolic tangent activation functions, and softmax output layer. We train such a neural network to predict the correct Iris species with the backpropagation algorithm and cross-validation~\cite{bishop2006PRML}. When quantizing the ANN to fixed-point arithmetic, we followed what was presented in Section \ref{sec:quantization-aspects}. We found that the maximum neuron output was bounded, in modulus, by $23.3$. Consequently, we allow for $6$ integer bits, including sign, as they are required to avoid overflow. In terms of safety properties, we specify hyper-rectangular input regions for each species: \textit{setosa}, \textit{versicolor}, and \textit{virginica}. We identify the center of these regions from the dataset with the granular fuzzy clustering algorithm in~\cite{Cordovil2020}. Then, for each of the four input variables, we computed its maximum range. With it, we generated nine regions $R_s$ for each class, sharing the same center but with different sizes $s\in\{1,2,5,8,10,20,30,40,50\}$ of the hyperrectangle surrounding it, where $s$ is a percentage representing the fraction of the maximum input range.

\subsubsection{Vocalic benchmark}
\label{sec:benchmark-vocalic}

The vocalic dataset~\cite{Sena20} consists of $200$ gray-scale images with dimensions $5\times5$ pixels. Half of the dataset consists of the base images illustrated in Fig.~\ref{fig:vowels} and also noisy versions of them. In contrast, the other half presents non-vocalic images. With it, we have trained a feedforward neural network with architecture $25\times10\times4\times5$ and sigmoid activation functions. As Fig.~\ref{fig:vowels} shows, there are five output classes that this network learned to discriminate via backpropagation algorithm and cross-validation. Once again, we have followed what was presented in Section \ref{sec:quantization-aspects} and found $53.9$ as maximum neuron output. Consequently, we have quantized this ANN to fixed-point arithmetic with a minimum of $7$ integer bits, including sign, as they are required to avoid overflow. As far as the safety properties are concerned, we specify five hypercubic input regions corresponding to the vocalic labels. The centers are defined by the base images in Fig.~\ref{fig:vowels}. Similarly to the Iris benchmark, we generate five instances $L_s$ of these regions with different sizes $s\in\{10,20,40,80,120\}$, where $s$ represents the hypercube's side length.

\begin{figure}[htb]
	\centering
	\begin{subfigure}[b]{0.08\linewidth}
		\centering
		\includegraphics[width=\linewidth]{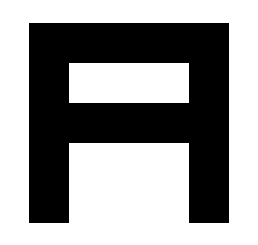}
		\label{fig:A}
	\end{subfigure}
	\begin{subfigure}[b]{0.08\linewidth}
		\centering
		\includegraphics[width=\linewidth]{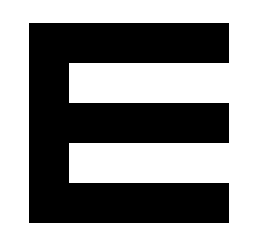}
		\label{fig:E}
	\end{subfigure}
	\begin{subfigure}[b]{0.08\linewidth}
		\centering
\includegraphics[width=\linewidth]{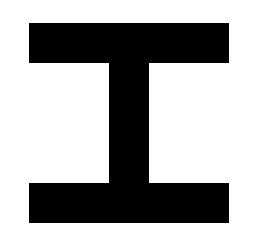}
\label{fig:I}
	\end{subfigure}
\begin{subfigure}[b]{0.08\linewidth}
		\centering
\includegraphics[width=\linewidth]{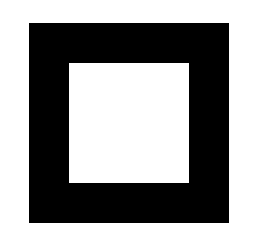}
\label{fig:O}
\end{subfigure}
\begin{subfigure}[b]{0.08\linewidth}
		\centering
\includegraphics[width=\linewidth]{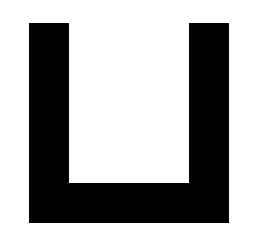}
\label{fig:U}
\end{subfigure}
	\caption{Vocalic images in benchmarks.}\label{fig:vowels}
\end{figure}

\subsubsection{AcasXu benchmark}
\label{sec:benchmark-acasxu}

The Acas Xu benchmark~\cite{julian2016policy} is the result of avionics research in airborne collision avoidance systems (ACAS) for unmanned aircrafts (Xu). In particular, when avoiding a nearby aircraft, some specific piloting decisions must be taken. These are recorded in a large state-action table that is impractical to store on-board due to its memory requirements. The Acas Xu benchmark splits and compresses such a table into a set of $45$ neural networks. The split is done by discretizing the following two input dimensions: time until loss of vertical separation ($9$ intervals), and previous advisory action ($5$ actions). The remaining $5$ inputs are fed into a fully-connected feedforward neural network with ReLU activation functions and architecture $5\times300\times300\times300\times300\times300\times300\times5$, which outputs a prediction for each of the $5$ possible actions. We quantize all these $45$ ANNs with $27$ integer bits, which is the least number of bits required to avoid overflow in the worst-case scenario, i.e., with a neuron output of $72142560.0$, as pointed out by FRAMA-C \cite{blanchard2018}. More details on its associated safety properties can be found in~\cite{katz2017reluplex} and in Section \ref{sec:comparison-with-sota-v2}.

\subsubsection{Experimental setup}
\label{sec:setup}

We have conducted our experimental evaluation on a Intel(R) Xeon(R) CPU E$5$-$2620$ v$4$ @ $2$.$10$GHz with $128$ GB of RAM and Linux OS. All presented execution times are CPU times, i.e., only the elapsed periods spent in allocated CPUs, which was measured with the \texttt{times} system call~\cite{monteiro2018esbmc}. All experimental results reported here were obtained by executing ESBMC v$6$.$6$.$0$\footnote{Available at \url{http://esbmc.org/}} with the following command line parameters, unless specifically noted: \texttt{esbmc <file.c> -I <path-to-OM> --force-malloc-success --no-div-by-zero-check --no-pointer-check --yices --no-bounds-check --interval-analysis --fixedbv}. In general, we let ESBMC run without time or memory limits. The timeouts reported in the following experiments are all due to exceedingly high memory consumption. All of our benchmarks have been annotated with the reachable intervals provided by FRAMA-C, unless specifically noted. In particular, we executed FRAMA-C using the following command: \texttt{frama-c -eva -eva-plevel 255 -eva-precision 11}.

\subsection{Ablation study}
\label{sec:ablation-study-v2}

This section aims at evaluating the impact of different aspects of our approach on the total verification time. Here, our aim is both to discover the best configuration for our verification tool and shed some light on the importance of each technique for reducing the search space of the verification problem. Specifically, we address four choices in our verification approach: SMT solver, optional parameters offered by the ESBMC verification engine, interval analysis technique and expression balancing strategy.

\subsubsection{SMT solvers comparison}
\label{sec:smt-solvers-comparison-v2}

As mentioned in Section~\ref{sec:verification}, our approach relies on model checking to reason about the satisfiability of a given safety property concerning an ANN implementation. For the experiments of the present section, we have chosen ESBMC as our verification engine since it has been extensively evaluated at various SV-Comp~\cite{beyer2021software} competitions, where it has consistently achieved state-of-the-art results~\cite{GadelhaMCN19}. More in detail, the ESBMC model checker takes care of converting input C code into SMT formulae and then calls an external SMT solver. Currently, ESBMC supports four solvers: Bitwuzla, Boolector, Yices, and Z3. In general, they yield different verification results, both in terms of the generated counterexample (if any) and verification time.

Here, we are interested in comparing the performance of such solvers in verifying ANN implementations. To this end, we run them on all our fixed-point benchmarks, with word lengths of 8, 16 and 32 bits. With this choice, we cover the most popular quantization lengths, and observe the behaviour of our verification methodology on a varied test suite. We use these experimental settings all throughout our ablation study (see also Sections \ref{sec:interval-analysis-comparison}, \ref{sec:fine-tuning-bmc-parameters-v2} and \ref{sec:code-balancing}). 

The results of our comparison are summarized in Fig.~\ref{fig:ablation-smt-solvers}. There, we can see that solvers Bitwuzla and Boolector have nearly identical performance, in terms of verification time (Fig.~\ref{fig:ablation-bitwuzla-boolector}). In contrast, Yices exhibits a considerable advantage across the whole verification suite, being, in some specific cases, even two orders of magnitude faster (Fig.~\ref{fig:ablation-yices-boolector}). Finally, solver Z3 struggled to complete the majority of verification runs, and it is, in general, orders of magnitude slower than the other three solvers. For this reason, we do not portray its results in Fig.~\ref{fig:ablation-smt-solvers}.
%
\begin{figure}[htb]
\centering
    \begin{subfigure}[b]{0.48\linewidth}
		\centering
		    \includegraphics[width=\linewidth]{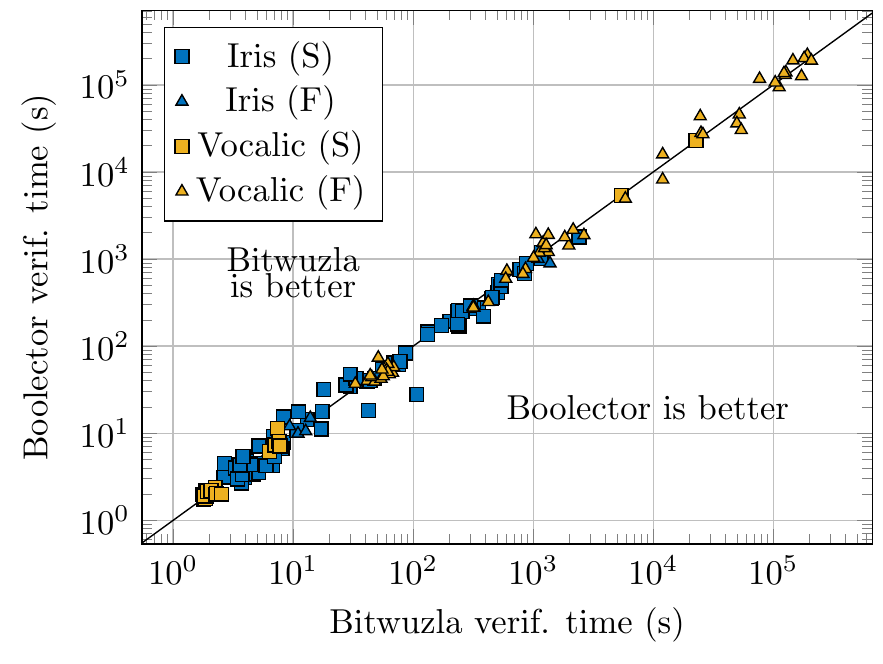}
		\caption{}
		\label{fig:ablation-bitwuzla-boolector}
	\end{subfigure}
    \begin{subfigure}[b]{0.48\linewidth}
		\centering
		    \includegraphics[width=\linewidth]{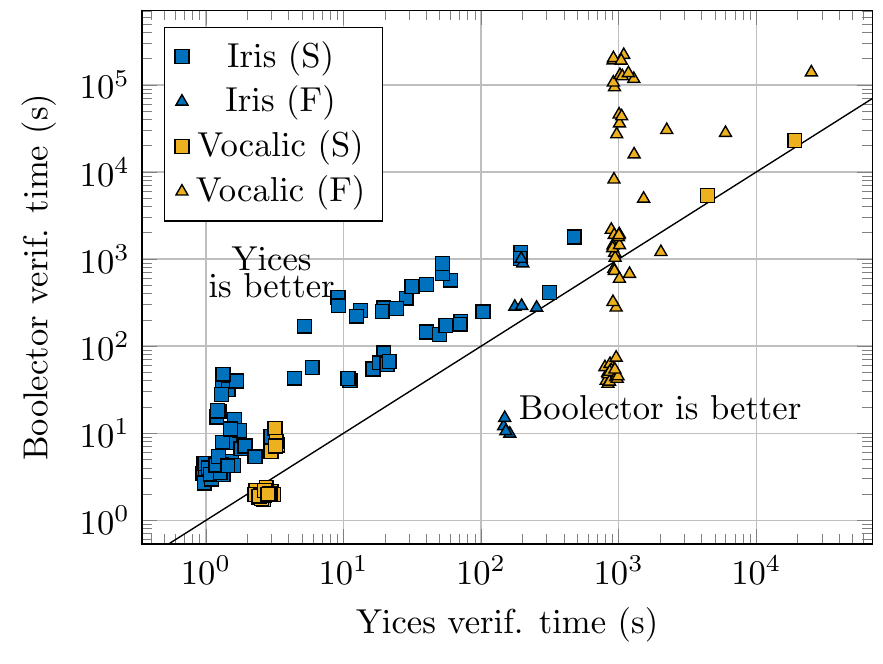}
		\caption{}
		\label{fig:ablation-yices-boolector}
	\end{subfigure}
\caption{Comparison with different SMT solvers, regarding verification time (in seconds), when handling the fixed-point Iris and Vocalic benchmarks. On the left, (a) Bitwuzla and Boolector show similar performance; on the right, (b) Yices is considerably faster than Boolector, in most instances. In both plots, we discriminate between successful verification outcomes (S) and falsifiable safety properties that admit a counterexample (F).}
\label{fig:ablation-smt-solvers}
\end{figure}

Given the results in Fig.~\ref{fig:ablation-smt-solvers}, we choose Yices as our underlying SMT solver for the rest of this experimental section. While it is impossible to know exactly why Yices is the best-performing solver on our test suite, we speculate it is a consequence of the fact that ESBMC encodes verification problems into SMT formulae with the formalism of \textit{QF\_AUFBV} logic.\footnote{\url{https://smtlib.cs.uiowa.edu/logics.shtml}} Here, \textit{QF} stands for quantifier-free formulas, \textit{A} stands for the theory of arrays, \textit{UF} stands for uninterpreted functions, and \textit{BV} stands for the theory of fixed-sized bit-vectors. For this type of formulae, Yices represents the state-of-the-art SMT solver.~\footnote{\url{https://smt-comp.github.io/2020/results/qf-aufbv-single-query}}

\subsubsection{Comparison regarding ESBMC's parameters }
\label{sec:fine-tuning-bmc-parameters-v2}

In Sections \ref{sec:IncrementalVerificationusingInvariantInference} and \ref{sec:SearchSpaceReduction}, we have presented a number of state-of-the-art software verification techniques that apply to ANN implementations. From our prior experience of participating in software verification and testing competitions (e.g., SV-COMP and Test-Comp), such techniques play an essential role in optimizing the performance of ESBMC on a given set of benchmarks~\cite{MorseRCN014,GadelhaMCN19,GadelhaMMCN20}. In the present section, we quantify their individual impact on verification times of our test suite and comment on their relative performance.

Here, we rely on the fact that the ESBMC's verification engine allows us to toggle each separate technique via command-line parameters. More specifically, the list of verification techniques and corresponding ESBMC parameters are as follows:
\begin{itemize}
    \item \textbf{Constant propagation.} It can be disabled with the option \texttt{no-propagation}. Otherwise, it will generate a minimal set of SSAs in the symbolic engine.
    \item \textbf{Slicing.} It can be disabled with the option \texttt{no-slice}. Otherwise, it will eliminate redundant or irrelevant portions of a program~\cite{de2001program}. In ESBMC, this is applied to the SSA program before it is encoded to SMT to reduce the number of variable assignments by identifying variables not used to evaluate any property assertion.
    \item \textbf{Incremental verification.} Activated with the (experimental) options \texttt{smt-during-symex} and \texttt{smt-symex-guard}. The former enables incremental SMT solving using the SMT solvers Yices or Z3, the latter allows calls to the solver during symbolic execution to check the satisfiability of the guards.
    \item \textbf{Expression simplification.} It can be disabled with the option \texttt{no-simplify}, effectively neutering constant propagation so that no fact is statically determined to be true or false, and always end up exploring to the top of the unwind bound.
\end{itemize}

\begin{figure}[htb]
\centering
    \begin{subfigure}[b]{0.32\linewidth}
		\centering
		    \includegraphics[width=\linewidth]{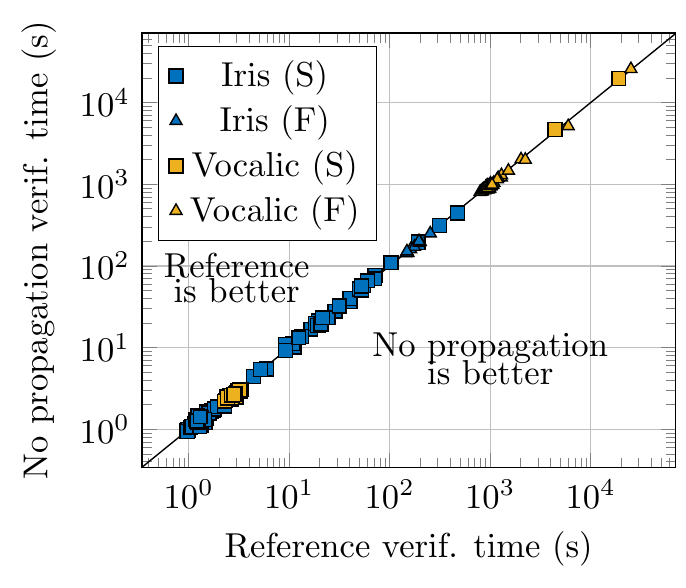}
		\caption{}
		\label{fig:esbmc-noprop}
	\end{subfigure}
    \begin{subfigure}[b]{0.32\linewidth}
		\centering
		    \includegraphics[width=\linewidth]{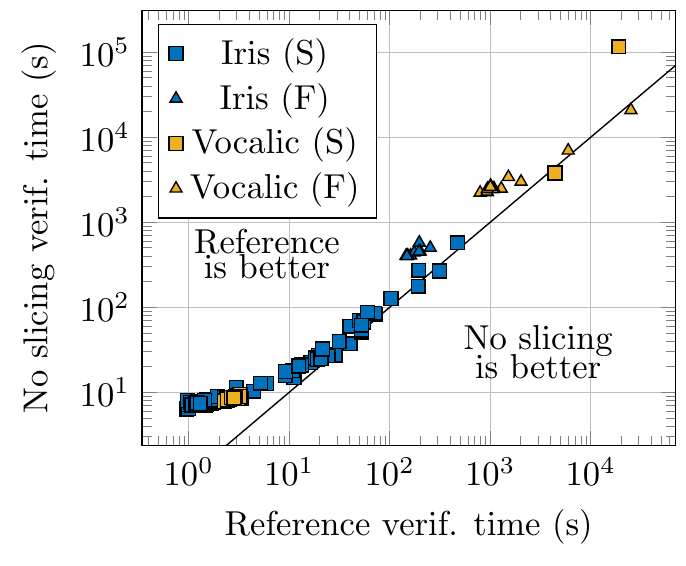}
		\caption{}
		\label{fig:esbmc-noslice}
	\end{subfigure}
	\begin{subfigure}[b]{0.32\linewidth}
		\centering
		    \includegraphics[width=\linewidth]{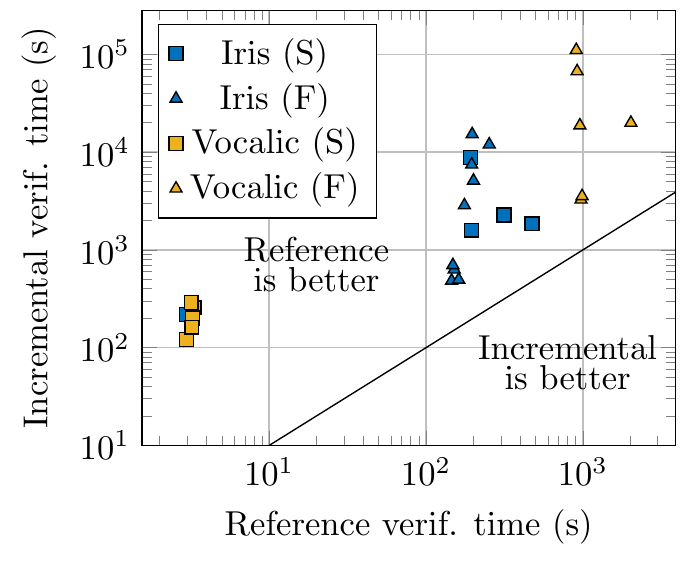}
		\caption{}
		\label{fig:esbmc-increment}
	\end{subfigure}
\caption{Comparison of verification times of ESBMC with different parameters settings on the fixed-point Iris and Vocalic benchmarks. In each figure, one individual technique has been changed from the best (reference) configuration: (a) disabling constant propagation, (b) disabling slicing, and (c) enabling incremental verification. In the plots, we discriminate between successful verification outcomes (S) and falsifiable safety properties that admit a counterexample (F).}
\label{fig:esbmc-parameters-ablation}
\end{figure}

Here, we quantify the impact of each technique on the same test suite of Section \ref{sec:smt-solvers-comparison-v2}. We do so by setting a reference configuration and toggling one verification technique at a time. For reasons that become clear from the results shown in Fig.~\ref{fig:esbmc-parameters-ablation}, our reference configuration of ESBMC has constant propagation, slicing and expression simplification enabled. In contrast, we choose to keep incremental verification disabled.

As the results in Fig.~\ref{fig:esbmc-noprop} show, constant propagation makes no difference on our test suite. This is because we are verifying a specific kind of safety property, namely robustness to adversarial examples, which allows all input variables to be modified. As such, there is no constant input that can be propagated through the ANN code, thus yielding no reduction in the SMT formulae size. At the same time, we believe that constant propagation is a useful technique for safety properties that restrict the attack surface to just a subset of the input variables, as the ones identified by Karmon, Zoran, and Goldberg~\cite{pmlr-v80-karmon18a}.

In contrast, Fig.~\ref{fig:esbmc-noslice} shows that slicing yields a small improvement in performance, which becomes the more significant the shorter the verification time is. We speculate that this is because neural networks are usually redundant (e.g., see dropout~\cite{bishop2006PRML}), and thus the majority of neurons contribute to the ANN output. As a consequence, only a small number of expressions can be removed with slicing.

Interestingly, incremental verification (cf.~Fig.~\ref{fig:esbmc-increment}) does not improve verification time as expected. We believe this happens because the cost of deriving and storing new facts during the verification process outweighs the reduction in search space they induce since it performs various calls to the solver. Still, we hypothesize that incremental verification may offer some advantages when verifying not only one but also a whole set of safety properties since it allows incrementally remembering important facts across properties, whose net contribution may pay off. For example, we could perform a query at any neuron using incremental lemma learning, which could help prune neural net implementation before deploying it to an embedded device with time, memory, and energy constraints. However, we leave the exploration of such a hypothesis for future work.

Finally, expression simplification is crucial in making the verification of our test suite practical. Indeed, without expression simplification, none of the safety properties could be checked before hitting our machine memory limit of 128GB, despite letting the verification process run without any time limit.

\subsubsection{Interval analysis comparison}
\label{sec:interval-analysis-comparison}

In Section~\ref{sssec:invariant}, we introduced interval analysis as an essential pre-processing stage before running the verification engine on ANN code. Here, we show the effect of disabling such an important step on total verification times. Furthermore, we compare two approaches to interval analysis and discuss their results. The first requires FRAMA-C \cite{blanchard2018} to annotate ANN code with additional \texttt{assume} instructions. In contrast, the second requires running ESBMC with the extra \texttt{--interval-analysis} option enabled. Note that both of them compute hyper-rectangular constraints over program variables.

For consistency with the previous experiments, we evaluate the impact of these two interval analysis options on the same test suite as in Sections~\ref{sec:smt-solvers-comparison-v2} and~\ref{sec:fine-tuning-bmc-parameters-v2}. We present the results in Fig.~\ref{fig:interval-analysis-comparison}, where the native \texttt{--interval-analysis} option and the externally computed intervals by FRAMA-C are compared with our reference configuration of ESBMC without any form of interval analysis. Note how the former has almost no impact on the verification time, while the latter can improve it by up to two orders of magnitude. Still, regarding the use of FRAMA-C, it is interesting to notice that we only observe improvement on successful safety properties (S), i.e., those that do not admit a counterexample. This way, the verification time of falsifiable properties (F) does not appear to be improved by interval analysis on our test suite.

On the one hand, as no counterexample is found, the FRAMA-C's more sophisticated interval analysis indeed pays off, given the apparent reduction in the state space that must be explored. On the other hand, when a property is falsifiable, that seems to be easily identified in the proposed framework and adopted test suite. As future work, we can perform a  deep analysis of that matter and then even propose improvements in this interval analysis focused on ANN code and properties. Note that the intervals produced by ESBMC work only for integer variables~\cite{esbmc-tacas-2019}, while Frama-C can make intervals for integer and floating-point ones~\cite{buhle2017}. Since our benchmarks contain heavily floating-point computations, we expected Frama-C to improve our verification results considerably compared to the interval analysis implemented in ESBMC, particularly for safe neural nets due to the state-space size.


\begin{figure}[htb]
\centering
	\begin{subfigure}[b]{0.48\linewidth}
		\centering
		    \includegraphics[width=\linewidth]{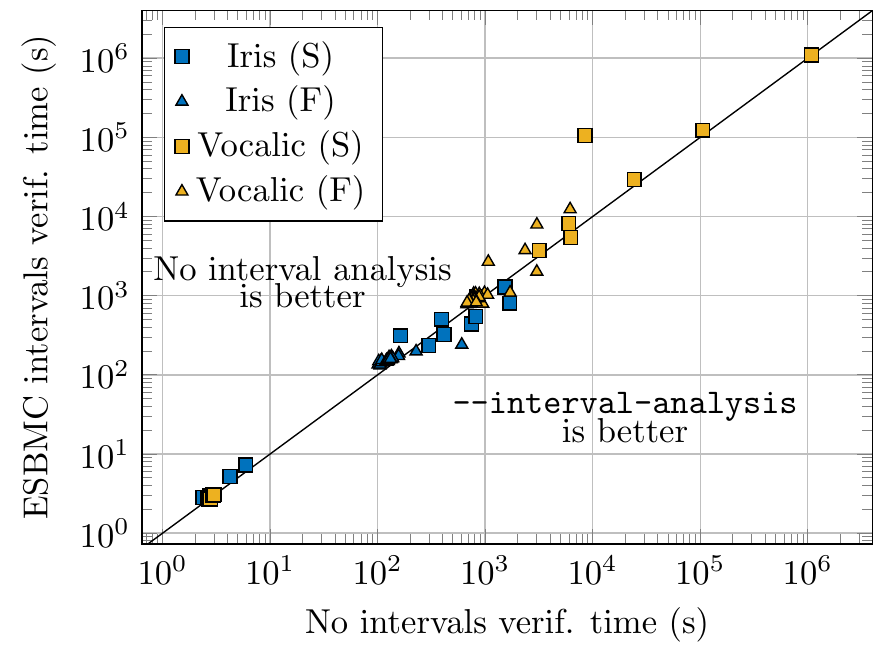}
		\caption{}
		\label{fig:interval-analysis-esbmc}
	\end{subfigure}
	\begin{subfigure}[b]{0.48\linewidth}
		\centering
		    \includegraphics[width=\linewidth]{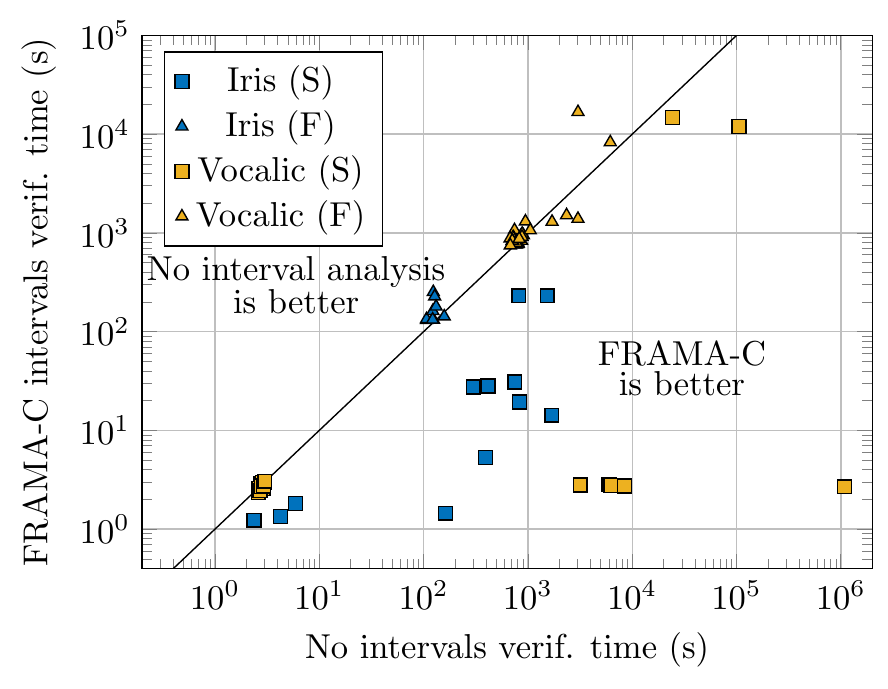}
		\caption{}
		\label{fig:interval-analysis-frama-c}
	\end{subfigure}
\caption{Comparison of verification times with and without interval analysis on the fixed-point Iris and Vocalic benchmarks. On the left, (a) enabling the native \texttt{--interval-analysis} option in ESBMC does not yield much improvement; on the right, (b) adding the intervals computed by FRAMA-C reduces verification times of a large number of safety properties. In both plots, we discriminate between successful verification outcomes (S) and falsifiable safety properties that admit a counterexample (F).}
\label{fig:interval-analysis-comparison}
\end{figure}

Such performance improvement is in line with our previous experiments over a large set of open-source software benchmarks when enabling invariant generation~\cite{GadelhaMCN19}. In particular, in the mentioned study, invariant generation based on intervals allowed us to verify $7$\% more programs using a \textit{k}-induction proof rule. Therefore, we chose to use both the \texttt{--interval-analysis} option in ESBMC and the FRAMA-C's intervals for the upcoming experiments.

\subsubsection{Activation function discretization comparison}
\label{sec:lookup-table}

The Iris and Vocalic benchmarks we use in the present ablation study are based on neural networks with sigmoid and hyperbolic tangent activation functions (see detailed descriptions in Sections \ref{sec:benchmark-iris} and \ref{sec:benchmark-vocalic}). An important step in our verification methodology is the discretization of such functions, as explained in Section \ref{ssec:discretise}. In practical terms, it means replacing the non-linear mathematical expression of the activation function with a look-up table. Here, we show the impact of the resolution of such look-up table on verification times, and how the error we introduce with the discretization influences the verification outcome.

To this end, we compare three different resolutions of our look-up tables, which we call \textit{Res1}, \textit{Res2}, and \textit{Res3}. These discretize the input interval $[-6,+6]$ with one, two, or three decimal fractional places, respectively. Outputs for inputs that fall outside that range are automatically saturated to $0$ or $1$ for the sigmoid function and $-1$ or $+1$ for the hyperbolic tangent one. We report the corresponding results on the Iris and Vocalic benchmarks with $8$, $16$, and $32$ bits, all condensed in Fig.~\ref{fig:lookup-table-time}. Although coarser resolutions usually result in faster verification times, as expected, given the inherent speed-up in operations, one may also notice some outliers: all regarding the Iris benchmark, when comparing \textit{Res1} with \textit{Res2}, and a mixture of Iris and Vocalic benchmarks, when comparing \textit{Res2} with \textit{Res3}. This is because different look-up table resolutions affect the computation of each neuron's output, and, in some cases, even the ANN's output itself (see example in Section \ref{ssec:quantisedann}). Consequently, a given violation that happened early during state-space exploration may then occur later or may not be even identified anymore, thus introducing a lot of variability in the verification time.

A more outcome-oriented comparison is presented in Table \ref{fig:lookup-table-outcome}. As one can notice, the verification outcome is indeed affected by the resolution choice. In fact, comparing \textit{Res1} and \textit{Res2} on the Vocalic benchmark yields one instance where the two verification runs disagree: \textit{Res1} reports a falsifiable property with a counterexample (F), whereas suh counterexample disappears with the finer resolution \textit{Res2} and the property is declared safe (S). Unfortunately, if we increase the resolution further to \textit{Res3}, the additional computational requirements overwhelm our verification setup, and we begin to observe a number of time-outs. This is more noticeable for the Vocalic benchmarks, because they employ a larger ANN.

\begin{figure}[htb]
\centering
	\begin{subfigure}[b]{0.48\linewidth}
		\centering
		    \includegraphics[width=\linewidth]{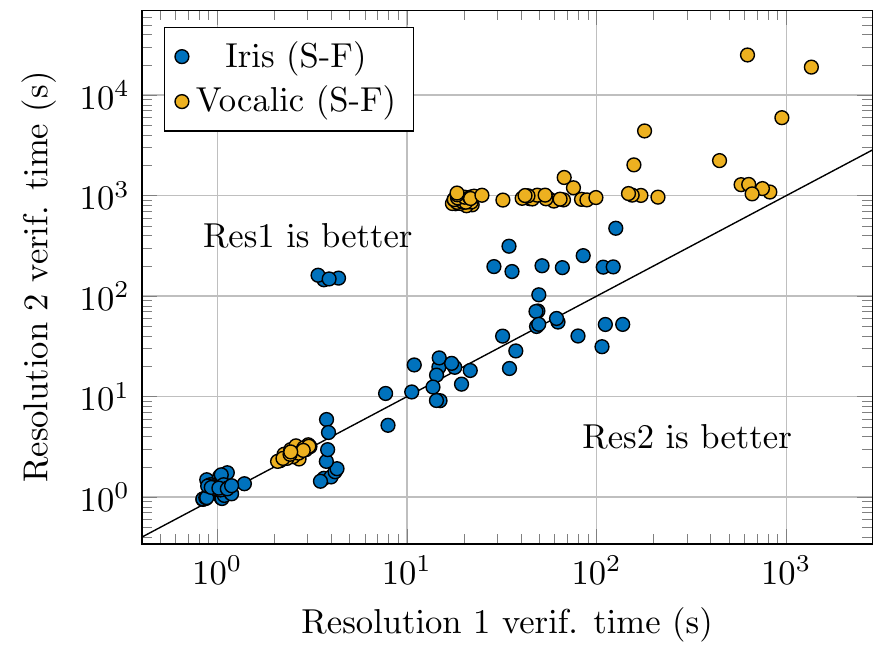}
		\caption{}
		\label{fig:lookup-table-time-res1-res2}
	\end{subfigure}
	\begin{subfigure}[b]{0.48\linewidth}
		\centering
		    \includegraphics[width=\linewidth]{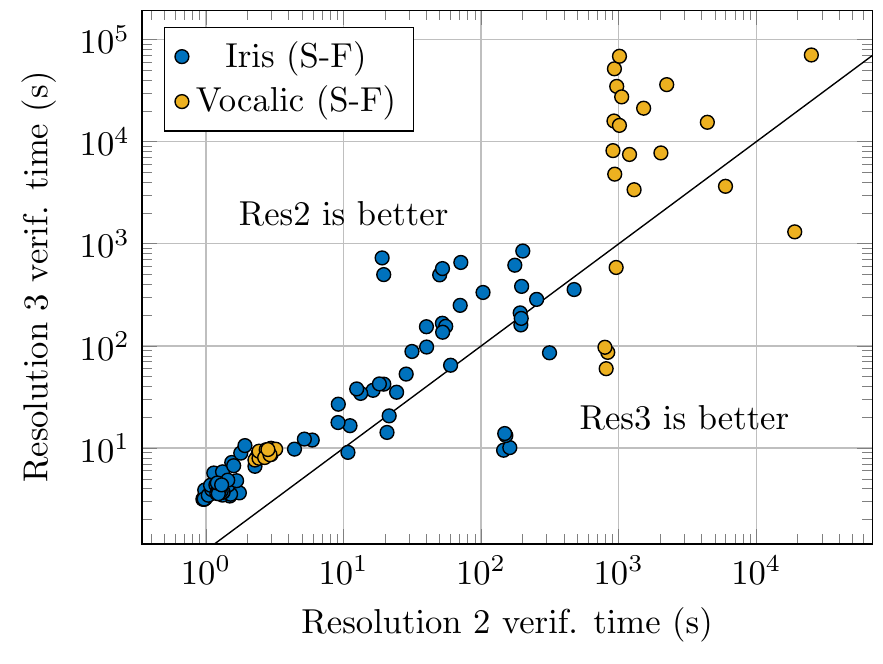}
		\caption{}
		\label{fig:lookup-table-time-res2-res3}
	\end{subfigure}
\caption{Comparison of verification times with different discretization resolutions for activation functions on the fixed-point Iris and Vocalic benchmarks. On the left, (a) comparison between one and two decimal places; on the right, (b) comparison between two and three decimal places. In both plots, we only report benchmarks that did not incur in timeout.}
\label{fig:lookup-table-time}
\end{figure}

In conclusion, choosing the right discretization resolution is a trade-off between verification time and possible errors in verification outcomes. In the ablation study in Section \ref{sec:ablation-study-v2} and the later quantization experiments in Section \ref{sec:verification-of-anns-with-fwl-implementation-v2}, we choose the intermediate resolution \textit{Res2}, based on two main reasons. First, it is the finest resolution that does not incur in large amounts of timeout when verifying our benchmarks. Second, all the counterexamples generated with it are valid, as we confirmed by running them through a non-discretized MATLAB implementation of the corresponding neural networks.

\begin{table}[htb]
\centering
    \begin{subtable}[b]{0.48\linewidth}
    \centering
        \begin{tabular}{|c|c|c|c|c|}
    	\hline
    	\multicolumn{2}{|c|}{Iris} & \multicolumn{3}{c|}{Res2}\\
    	\cline{3-5}
    	\multicolumn{2}{|c|}{Dataset} & S & F & TO\\
    	\hline
    	\multirow{3}{*}{Res1} & S & 72 & 0 & 0\\
    	\cline{2-5}
    	& F & 0 & 9 & 0\\
    	\cline{2-5}
    	& TO & 0 & 0 & 0\\
    	\hline
    	\multicolumn{5}{c}{}\\
    	\hline
    	\multicolumn{2}{|c|}{Vocalic} & \multicolumn{3}{c|}{Res2}\\
    	\cline{3-5}
    	\multicolumn{2}{|c|}{Dataset} & S & F & TO\\
    	\hline
    	\multirow{3}{*}{Res1} & S & 21 & 0 & 0\\
    	\cline{2-5}
    	& F & 1 & 53 & 0\\
    	\cline{2-5}
    	& TO & 0 & 0 & 0\\
    	\hline
        \end{tabular}
    \caption{}
    \label{fig:lookup-table-outcome-res1-res2}
    \end{subtable}
    \begin{subtable}[b]{0.48\linewidth}
    \centering
        \begin{tabular}{|c|c|c|c|c|}
    	\hline
    	\multicolumn{2}{|c|}{Iris} & \multicolumn{3}{c|}{Res3}\\
    	\cline{3-5}
    	\multicolumn{2}{|c|}{Dataset} & S & F & TO\\
    	\hline
    	\multirow{3}{*}{Res2} & S & 70 & 0 & 2\\
    	\cline{2-5}
    	& F & 0 & 0 & 9\\
    	\cline{2-5}
    	& TO & 0 & 0 & 0\\
    	\hline
    	\multicolumn{5}{c}{}\\
    	\hline
    	\multicolumn{2}{|c|}{Vocalic} & \multicolumn{3}{c|}{Res3}\\
    	\cline{3-5}
    	\multicolumn{2}{|c|}{Dataset} & S & F & TO\\
    	\hline
    	\multirow{3}{*}{Res2} & S & 20 & 0 & 2\\
    	\cline{2-5}
    	& F & 0 & 0 & 53\\
    	\cline{2-5}
    	& TO & 0 & 0 & 0\\
    	\hline
        \end{tabular}
    \caption{}
    \label{fig:lookup-table-outcome-res2-res3}
    \end{subtable}
\caption{Comparison of verification outcomes with different discretization resolutions of activation functions on the fixed-point Iris and Vocalic benchmarks. On the left, (a) comparison between one and two decimal places; on the right, (b) comparison between two and three decimal places. Both tables are structured as confusion matrices: entries on the main diagonal represent benchmarks with the same outcome under both resolutions. There, we discriminate between successful verification outcomes (S), falsifiable properties that admit a counterexample (F), and properties that incurred in timeout (TO).}
\label{fig:lookup-table-outcome}
\end{table}

\subsubsection{Code generation comparison}
\label{sec:code-balancing}

In Section \ref{ssec:codegen} we mentioned that a single ANN can be implemented in multiple ways. In fact, due to the intrinsic parallelism of neural architectures, the order of many mathematical operations can be shuffled arbitrarily. Here, we show that our verification methodology produces the same result (time and outcome) for very different orderings of these mathematical operations, and thus its performance is stable across them.

Specifically, we focus on the order of operations required to compute the activation potential of each neuron, one of the basic building blocks of ANNs (see \eqref{eq:anncalc_aff}). In this regard, we compare two opposite implementations of it that we exemplify in Fig.~\ref{fig:code-balancing}. On the one hand, we have run a fully sequential version of that ANN code, where each multiply-and-accumulate (MAC) operation in \eqref{eq:anncalc_aff} is executed in the same order as the input vector $\mathbf{x}$. We implement this version of the code with simple loops as in the example of Fig.~\ref{fig:code-balancing-seq}. On the other hand, we have also run a \textit{balanced} version of the ANN code, where the MAC operations are reordered in a divide-and-conquer sequence to minimize the number of additions, as in the example of Fig.~\ref{fig:code-balancing-bal}. Such associative rebalancing procedures are common optimizations performed by compilers, as they reduce the total number of machine instructions and improve execution time on out-of-order processors~\cite{Zory98usingalgebraic,Dekel1986,Karfa2013}.

\begin{figure}[ht]
\centering
    \begin{subfigure}{0.4\textwidth}
        \begin{lstlisting}[numbers=left]
float potential_8(float *w,
                  float *x,
                  float b) {
  float result = 0;
  
  for (unsigned int i=0; i<8; ++i) {
    result += w[i] * x[i];
  }
  
  result += b;
  
  return result;
}
        \end{lstlisting}
    \caption{}
    \label{fig:code-balancing-seq}
    \end{subfigure}
\hspace{0.05\textwidth}
    \begin{subfigure}{0.4\textwidth}
\begin{lstlisting}[numbers=left]  
float potential_8(float *w,
                  float *x,
                  float b) {
  float tmp_01 = w[0]*x[0]+w[1]*x[1];
  float tmp_23 = w[2]*x[2]+w[3]*x[3];
  float tmp_45 = w[4]*x[4]+w[5]*x[5];
  float tmp_67 = w[6]*x[6]+w[7]*x[7];
  float tmp_0123 = tmp_01 + tmp_23;
  float tmp_4567 = tmp_45 + tmp_67;
  float result = tmp_0123 + tmp_4567;
  result += b;
  return result;
}
        \end{lstlisting}
    \caption{}
    \label{fig:code-balancing-bal}
    \end{subfigure}
\caption{An example of balancing the order of MAC operations when computing activation potentials with eight inputs: (a) sequential version with a loop and (b) balanced version with a divide-and-conquer pattern.}
\label{fig:code-balancing}
\end{figure}

As we have done in the previous Sections \ref{sec:smt-solvers-comparison-v2}, \ref{sec:fine-tuning-bmc-parameters-v2} and \ref{sec:interval-analysis-comparison}, we compare the verification performance of these two code generation approaches on the fixed-point Iris and Vocalic benchmarks, considering the word lengths $8$, $16$ and $32$ (bits). The results in Fig.~\ref{fig:code-balancing-result} show very little difference in verification time and identical verification outcomes (except for one single timeout with balanced code). The reason for such a behavior lies in the understanding that ESBMC performs several aggressive expression-simplification steps including associative techniques, as explained in Section \ref{sec:SearchSpaceReduction}. As such, the final set of SMT formulae that are fed into the solver are quite insensitive to the order of operations in ANN code. Thus, we can conclude that the performance of our verification methodology is consistent across different implementations of the same ANN.

\begin{figure}[htb]
\centering
	\begin{subfigure}[b]{0.48\linewidth}
		\centering
		    \includegraphics[width=\linewidth]{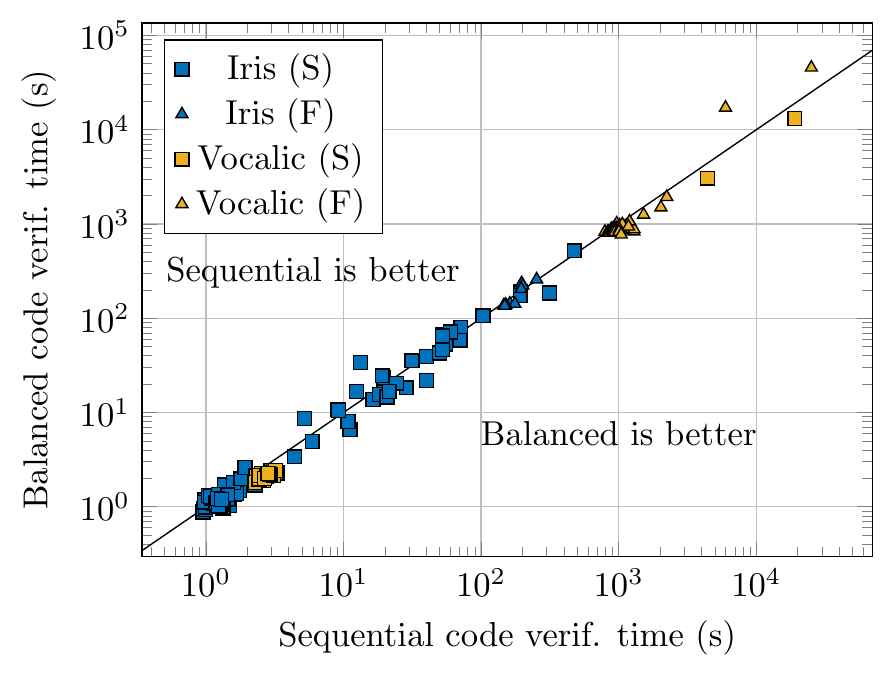}
		\caption{}
		\label{fig:balancing-time}
	\end{subfigure}
	\begin{subfigure}[b]{0.48\linewidth}
    \centering
    \raisebox{2.8cm}{
        \begin{tabular}{|c|c|c|c|c|}
        \hline
        \multicolumn{2}{|c|}{Iris} & \multicolumn{3}{c|}{Balanced}\\
        \cline{3-5}
        \multicolumn{2}{|c|}{Dataset} & S & F & TO\\
        \hline
        \multirow{3}{*}{Sequential} & S & 72 & 0 & 0\\
        \cline{2-5}
        & F & 0 & 9 & 0\\
        \cline{2-5}
        & TO & 0 & 0 & 0\\
        \hline
        \multicolumn{5}{c}{}\\
        \hline
        \multicolumn{2}{|c|}{Vocalic} & \multicolumn{3}{c|}{Balanced}\\
        \cline{3-5}
        \multicolumn{2}{|c|}{Dataset} & S & F & TO\\
        \hline
        \multirow{3}{*}{Sequential} & S & 22 & 0 & 0\\
        \cline{2-5}
        & F & 0 & 52 & 1\\
        \cline{2-5}
        & TO & 0 & 0 & 0\\
        \hline
        \end{tabular}
    }
    \caption{}
    \label{fig:balancing-outcome}
    \end{subfigure}
\caption{Comparison of verification performance with different ANN code generation techniques on the fixed-point Iris and Vocalic datasets. On the left, (a) verification time; on the right, (b) verification outcome. In all plots and tables, we discriminate between successful verification outcomes (S) and falsifiable safety properties that admit a counterexample (F).}
\label{fig:code-balancing-result}
\end{figure}

\begin{tcolorbox}
The results presented here successfully answer \textbf{RQ1 - Ablation study}: we have identified an optimal configuration for the ESBMC verification engine within our framework, which consists of using the SMT solver Yices and the \texttt{interval-analysis} option in conjunction with FRAMA-C intervals. Moreover, we quantified the individual importance and associated influence of a number of related techniques: constant propagation, expression simplification, slicing, incremental verification, discretization of non-linear activation functions, and code generation.
\end{tcolorbox}

\subsection{Verification of quantized ANNs}
\label{sec:verification-of-anns-with-fwl-implementation-v2}

In Section \ref{sec:ablation-study-v2}, we established what the best configuration of our verification method is by comparing its runtime under different scenarios. Similarly, in the present section, we compare its verification time and output along another dimension: the quantization level of ANNs. Our main result is that the granularity of ANN quantization may influence verification performance, but that may even be considered minor, depending on the specific aspect being evaluated. Here, we show that this is true both for verification time and verification outcome. Consequently, ANN quantization can be regarded as a viable and effective tool for adaptation towards a given target platform, as long as some evaluation is performed. 

\subsubsection{Effects of quantization on verification time}
\label{sec:quantization-vs-verification-time}

First, let us comment on how the quantization of an ANN affects its verification time of its safety properties. First, recall that verifying quantized neural networks is PSPACE-hard, as proven by~\cite{Henzinger2020}. However, this is a theoretical worst case, and existing empirical results in~\cite{Giacobbe2020} show a positive correlation between the number of bits used in a quantized representation and the total verification time. Here, we show that this correlation holds only for small number of bits and specific safety properties, and there is no general trend for word lengths equal or longer than $16$ bits.

To this end, we run our Iris and Vocalic benchmarks with a broad range of quantization levels, covering the span between the common word lengths of $8$, $16$ and $32$ bits, and extending to smaller word lengths with zero fractional bits. We present such results in Fig.~\ref{fig:quantum-scatter}. Note that there is a general upwards trend in verification time for short word lengths (from $6-7$ to $15$ bits), but this phenomenon almost disappears for longer word lengths ($16$ bits and above). Moreover, results are spread across six orders of magnitude, thus it is difficult to prove the existence of a true correlation in the associated data. In fact, applying common summary statistics (e.g., median verification time like in~\cite{Giacobbe2020}) shows only a partial correlation between time and quantization for the Iris benchmarks, and none for the Vocalic ones.

\begin{figure}[htb]
\centering
    \begin{subfigure}[b]{0.48\linewidth}
		\centering
		    \includegraphics[width=\linewidth]{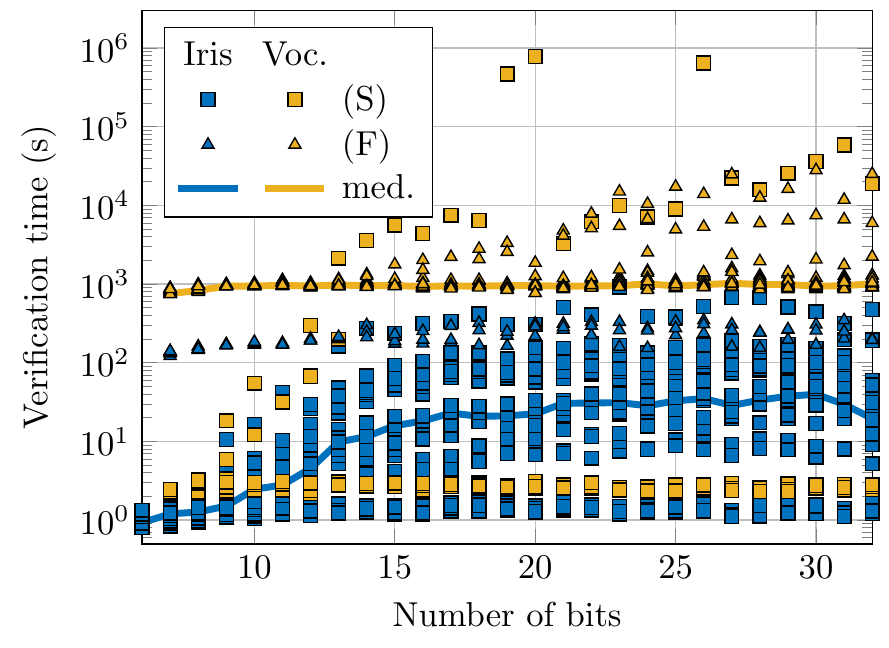}
		\caption{}
		\label{fig:quantum-scatter}
	\end{subfigure}
    \begin{subfigure}[b]{0.48\linewidth}
		\centering
		    \includegraphics[width=\linewidth]{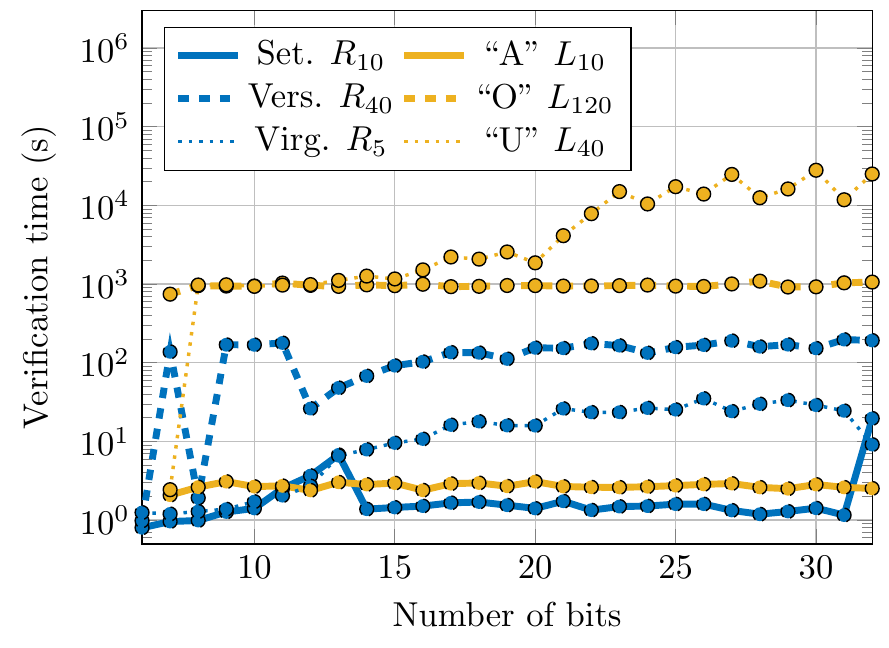}
		\caption{}
		\label{fig:quantum-prop}
	\end{subfigure}
\caption{Comparison of verification times with different quantization levels on the fixed-point Iris and Vocalic benchmarks. On the left, (a) a scatter plot of all the safety properties in our benchmarks with their respective median times; on the right, (b) a selection of six safety properties shows a very limited correlation between number of bits and verification time.}
\label{fig:quantum_sweep_time}
\end{figure}

A better understanding can be extracted by selecting individual safety properties and comparing their verification time across different quantization levels. We do so in Fig.~~\ref{fig:quantum-prop}, where we choose six properties from Fig.~\ref{fig:quantum-scatter} that showcase the full range of behaviors. More specifically, we broadly observe three different behaviors. First, properties like Vocalic ``A'' $L_{10}$ and Vocalic ``O'' $L_{120}$ exhibit almost identical verification time across all quantization levels. Second, properties like Iris Setosa $R_{10}$ and Iris Versicolor $R_{40}$ are somewhat erratic across quantization levels. However, their verification time falls into a limited range, where no systematic trend emerges. Third, properties like Vocalic ``U'' $L_{40}$ and Iris Virginica $R_{5}$ have verification time that is mildly correlated with the quantization level.

Overall, we believe that the quantization level has only a minor impact on the hardness of the verification problem from a practical perspective. Other factors, like the number of active neurons or the size of input regions of a given safety property, are probably better predictors regarding verification time. However, since these are beyond the scope of the present paper, we leave a thorough exploration of them to future work, where we might establish predictors and bounds.

\subsubsection{Effects of quantization on verification outcome}
\label{sec:quantization-vs-verification-outcome}


Another aspect regards verification outcomes, where narrower bit widths deserve some discussion. Here, we take the results of the same experiments shown in Section \ref{sec:quantization-vs-verification-time} and plot, in Fig.~\ref{fig:quantum_sweep_outcome}, a summary of how many safety properties are declared safe (S), generate a counterexample (F), or result in timeout (TO). As the figure shows, the percentage of successful safety properties is stable across quantization levels. The only noticeable differences happen in the Iris and Vocalic benchmarks for small word lengths. In the former, we observe a sudden drop in the number of safe properties between $6$ and $7$ bits, which goes through behavior that resembles transient responses in control systems~\cite{ChavesBIFCF18}, until a more suitable representation is achieved ($12$ bits). In addition, with $6$ bits, all safety properties are declared safe, which is indeed due to differences caused by computation with quantized values (see Section \ref{ssec:quantisedann}). Moreover, one may notice a clear trend related to more comprehensive formats, indicating an increasing number of correct operations.

We observe a higher incidence of undecidable safety properties regarding the Vocalic benchmarks that lead to a timeout. Note, however, that the Vocalic ANN is larger than Iris. Thus, more timeout events are expected due to the additional computational complexity, which is also worsened by the chosen representation. Again, stability regarding verification outcome is only achieved when a more suitable representation is used ($14$ bits).

In this context, some conclusions can be drawn. Indeed, there is a clear relationship between data representation and safety-property verification when using restricted formats. In addition, it becomes negligible when more bits are used. Moreover, arbitrarily small representations should not be carelessly used, as erratic behavior may be experienced.

\begin{figure}[htb]
\centering
    \begin{subfigure}[b]{0.48\linewidth}
		\centering
		    \includegraphics[width=\linewidth]{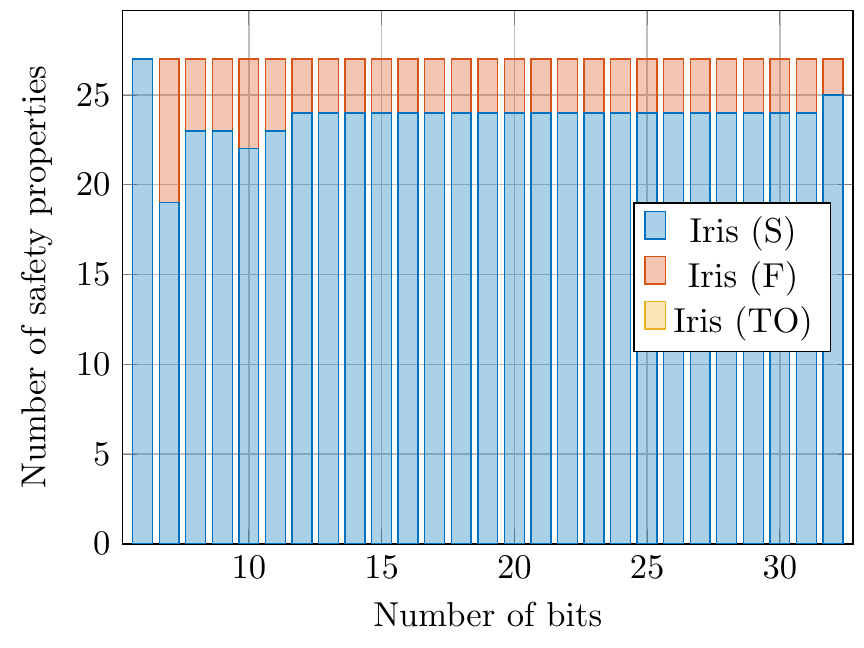}
		\caption{}
		\label{fig:quantum-out-iris}
	\end{subfigure}
    \begin{subfigure}[b]{0.48\linewidth}
		\centering
		    \includegraphics[width=\linewidth]{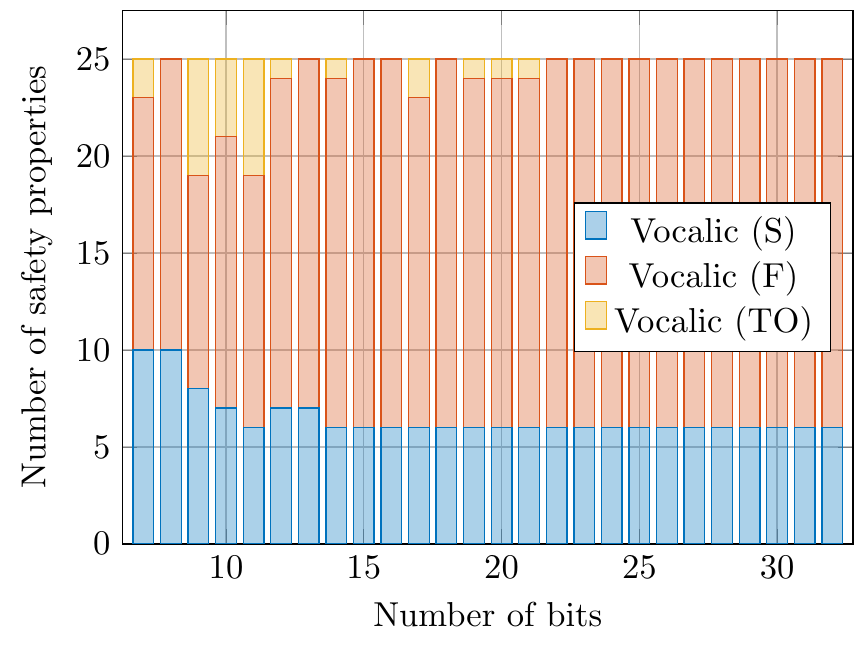}
		\caption{}
		\label{fig:quantum-out-vocalic}
	\end{subfigure}
\caption{Comparison of verification outcomes with different quantization levels on the fixed-point Iris and Vocalic benchmarks. On the left, (a) Iris dataset; on the right, (b) Vocalic dataset. In both histograms, we discriminate between successful verification outcomes (S), falsifiable properties that admit a counterexample (F), and properties that resulted in timeout (TO).}
\label{fig:quantum_sweep_outcome}
\end{figure}

A more focused picture of the relationship between quantization and verification outcome can be extracted by looking at individual safety properties. To this end, we report, in Tables \ref{table:quantum_sweep_iris_out} and \ref{table:quantum_sweep_vocalic_out}, all safety properties that have different outcomes across quantization levels. There, we can see two completely opposite behaviors. On the one hand, properties like Vocalic ``A'' $L_{20}$, Vocalic ``I'' $L_{10}$, Iris Versicolor $R_{50}$ and Iris Virginica $L_{50}$ are only safe for very short word lengths. On the other hand, properties like Vocalic ``U'' $L_{20}$  and Iris Versicolor $R_{40}$ tend to be safe as the word length increases.

\begin{table}[htb]
\resizebox{\textwidth}{!}{%
\begin{tabular}{|c|c|c|c|c|c|c|c|c|c|c|c|c|c|c|c|c|c|c|c|c|c|c|c|c|c|c|c|c|}
\hline
\multicolumn{2}{|c|}{Iris} & \multicolumn{27}{c|}{Number of bits}\\
\hline
\multicolumn{2}{|c|}{Property} & 6 & 7 & 8 & 9 & 10 & 11 & 12 & 13 & 14 & 15 & 16 & 17 & 18 & 19 & 20 & 21 & 22 & 23 & 24 & 25 & 26 & 27 & 28 & 29 & 30 & 31 & 32\\
\hline
\multirow{2}{*}{Set.} & $R_{40}$ & \cellcolor{blue!25}S & \cellcolor{blue!25}S & \cellcolor{red!25}F & \cellcolor{blue!25}S & \cellcolor{blue!25}S & \cellcolor{blue!25}S & \cellcolor{blue!25}S & \cellcolor{blue!25}S & \cellcolor{blue!25}S & \cellcolor{blue!25}S & \cellcolor{blue!25}S & \cellcolor{blue!25}S & \cellcolor{blue!25}S & \cellcolor{blue!25}S & \cellcolor{blue!25}S & \cellcolor{blue!25}S & \cellcolor{blue!25}S & \cellcolor{blue!25}S & \cellcolor{blue!25}S & \cellcolor{blue!25}S & \cellcolor{blue!25}S & \cellcolor{blue!25}S & \cellcolor{blue!25}S & \cellcolor{blue!25}S & \cellcolor{blue!25}S & \cellcolor{blue!25}S & \cellcolor{blue!25}S\\
\cline{2-29}
 & $R_{50}$ & \cellcolor{blue!25}S & \cellcolor{blue!25}S & \cellcolor{red!25}F & \cellcolor{red!25}F & \cellcolor{red!25}F & \cellcolor{red!25}F & \cellcolor{red!25}F & \cellcolor{red!25}F & \cellcolor{red!25}F & \cellcolor{red!25}F & \cellcolor{red!25}F & \cellcolor{red!25}F & \cellcolor{red!25}F & \cellcolor{red!25}F & \cellcolor{red!25}F & \cellcolor{red!25}F & \cellcolor{red!25}F & \cellcolor{red!25}F & \cellcolor{red!25}F & \cellcolor{red!25}F & \cellcolor{red!25}F & \cellcolor{red!25}F & \cellcolor{red!25}F & \cellcolor{red!25}F & \cellcolor{red!25}F & \cellcolor{red!25}F & \cellcolor{blue!25}S\\
\hline
\multirow{4}{*}{Vers.} & $R_{20}$ & \cellcolor{blue!25}S & \cellcolor{red!25}F & \cellcolor{blue!25}S & \cellcolor{blue!25}S & \cellcolor{blue!25}S & \cellcolor{blue!25}S & \cellcolor{blue!25}S & \cellcolor{blue!25}S & \cellcolor{blue!25}S & \cellcolor{blue!25}S & \cellcolor{blue!25}S & \cellcolor{blue!25}S & \cellcolor{blue!25}S & \cellcolor{blue!25}S & \cellcolor{blue!25}S & \cellcolor{blue!25}S & \cellcolor{blue!25}S & \cellcolor{blue!25}S & \cellcolor{blue!25}S & \cellcolor{blue!25}S & \cellcolor{blue!25}S & \cellcolor{blue!25}S & \cellcolor{blue!25}S & \cellcolor{blue!25}S & \cellcolor{blue!25}S & \cellcolor{blue!25}S & \cellcolor{blue!25}S\\
\cline{2-29}
 & $R_{30}$ & \cellcolor{blue!25}S & \cellcolor{red!25}F & \cellcolor{blue!25}S & \cellcolor{blue!25}S & \cellcolor{blue!25}S & \cellcolor{blue!25}S & \cellcolor{blue!25}S & \cellcolor{blue!25}S & \cellcolor{blue!25}S & \cellcolor{blue!25}S & \cellcolor{blue!25}S & \cellcolor{blue!25}S & \cellcolor{blue!25}S & \cellcolor{blue!25}S & \cellcolor{blue!25}S & \cellcolor{blue!25}S & \cellcolor{blue!25}S & \cellcolor{blue!25}S & \cellcolor{blue!25}S & \cellcolor{blue!25}S & \cellcolor{blue!25}S & \cellcolor{blue!25}S & \cellcolor{blue!25}S & \cellcolor{blue!25}S & \cellcolor{blue!25}S & \cellcolor{blue!25}S & \cellcolor{blue!25}S\\
\cline{2-29}
 & $R_{40}$ & \cellcolor{blue!25}S & \cellcolor{red!25}F & \cellcolor{blue!25}S & \cellcolor{red!25}F & \cellcolor{red!25}F & \cellcolor{red!25}F & \cellcolor{blue!25}S & \cellcolor{blue!25}S & \cellcolor{blue!25}S & \cellcolor{blue!25}S & \cellcolor{blue!25}S & \cellcolor{blue!25}S & \cellcolor{blue!25}S & \cellcolor{blue!25}S & \cellcolor{blue!25}S & \cellcolor{blue!25}S & \cellcolor{blue!25}S & \cellcolor{blue!25}S & \cellcolor{blue!25}S & \cellcolor{blue!25}S & \cellcolor{blue!25}S & \cellcolor{blue!25}S & \cellcolor{blue!25}S & \cellcolor{blue!25}S & \cellcolor{blue!25}S & \cellcolor{blue!25}S & \cellcolor{blue!25}S\\
\cline{2-29}
 & $R_{50}$ & \cellcolor{blue!25}S & \cellcolor{red!25}F & \cellcolor{red!25}F & \cellcolor{red!25}F & \cellcolor{red!25}F & \cellcolor{red!25}F & \cellcolor{red!25}F & \cellcolor{red!25}F & \cellcolor{red!25}F & \cellcolor{red!25}F & \cellcolor{red!25}F & \cellcolor{red!25}F & \cellcolor{red!25}F & \cellcolor{red!25}F & \cellcolor{red!25}F & \cellcolor{red!25}F & \cellcolor{red!25}F & \cellcolor{red!25}F & \cellcolor{red!25}F & \cellcolor{red!25}F & \cellcolor{red!25}F & \cellcolor{red!25}F & \cellcolor{red!25}F & \cellcolor{red!25}F & \cellcolor{red!25}F & \cellcolor{red!25}F & \cellcolor{red!25}F\\
\hline
\multirow{4}{*}{Virg.} & $R_{20}$ & \cellcolor{blue!25}S & \cellcolor{red!25}F & \cellcolor{blue!25}S & \cellcolor{blue!25}S & \cellcolor{blue!25}S & \cellcolor{blue!25}S & \cellcolor{blue!25}S & \cellcolor{blue!25}S & \cellcolor{blue!25}S & \cellcolor{blue!25}S & \cellcolor{blue!25}S & \cellcolor{blue!25}S & \cellcolor{blue!25}S & \cellcolor{blue!25}S & \cellcolor{blue!25}S & \cellcolor{blue!25}S & \cellcolor{blue!25}S & \cellcolor{blue!25}S & \cellcolor{blue!25}S & \cellcolor{blue!25}S & \cellcolor{blue!25}S & \cellcolor{blue!25}S & \cellcolor{blue!25}S & \cellcolor{blue!25}S & \cellcolor{blue!25}S & \cellcolor{blue!25}S & \cellcolor{blue!25}S\\
\cline{2-29}
 & $R_{30}$ & \cellcolor{blue!25}S & \cellcolor{red!25}F & \cellcolor{blue!25}S & \cellcolor{blue!25}S & \cellcolor{blue!25}S & \cellcolor{blue!25}S & \cellcolor{blue!25}S & \cellcolor{blue!25}S & \cellcolor{blue!25}S & \cellcolor{blue!25}S & \cellcolor{blue!25}S & \cellcolor{blue!25}S & \cellcolor{blue!25}S & \cellcolor{blue!25}S & \cellcolor{blue!25}S & \cellcolor{blue!25}S & \cellcolor{blue!25}S & \cellcolor{blue!25}S & \cellcolor{blue!25}S & \cellcolor{blue!25}S & \cellcolor{blue!25}S & \cellcolor{blue!25}S & \cellcolor{blue!25}S & \cellcolor{blue!25}S & \cellcolor{blue!25}S & \cellcolor{blue!25}S & \cellcolor{blue!25}S\\
\cline{2-29}
 & $R_{40}$ & \cellcolor{blue!25}S & \cellcolor{red!25}F & \cellcolor{blue!25}S & \cellcolor{blue!25}S & \cellcolor{red!25}F & \cellcolor{blue!25}S & \cellcolor{blue!25}S & \cellcolor{blue!25}S & \cellcolor{blue!25}S & \cellcolor{blue!25}S & \cellcolor{blue!25}S & \cellcolor{blue!25}S & \cellcolor{blue!25}S & \cellcolor{blue!25}S & \cellcolor{blue!25}S & \cellcolor{blue!25}S & \cellcolor{blue!25}S & \cellcolor{blue!25}S & \cellcolor{blue!25}S & \cellcolor{blue!25}S & \cellcolor{blue!25}S & \cellcolor{blue!25}S & \cellcolor{blue!25}S & \cellcolor{blue!25}S & \cellcolor{blue!25}S & \cellcolor{blue!25}S & \cellcolor{blue!25}S\\
\cline{2-29}
 & $R_{50}$ & \cellcolor{blue!25}S & \cellcolor{red!25}F & \cellcolor{red!25}F & \cellcolor{red!25}F & \cellcolor{red!25}F & \cellcolor{red!25}F & \cellcolor{red!25}F & \cellcolor{red!25}F & \cellcolor{red!25}F & \cellcolor{red!25}F & \cellcolor{red!25}F & \cellcolor{red!25}F & \cellcolor{red!25}F & \cellcolor{red!25}F & \cellcolor{red!25}F & \cellcolor{red!25}F & \cellcolor{red!25}F & \cellcolor{red!25}F & \cellcolor{red!25}F & \cellcolor{red!25}F & \cellcolor{red!25}F & \cellcolor{red!25}F & \cellcolor{red!25}F & \cellcolor{red!25}F & \cellcolor{red!25}F & \cellcolor{red!25}F & \cellcolor{red!25}F\\
\hline
\end{tabular}}
\caption{Iris safety properties with different verification outcomes across quantization levels.}
\label{table:quantum_sweep_iris_out}
\end{table}

\begin{table}[htb]
\resizebox{\textwidth}{!}{%
\begin{tabular}{|c|c|c|c|c|c|c|c|c|c|c|c|c|c|c|c|c|c|c|c|c|c|c|c|c|c|c|c|}
\hline
\multicolumn{2}{|c|}{Vocalic} & \multicolumn{26}{c|}{Number of bits}\\
\hline
\multicolumn{2}{|c|}{Property} & 7 & 8 & 9 & 10 & 11 & 12 & 13 & 14 & 15 & 16 & 17 & 18 & 19 & 20 & 21 & 22 & 23 & 24 & 25 & 26 & 27 & 28 & 29 & 30 & 31 & 32\\
\hline
\multirow{4}{*}{A} & $L_{20}$ & \cellcolor{blue!25}S & \cellcolor{blue!25}S & \cellcolor{blue!25}S & \cellcolor{blue!25}S & \cellcolor{blue!25}S & \cellcolor{blue!25}S & \cellcolor{blue!25}S & \cellcolor{red!25}F & \cellcolor{red!25}F & \cellcolor{red!25}F & \cellcolor{yellow!50}TO & \cellcolor{red!25}F & \cellcolor{red!25}F & \cellcolor{yellow!50}TO & \cellcolor{red!25}F & \cellcolor{red!25}F & \cellcolor{red!25}F & \cellcolor{red!25}F & \cellcolor{red!25}F & \cellcolor{red!25}F & \cellcolor{red!25}F & \cellcolor{red!25}F & \cellcolor{red!25}F & \cellcolor{red!25}F & \cellcolor{red!25}F & \cellcolor{red!25}F\\
\cline{2-28}
 & $L_{40}$ & \cellcolor{blue!25}S & \cellcolor{red!25}F & \cellcolor{blue!25}S & \cellcolor{red!25}F & \cellcolor{red!25}F & \cellcolor{red!25}F & \cellcolor{red!25}F & \cellcolor{red!25}F & \cellcolor{red!25}F & \cellcolor{red!25}F & \cellcolor{yellow!50}TO & \cellcolor{red!25}F & \cellcolor{yellow!50}TO & \cellcolor{red!25}F & \cellcolor{red!25}F & \cellcolor{red!25}F & \cellcolor{red!25}F & \cellcolor{red!25}F & \cellcolor{red!25}F & \cellcolor{red!25}F & \cellcolor{red!25}F & \cellcolor{red!25}F & \cellcolor{red!25}F & \cellcolor{red!25}F & \cellcolor{red!25}F & \cellcolor{red!25}F\\
\cline{2-28}
 & $L_{80}$ & \cellcolor{red!25}F & \cellcolor{red!25}F & \cellcolor{yellow!50}TO & \cellcolor{red!25}F & \cellcolor{red!25}F & \cellcolor{red!25}F & \cellcolor{red!25}F & \cellcolor{red!25}F & \cellcolor{red!25}F & \cellcolor{red!25}F & \cellcolor{red!25}F & \cellcolor{red!25}F & \cellcolor{red!25}F & \cellcolor{red!25}F & \cellcolor{red!25}F & \cellcolor{red!25}F & \cellcolor{red!25}F & \cellcolor{red!25}F & \cellcolor{red!25}F & \cellcolor{red!25}F & \cellcolor{red!25}F & \cellcolor{red!25}F & \cellcolor{red!25}F & \cellcolor{red!25}F & \cellcolor{red!25}F & \cellcolor{red!25}F\\
\cline{2-28}
 & $L_{120}$ & \cellcolor{red!25}F & \cellcolor{red!25}F & \cellcolor{red!25}F & \cellcolor{red!25}F & \cellcolor{red!25}F & \cellcolor{red!25}F & \cellcolor{red!25}F & \cellcolor{yellow!50}TO & \cellcolor{red!25}F & \cellcolor{red!25}F & \cellcolor{red!25}F & \cellcolor{red!25}F & \cellcolor{red!25}F & \cellcolor{red!25}F & \cellcolor{yellow!50}TO & \cellcolor{red!25}F & \cellcolor{red!25}F & \cellcolor{red!25}F & \cellcolor{red!25}F & \cellcolor{red!25}F & \cellcolor{red!25}F & \cellcolor{red!25}F & \cellcolor{red!25}F & \cellcolor{red!25}F & \cellcolor{red!25}F & \cellcolor{red!25}F\\
\hline
\multirow{3}{*}{E} & $L_{40}$ & \cellcolor{yellow!50}TO & \cellcolor{red!25}F & \cellcolor{yellow!50}TO & \cellcolor{yellow!50}TO & \cellcolor{yellow!50}TO & \cellcolor{yellow!50}TO & \cellcolor{red!25}F & \cellcolor{red!25}F & \cellcolor{red!25}F & \cellcolor{red!25}F & \cellcolor{red!25}F & \cellcolor{red!25}F & \cellcolor{red!25}F & \cellcolor{red!25}F & \cellcolor{red!25}F & \cellcolor{red!25}F & \cellcolor{red!25}F & \cellcolor{red!25}F & \cellcolor{red!25}F & \cellcolor{red!25}F & \cellcolor{red!25}F & \cellcolor{red!25}F & \cellcolor{red!25}F & \cellcolor{red!25}F & \cellcolor{red!25}F & \cellcolor{red!25}F\\
\cline{2-28}
 & $L_{80}$ & \cellcolor{red!25}F & \cellcolor{red!25}F & \cellcolor{yellow!50}TO & \cellcolor{red!25}F & \cellcolor{yellow!50}TO & \cellcolor{red!25}F & \cellcolor{red!25}F & \cellcolor{red!25}F & \cellcolor{red!25}F & \cellcolor{red!25}F & \cellcolor{red!25}F & \cellcolor{red!25}F & \cellcolor{red!25}F & \cellcolor{red!25}F & \cellcolor{red!25}F & \cellcolor{red!25}F & \cellcolor{red!25}F & \cellcolor{red!25}F & \cellcolor{red!25}F & \cellcolor{red!25}F & \cellcolor{red!25}F & \cellcolor{red!25}F & \cellcolor{red!25}F & \cellcolor{red!25}F & \cellcolor{red!25}F & \cellcolor{red!25}F\\
\cline{2-28}
 & $L_{120}$ & \cellcolor{red!25}F & \cellcolor{red!25}F & \cellcolor{red!25}F & \cellcolor{red!25}F & \cellcolor{yellow!50}TO & \cellcolor{red!25}F & \cellcolor{red!25}F & \cellcolor{red!25}F & \cellcolor{red!25}F & \cellcolor{red!25}F & \cellcolor{red!25}F & \cellcolor{red!25}F & \cellcolor{red!25}F & \cellcolor{red!25}F & \cellcolor{red!25}F & \cellcolor{red!25}F & \cellcolor{red!25}F & \cellcolor{red!25}F & \cellcolor{red!25}F & \cellcolor{red!25}F & \cellcolor{red!25}F & \cellcolor{red!25}F & \cellcolor{red!25}F & \cellcolor{red!25}F & \cellcolor{red!25}F & \cellcolor{red!25}F\\
\hline
\multirow{2}{*}{I} & $L_{10}$ & \cellcolor{blue!25}S & \cellcolor{blue!25}S & \cellcolor{blue!25}S & \cellcolor{yellow!50}TO & \cellcolor{red!25}F & \cellcolor{red!25}F & \cellcolor{red!25}F & \cellcolor{red!25}F & \cellcolor{red!25}F & \cellcolor{red!25}F & \cellcolor{red!25}F & \cellcolor{red!25}F & \cellcolor{red!25}F & \cellcolor{red!25}F & \cellcolor{red!25}F & \cellcolor{red!25}F & \cellcolor{red!25}F & \cellcolor{red!25}F & \cellcolor{red!25}F & \cellcolor{red!25}F & \cellcolor{red!25}F & \cellcolor{red!25}F & \cellcolor{red!25}F & \cellcolor{red!25}F & \cellcolor{red!25}F & \cellcolor{red!25}F\\
\cline{2-28}
 & $L_{20}$ & \cellcolor{red!25}F & \cellcolor{blue!25}S & \cellcolor{red!25}F & \cellcolor{yellow!50}TO & \cellcolor{red!25}F & \cellcolor{red!25}F & \cellcolor{red!25}F & \cellcolor{red!25}F & \cellcolor{red!25}F & \cellcolor{red!25}F & \cellcolor{red!25}F & \cellcolor{red!25}F & \cellcolor{red!25}F & \cellcolor{red!25}F & \cellcolor{red!25}F & \cellcolor{red!25}F & \cellcolor{red!25}F & \cellcolor{red!25}F & \cellcolor{red!25}F & \cellcolor{red!25}F & \cellcolor{red!25}F & \cellcolor{red!25}F & \cellcolor{red!25}F & \cellcolor{red!25}F & \cellcolor{red!25}F & \cellcolor{red!25}F\\
\hline
\multirow{3}{*}{O} & $L_{20}$ & \cellcolor{yellow!50}TO & \cellcolor{blue!25}S & \cellcolor{yellow!50}TO & \cellcolor{red!25}F & \cellcolor{yellow!50}TO & \cellcolor{red!25}F & \cellcolor{red!25}F & \cellcolor{red!25}F & \cellcolor{red!25}F & \cellcolor{red!25}F & \cellcolor{red!25}F & \cellcolor{red!25}F & \cellcolor{red!25}F & \cellcolor{red!25}F & \cellcolor{red!25}F & \cellcolor{red!25}F & \cellcolor{red!25}F & \cellcolor{red!25}F & \cellcolor{red!25}F & \cellcolor{red!25}F & \cellcolor{red!25}F & \cellcolor{red!25}F & \cellcolor{red!25}F & \cellcolor{red!25}F & \cellcolor{red!25}F & \cellcolor{red!25}F\\
\cline{2-28}
 & $L_{40}$ & \cellcolor{red!25}F & \cellcolor{red!25}F & \cellcolor{red!25}F & \cellcolor{red!25}F & \cellcolor{yellow!50}TO & \cellcolor{red!25}F & \cellcolor{red!25}F & \cellcolor{red!25}F & \cellcolor{red!25}F & \cellcolor{red!25}F & \cellcolor{red!25}F & \cellcolor{red!25}F & \cellcolor{red!25}F & \cellcolor{red!25}F & \cellcolor{red!25}F & \cellcolor{red!25}F & \cellcolor{red!25}F & \cellcolor{red!25}F & \cellcolor{red!25}F & \cellcolor{red!25}F & \cellcolor{red!25}F & \cellcolor{red!25}F & \cellcolor{red!25}F & \cellcolor{red!25}F & \cellcolor{red!25}F & \cellcolor{red!25}F\\
\cline{2-28}
 & $L_{80}$ & \cellcolor{red!25}F & \cellcolor{red!25}F & \cellcolor{yellow!50}TO & \cellcolor{red!25}F & \cellcolor{yellow!50}TO & \cellcolor{red!25}F & \cellcolor{red!25}F & \cellcolor{red!25}F & \cellcolor{red!25}F & \cellcolor{red!25}F & \cellcolor{red!25}F & \cellcolor{red!25}F & \cellcolor{red!25}F & \cellcolor{red!25}F & \cellcolor{red!25}F & \cellcolor{red!25}F & \cellcolor{red!25}F & \cellcolor{red!25}F & \cellcolor{red!25}F & \cellcolor{red!25}F & \cellcolor{red!25}F & \cellcolor{red!25}F & \cellcolor{red!25}F & \cellcolor{red!25}F & \cellcolor{red!25}F & \cellcolor{red!25}F\\
\hline
\multirow{3}{*}{U} & $L_{20}$ & \cellcolor{blue!25}S & \cellcolor{blue!25}S & \cellcolor{red!25}F & \cellcolor{blue!25}S & \cellcolor{red!25}F & \cellcolor{blue!25}S & \cellcolor{blue!25}S & \cellcolor{blue!25}S & \cellcolor{blue!25}S & \cellcolor{blue!25}S & \cellcolor{blue!25}S & \cellcolor{blue!25}S & \cellcolor{blue!25}S & \cellcolor{blue!25}S & \cellcolor{blue!25}S & \cellcolor{blue!25}S & \cellcolor{blue!25}S & \cellcolor{blue!25}S & \cellcolor{blue!25}S & \cellcolor{blue!25}S & \cellcolor{blue!25}S & \cellcolor{blue!25}S & \cellcolor{blue!25}S & \cellcolor{blue!25}S & \cellcolor{blue!25}S & \cellcolor{blue!25}S\\
\cline{2-28}
 & $L_{40}$ & \cellcolor{blue!25}S & \cellcolor{red!25}F & \cellcolor{red!25}F & \cellcolor{red!25}F & \cellcolor{red!25}F & \cellcolor{red!25}F & \cellcolor{red!25}F & \cellcolor{red!25}F & \cellcolor{red!25}F & \cellcolor{red!25}F & \cellcolor{red!25}F & \cellcolor{red!25}F & \cellcolor{red!25}F & \cellcolor{red!25}F & \cellcolor{red!25}F & \cellcolor{red!25}F & \cellcolor{red!25}F & \cellcolor{red!25}F & \cellcolor{red!25}F & \cellcolor{red!25}F & \cellcolor{red!25}F & \cellcolor{red!25}F & \cellcolor{red!25}F & \cellcolor{red!25}F & \cellcolor{red!25}F & \cellcolor{red!25}F\\
\cline{2-28}
 & $L_{80}$ & \cellcolor{red!25}F & \cellcolor{red!25}F & \cellcolor{yellow!50}TO & \cellcolor{red!25}F & \cellcolor{red!25}F & \cellcolor{red!25}F & \cellcolor{red!25}F & \cellcolor{red!25}F & \cellcolor{red!25}F & \cellcolor{red!25}F & \cellcolor{red!25}F & \cellcolor{red!25}F & \cellcolor{red!25}F & \cellcolor{red!25}F & \cellcolor{red!25}F & \cellcolor{red!25}F & \cellcolor{red!25}F & \cellcolor{red!25}F & \cellcolor{red!25}F & \cellcolor{red!25}F & \cellcolor{red!25}F & \cellcolor{red!25}F & \cellcolor{red!25}F & \cellcolor{red!25}F & \cellcolor{red!25}F & \cellcolor{red!25}F\\
\cline{2-28}
 & $L_{120}$ & \cellcolor{red!25}F & \cellcolor{red!25}F & \cellcolor{red!25}F & \cellcolor{yellow!50}TO & \cellcolor{red!25}F & \cellcolor{red!25}F & \cellcolor{red!25}F & \cellcolor{red!25}F & \cellcolor{red!25}F & \cellcolor{red!25}F & \cellcolor{red!25}F & \cellcolor{red!25}F & \cellcolor{red!25}F & \cellcolor{red!25}F & \cellcolor{red!25}F & \cellcolor{red!25}F & \cellcolor{red!25}F & \cellcolor{red!25}F & \cellcolor{red!25}F & \cellcolor{red!25}F & \cellcolor{red!25}F & \cellcolor{red!25}F & \cellcolor{red!25}F & \cellcolor{red!25}F & \cellcolor{red!25}F & \cellcolor{red!25}F\\
\hline
\end{tabular}}
\caption{Vocalic safety properties with different verification outcomes across quantization levels.}
\label{table:quantum_sweep_vocalic_out}
\end{table}

A possible explanation of this behavior is that the properties listed in Tables \ref{table:quantum_sweep_iris_out} and \ref{table:quantum_sweep_vocalic_out} are on the verge of breaking the ANN robustness. In fact, for each of these properties, reducing the input region size $L_i$ or $R_i$ makes them more safe, and increasing it makes them less safe. Thus, by staying on the threshold between these two regimes, any minor change in the ANN implementation (e.g., the quantization level) can easily flip the verification outcome and yield the erratic results we observe.

Indeed, there is no prediction methodology or technique capable of indicating that such behavior will occur; however, the present work successfully reveals it and hints at how to devise a suitable scheme. For instance, a closed-loop approach could test a chosen set of properties against some possible quantization levels so that the realization with the smallest number of bits that still provide an output that is considered stable is chosen. This way, we could navigate through a response dependent on the target quantization level and, when a steady-state region is achieved, the narrowest format that allows this behavior is chosen. Currently, we can only validate or not a given quantized ANN; we leave the synthesis of neural net implementations as future work.

\subsubsection{Effects of quantization on adversarial examples}
\label{sec:adversarial-cases-verification-in-fwl-ann-implementations-v2}

Finally, let us comment on the effect of quantization on the counterexamples returned by our verification approach. As Tables~\ref{table:quantum_sweep_iris_out} and \ref{table:quantum_sweep_vocalic_out} show, the verification outcome of the same safety property changes depending on the chosen ANN representation. This is because different quantization granularities may either hide or reveal specific vulnerabilities in ANN computation. At the same time, even if the verification outcome is the same and the safety property is kept falsifiable (F), the counterexamples returned by the verification engine may be different.

Here, we present a qualitative comparison between counterexamples, which, in the context of machine learning research, are also known as \textit{adversarial examples}. Specifically, we focus on our three levels of fixed-point quantization, namely 8, 16 and 32 bits, for which we present a selection of adversarial examples from the Vocalic benchmarks in Figs.~\ref{fig:imagequality8bits}, \ref{fig:imagequality16bits}, and \ref{fig:imagequality32bits} respectively. Each figure contains pairs of images, where the center of the input region is on the left, and its corresponding adversarial example is on the right. 

\begin{figure}[htb]
\centering
    \begin{subfigure}[t]{\linewidth}
    	\centering
    	\begin{subfigure}[t]{0.19\linewidth}
    		\centering
    		\begin{subfigure}[t]{0.4\linewidth}
    			\centering
    			\includegraphics[height=0.4in]{images/imagema.png}
    		\end{subfigure}%
    		~
    		\begin{subfigure}[t]{0.4\linewidth}
    			\centering
    			\includegraphics[height=0.4in]{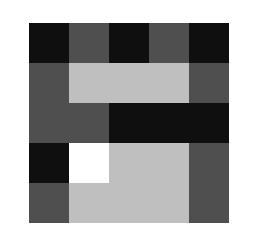}
    		\end{subfigure}
    		\caption*{``A'', $L_{40} \rightarrow$ ``O''.}
    	\end{subfigure}
    	\begin{subfigure}[t]{0.19\linewidth}
    		\centering
    		\begin{subfigure}[t]{0.4\linewidth}
    			\centering
    			\includegraphics[height=0.4in]{images/imageme.png}
    		\end{subfigure}%
    		~
    		\begin{subfigure}[t]{0.4\linewidth}
    			\centering
    			\includegraphics[height=0.4in]{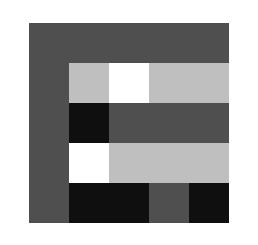}
    		\end{subfigure}
    		\caption*{``E'', $L_{40} \rightarrow$ ``A''.}
    	\end{subfigure}
    	\begin{subfigure}[t]{0.19\linewidth}
    		\centering
    		\begin{subfigure}[t]{0.4\linewidth}
    			\centering
    			\includegraphics[height=0.4in]{images/imagemi.png}
    		\end{subfigure}%
    		~
    		\begin{subfigure}[t]{0.4\linewidth}
    			\centering
    			\includegraphics[height=0.4in]{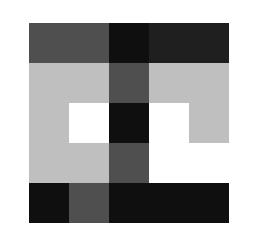}
    		\end{subfigure}
    		\caption*{``I'', $L_{40} \rightarrow$ ``A''.}
    	\end{subfigure}
    	\begin{subfigure}[t]{0.19\linewidth}
    		\centering
    		\begin{subfigure}[t]{0.4\linewidth}
    			\centering
    			\includegraphics[height=0.4in]{images/imagemo.png}
    		\end{subfigure}%
    		~
    		\begin{subfigure}[t]{0.4\linewidth}
    			\centering
    			\includegraphics[height=0.4in]{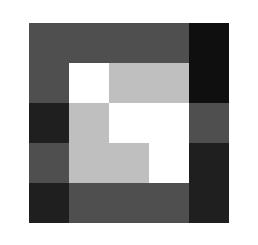}
    		\end{subfigure}
    		\caption*{``O'', $L_{40} \rightarrow$ ``A''.}
    	\end{subfigure}
    	\begin{subfigure}[t]{0.19\linewidth}
    		\centering
    		\begin{subfigure}[t]{0.4\linewidth}
    			\centering
    			\includegraphics[height=0.4in]{images/imagemu.png}
    		\end{subfigure}%
    		~
    		\begin{subfigure}[t]{0.4\linewidth}
    			\centering
    			\includegraphics[height=0.4in]{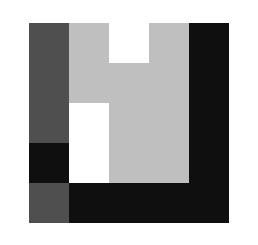}
    		\end{subfigure}
    		\caption*{``U'', $L_{40} \rightarrow$ ``A''.}
    	\end{subfigure}
        \caption{}
        \label{fig:imagequality8bits}
    \end{subfigure}
    	
    \begin{subfigure}[t]{\linewidth}
        \centering
    	\begin{subfigure}[t]{0.19\linewidth}
    		\centering
    		\begin{subfigure}[t]{0.4\linewidth}
    			\centering
    			\includegraphics[height=0.4in]{images/imagema.png}
    		\end{subfigure}%
    		~
    		\begin{subfigure}[t]{0.4\linewidth}
    			\centering
    			\includegraphics[height=0.4in]{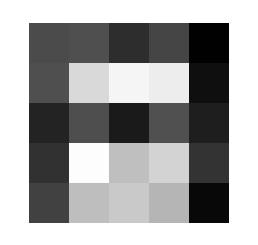}
    		\end{subfigure}
    		\caption*{``A'', $L_{40} \rightarrow$ ``U''.}
    	\end{subfigure}
    	\begin{subfigure}[t]{0.19\linewidth}
    		\centering
    		\begin{subfigure}[t]{0.4\linewidth}
    			\centering
    			\includegraphics[height=0.4in]{images/imageme.png}
    		\end{subfigure}%
    		~
    		\begin{subfigure}[t]{0.4\linewidth}
    			\centering
    			\includegraphics[height=0.4in]{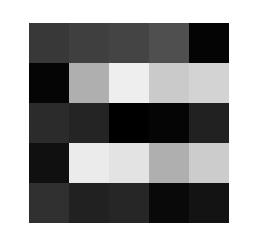}
    		\end{subfigure}
    		\caption*{``E'', $L_{40} \rightarrow$ ``A''}
    	\end{subfigure}
    	\begin{subfigure}[t]{0.19\linewidth}
    		\centering
    		\begin{subfigure}[t]{0.4\linewidth}
    			\centering
    			\includegraphics[height=0.4in]{images/imagemi.png}
    		\end{subfigure}%
    		~
    		\begin{subfigure}[t]{0.4\linewidth}
    			\centering
    			\includegraphics[height=0.4in]{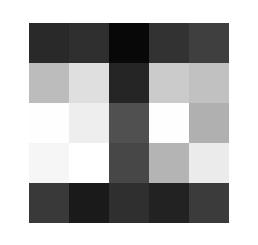}
    		\end{subfigure}
    		\caption*{``I'', $L_{40} \rightarrow$ ``E''.}
    	\end{subfigure}
    	\begin{subfigure}[t]{0.19\linewidth}
    		\centering
    		\begin{subfigure}[t]{0.4\linewidth}
    			\centering
    			\includegraphics[height=0.4in]{images/imagemo.png}
    		\end{subfigure}%
    		~
    		\begin{subfigure}[t]{0.4\linewidth}
    			\centering
    			\includegraphics[height=0.4in]{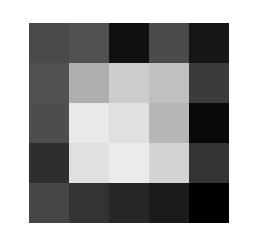}
    		\end{subfigure}
    		\caption*{``O'', $L_{40} \rightarrow$ ``I''.}
    	\end{subfigure}
    	\begin{subfigure}[t]{0.19\linewidth}
    		\centering
    		\begin{subfigure}[t]{0.4\linewidth}
    			\centering
    			\includegraphics[height=0.4in]{images/imagemu.png}
    		\end{subfigure}%
    		~
    		\begin{subfigure}[t]{0.4\linewidth}
    			\centering
    			\includegraphics[height=0.4in]{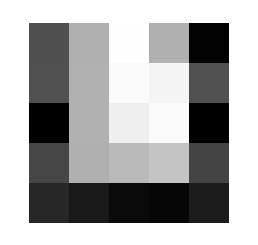}
    		\end{subfigure}
    		\caption*{``U'', $L_{40} \rightarrow$ ``I''.}
    	\end{subfigure}
        \caption{}
        \label{fig:imagequality16bits}
    \end{subfigure}
    	
    \begin{subfigure}[t]{\linewidth}
        \centering
    	\begin{subfigure}[t]{0.19\linewidth}
    		\centering
    		\begin{subfigure}[t]{0.4\linewidth}
    			\centering
    			\includegraphics[height=0.4in]{images/imagema.png}
    		\end{subfigure}%
    		~
    		\begin{subfigure}[t]{0.4\linewidth}
    			\centering
    			\includegraphics[height=0.4in]{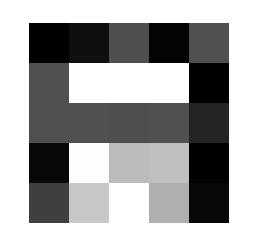}
    		\end{subfigure}
    		\caption*{``A'', $L_{40} \rightarrow$ ``U''.}
    	\end{subfigure}
    	\begin{subfigure}[t]{0.19\linewidth}
    		\centering
    		\begin{subfigure}[t]{0.4\linewidth}
    			\centering
    			\includegraphics[height=0.4in]{images/imageme.png}
    		\end{subfigure}%
    		~
    		\begin{subfigure}[t]{0.4\linewidth}
    			\centering
    			\includegraphics[height=0.4in]{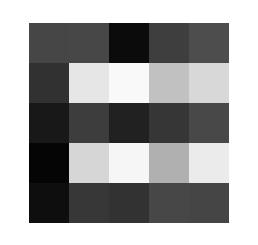}
    		\end{subfigure}
    		\caption*{``E'', $L_{40} \rightarrow$ ``I''.}
    	\end{subfigure}
    	\begin{subfigure}[t]{0.19\linewidth}
    		\centering
    		\begin{subfigure}[t]{0.4\linewidth}
    			\centering
    			\includegraphics[height=0.4in]{images/imagemi.png}
    		\end{subfigure}%
    		~
    		\begin{subfigure}[t]{0.4\linewidth}
    			\centering
    			\includegraphics[height=0.4in]{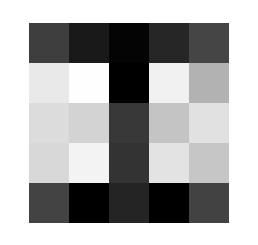}
    		\end{subfigure}
    		\caption*{``I'', $L_{40} \rightarrow$ ``E''.}
    	\end{subfigure}
    	\begin{subfigure}[t]{0.19\linewidth}
    		\centering
    		\begin{subfigure}[t]{0.4\linewidth}
    			\centering
    			\includegraphics[height=0.4in]{images/imagemo.png}
    		\end{subfigure}%
    		~
    		\begin{subfigure}[t]{0.4\linewidth}
    			\centering
    			\includegraphics[height=0.4in]{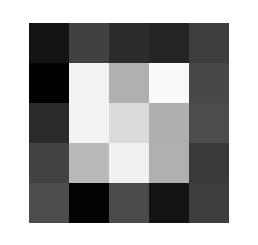}
    		\end{subfigure}
    		\caption*{``O'', $L_{40} \rightarrow$ ``I''.}
    	\end{subfigure}
    	\begin{subfigure}[t]{0.19\linewidth}
    		\centering
    		\begin{subfigure}[t]{0.4\linewidth}
    			\centering
    			\includegraphics[height=0.4in]{images/imagemu.png}
    		\end{subfigure}%
    		~
    		\begin{subfigure}[t]{0.4\linewidth}
    			\centering
    			\includegraphics[height=0.4in]{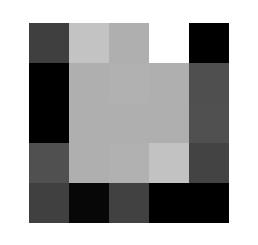}
    		\end{subfigure}
    		\caption*{``U'', $L_{40} \rightarrow$ ``A''.}
    	\end{subfigure}
        \caption{}
        \label{fig:imagequality32bits}
    \end{subfigure}
    \caption{Counterexamples for the Vocalic benchmarks with (a) $8$, (b) $16$, and (c) $32$-bits fixed-point representations. For each safety property, we report the centroid and associated counterexample, the input region dimension $L_i$, and the incorrect output label that was generated.}
    \label{fig:adversarial-image-quality}
\end{figure}

As Figures~\ref{fig:imagequality8bits}, \ref{fig:imagequality16bits}, and \ref{fig:imagequality32bits} show, the granularity of ANN quantization has a significant effect on the quality of the adversarial examples. On the one hand, coarser fixed-point representations, such as $\langle 4, 4 \rangle$, restrict the search space to fewer gray-scale levels, which is clearly seen in Fig.~\ref{fig:imagequality8bits}, given the easily noticeable differences. On the other hand, finer quantizations, such as $\langle 16, 16 \rangle$, let the verification engine produce counterexamples with minimal noise spread across the whole image (see Fig.~\ref{fig:imagequality32bits}). The latter is typical of floating-point ANNs. Besides, it is especially dangerous for this specific case of image classification, as such adversarial examples may go undetected even by a human observer~\cite{Szegedy2014}. It is important to notice that quantization levels also influence counterexamples. Moreover, they can be regarded as adapted to the contexts created by the latter.

\begin{tcolorbox}
These results successfully answer \textbf{RQ2 - Quantization effects}: we established that the verification time has some correlation with the number of bits of the ANN quantization. Moreover, we showed that the safety of an ANN is mostly stable across different quantization levels, which supports the use of aggressive quantization in machine learning practice as long as some verification is performed.
\end{tcolorbox}

\subsection{Comparison with state-of-the-art verification tools}
\label{sec:comparison-with-sota-v2}

This section compares our verification methodology with existing works in the literature. We note that the field is progressing very rapidly at the time of writing, and thus the present comparison is limited to what tools are currently available. Namely, the few existing approaches for verifying quantized ANNs (Giacobbe \textit{et al.}~\cite{Giacobbe2020}, Baranowski \textit{et al.}~\cite{Baranowski2020}, Kai Jia {\it et al.}~\cite{jia2021verifying}, Guy Amir {\it et al.}~\cite{amir2021smt}) do not provide reliable source code to replicate their experiments. As such, we can only compare our methodology with earlier tools that verify the safety of ANNs as abstract mathematical models, i.e. in infinite precision. Among those, we choose the two most popular ones as follow:
\begin{itemize}
    \item \textbf{Marabou}~\cite{Katz2019}. Based on the earlier tool Reluplex~\cite{katz2017reluplex}, Marabou uses a simplex-like algorithm to split the verification problem in smaller subproblems and invoke an SMT solver on each of them.
    \item \textbf{Neurify}~\cite{Wang2018}. It uses symbolic intervals to over-approximate the ReLU non-linearity of each neuron, and turn the verification process into finding the solution of a linear problem. These over-approximation are iteratively tightened by splitting each ReLU activation function in two independent linear problems. 
\end{itemize}

The goal of this comparison is showing that our quantized methodology is at least as efficient as these two state-of-the-art tools. At the same time, notice that our methodology provides more information on the safety of the actual ANN implementations than the abstract safety guarantees provided by tools like Marabou and Neurify.

To this end, we choose the AcasXu benchmark as our comparison suite (see Section \ref{sec:benchmark-acasxu}). This benchmark has the advantage of being already implemented in both Marabou and Neurify,\footnote{\url{https://github.com/NeuralNetworkVerification/Marabou}} \footnote{\url{https://github.com/tcwangshiqi-columbia/ReluVal}} thus allowing us to run the authors' code for a fair comparison of their performance. Furthermore, the neural networks in the AcasXu benchmark contain ReLU activation functions exclusively, which makes them compatible with Neurify. In a similar vein, we focus our comparison on safety property 1 of the AcasXu benchmark, since it is the one that incurs the fewest time-outs with the aforementioned verification tools~\cite{Katz2019, Wang2018-2}. Note that 45 different neural networks need to be verified for each safety property of AcasXu, thus giving us a larger enough sample size for a significant comparison. Regarding our verification methodology, we choose a 32 bit representation with 28 integer bits (including sign), which are needed to avoid overflows.

The summary of our results regarding verification time are shown in Figures \ref{fig:sota-comparison-marabou} and \ref{fig:sota-comparison-neurify}. On the one hand, in Fig.~\ref{fig:sota-comparison-marabou}, we compare our methodology with the SMT-based tool Marabou. Note how our verification methodology is considerably faster than Marabou. We believe this is because our underlying model checker, ESBMC, is more efficient at producing optimized SMT formulae (see Section~\ref{sec:SearchSpaceReduction}) than the custom simplex-like method employed by Marabou~\cite{Katz2019}. This also explains why the verification times of our methodology are almost constant across the whole comparison suite.

On the other hand, in Fig.~\ref{fig:sota-comparison-neurify}, we compare our methodology with the symbolic interval tool Neurify. Note that this tool has been released by the authors as a multi-threaded software~\cite{Wang2018}. This is the version we compare to in Fig.~\ref{fig:neurify-ours-27-4}. However, for the sake of a fair comparison with our methodology, we also present a modified version of Neurify that uses only a single thread in Fig.~\ref{fig:neurify-st-ours-27-4}.\footnote{The modified code is available at \url{https://github.com/ericksonalves/nn-verification-comparison}} Note how the multi-threaded version of Neurify is faster than our methodology in a majority of cases. This is because, on our machine, the multi-threaded Neurify uses up to 22 processors in parallel, giving it an obvious advantage over our single-threaded methodology. At the same time, such advantage disappears for the single-threaded version: more specifically, for the latter our methodology is faster in verifying 24 out of the 45 neural networks.

\begin{figure}[htb]
\centering
    \includegraphics[width=0.48\linewidth]{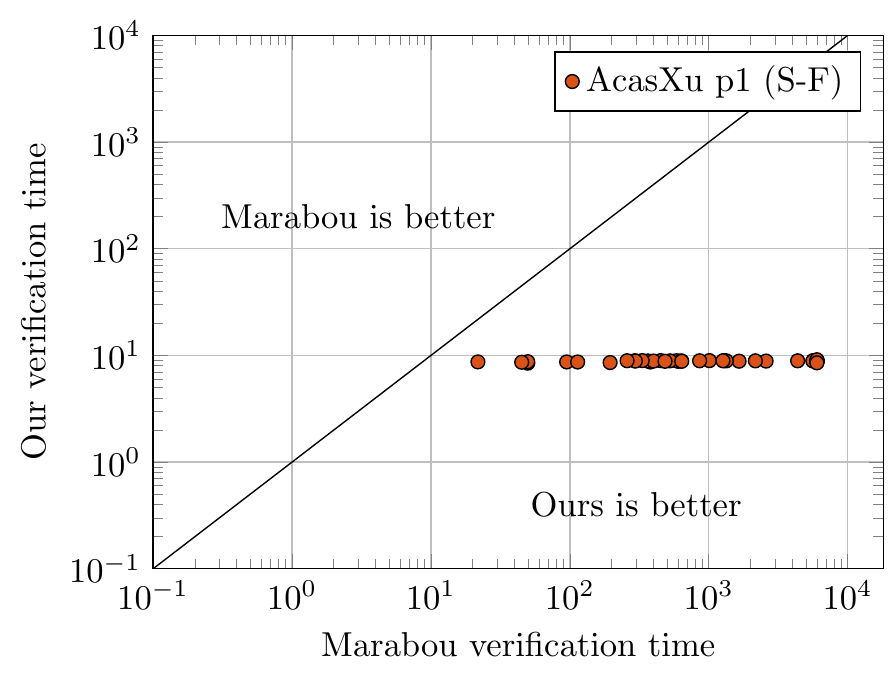}
\caption{Comparison of the verification times of Marabou and our methodology on property 1 of the AcasXu benchmark.}
\label{fig:sota-comparison-marabou}
\end{figure}

\begin{figure}[htb]
\centering
    \begin{subfigure}[b]{0.48\linewidth}
		\centering
		    \includegraphics[width=\linewidth]{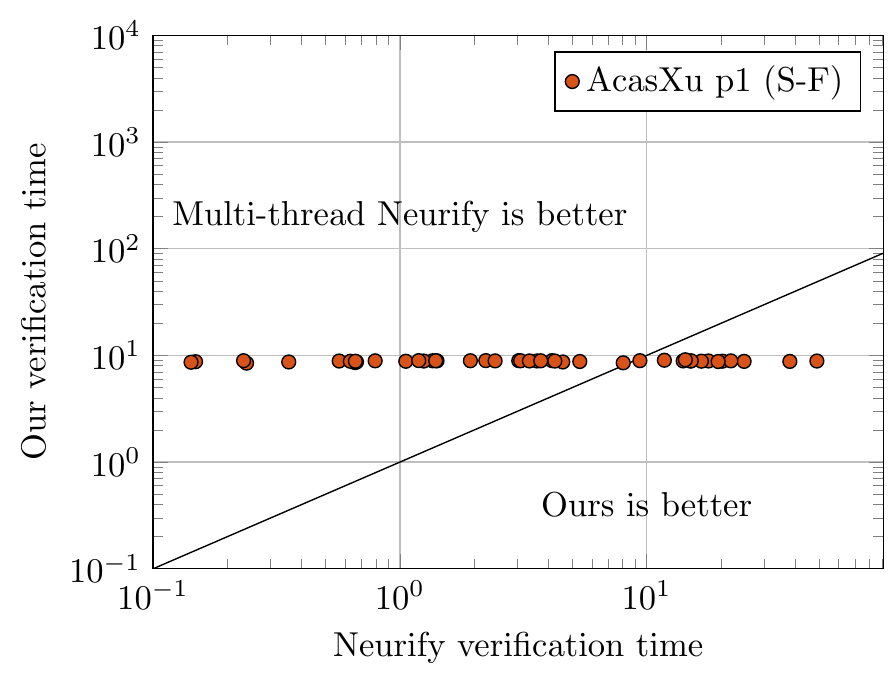}
		\caption{}
		\label{fig:neurify-ours-27-4}
	\end{subfigure}
	\begin{subfigure}[b]{0.48\linewidth}
		\centering
		    \includegraphics[width=\linewidth]{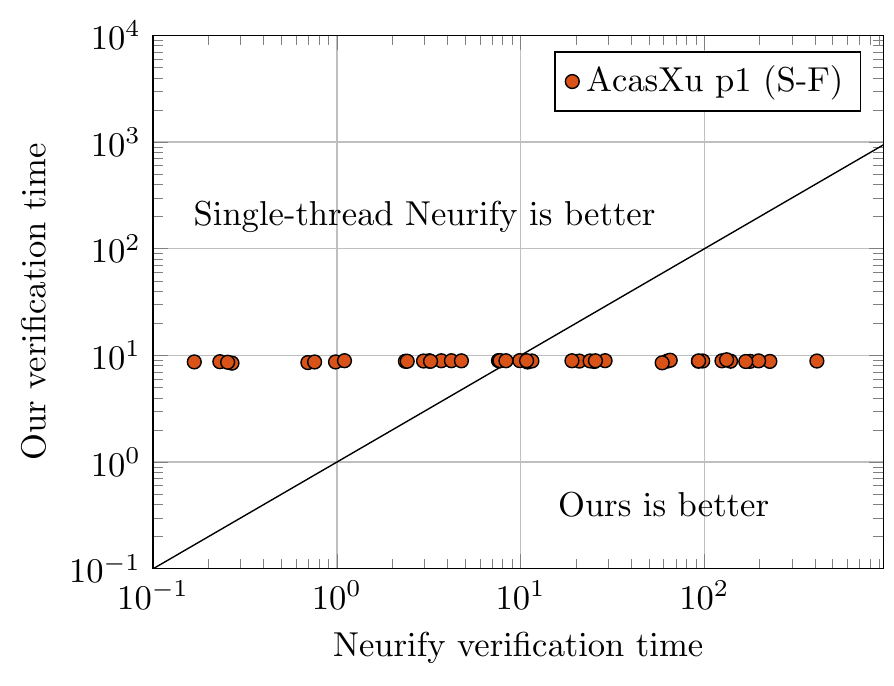}
		\caption{}
		\label{fig:neurify-st-ours-27-4}
	\end{subfigure}
\caption{Comparison of the verification times of Neurify and our methodology on property 1 of the AcasXu benchmark. On the left, (a) the original multi-threaded version of Neurify; on the right, (b) the modified single-threaded version.}
\label{fig:sota-comparison-neurify}
\end{figure}

For completeness, we also report the verification outcomes on all 45 benchmarks in Table \ref{table:comparison-floating-point}. Note how both Neurify and our methodology are able to successfully verify all 45 neural networks, whereas Marabou incurs a time-out for 13 of them. These results confirm that our methodology offers a comparable performance to that of Neurify, and faster than the state-of-the-art tool Marabou. Note also that our methodology offers guarantees on the actual implementation of the ANNs, e.g. the ones that would be deployed on an autonomous aircraft in the AcasXu case, thus making it more attractive for practical scenarios where the safety of a deployed system is critical.

\begin{table}[htb]
\begin{subtable}{\textwidth}
\resizebox{\textwidth}{!}{%
\begin{tabular}{|c|c|c|c|c|c|c|c|c|c|c|c|c|c|c|c|c|c|c|c|c|c|c|c|}
\hline
ACAS XU                                & \multicolumn{23}{c|}{Benchmark}                                                                                                                                \\ \hline
Tool                                   & 1\_1 & 1\_2 & 1\_3 & 1\_4 & 1\_5 & 1\_6 & 1\_7 & 1\_8 & 1\_9 & 2\_1 & 2\_2 & 2\_3 & 2\_4 & 2\_5 & 2\_6 & 2\_7 & 2\_8 & 2\_9 & 3\_1 & 3\_2 & 3\_3 & 3\_4 & 3\_5 \\ \hline
Marabou                                & \cellcolor{blue!25}S    & \cellcolor{blue!25}S    & \cellcolor{blue!25}S    & \cellcolor{blue!25}S    & \cellcolor{blue!25}S    & \cellcolor{blue!25}S    & \cellcolor{blue!25}S    & \cellcolor{blue!25}S    & \cellcolor{blue!25}S    & \cellcolor{blue!25}S    & \cellcolor{blue!25}S    & \cellcolor{blue!25}S    & \cellcolor{blue!25}S    & \cellcolor{blue!25}S    & \cellcolor{blue!25}S    & \cellcolor{yellow!50}TO   & \cellcolor{yellow!50}TO   & \cellcolor{yellow!50}TO   & \cellcolor{blue!25}S    & \cellcolor{blue!25}S    & \cellcolor{blue!25}S    & \cellcolor{blue!25}S    & \cellcolor{blue!25}S    \\ \hline
Neurify                                & \cellcolor{blue!25}S    & \cellcolor{blue!25}S    & \cellcolor{blue!25}S    & \cellcolor{blue!25}S    & \cellcolor{blue!25}S    & \cellcolor{blue!25}S    & \cellcolor{blue!25}S    & \cellcolor{blue!25}S    & \cellcolor{blue!25}S    & \cellcolor{blue!25}S    & \cellcolor{blue!25}S    & \cellcolor{blue!25}S    & \cellcolor{blue!25}S    & \cellcolor{blue!25}S    & \cellcolor{blue!25}S    & \cellcolor{blue!25}S    & \cellcolor{blue!25}S    & \cellcolor{blue!25}S    & \cellcolor{blue!25}S    & \cellcolor{blue!25}S    & \cellcolor{blue!25}S    & \cellcolor{blue!25}S    & \cellcolor{blue!25}S    \\ \hline
Ours \textless{}28,   4\textgreater{}   & \cellcolor{blue!25}S    & \cellcolor{blue!25}S    & \cellcolor{blue!25}S    & \cellcolor{blue!25}S    & \cellcolor{blue!25}S    & \cellcolor{blue!25}S    & \cellcolor{blue!25}S    & \cellcolor{blue!25}S    & \cellcolor{blue!25}S    & \cellcolor{blue!25}S    & \cellcolor{blue!25}S    & \cellcolor{blue!25}S    & \cellcolor{blue!25}S    & \cellcolor{blue!25}S    & \cellcolor{blue!25}S    & \cellcolor{blue!25}S    & \cellcolor{blue!25}S    & \cellcolor{blue!25}S    & \cellcolor{blue!25}S    & \cellcolor{blue!25}S    & \cellcolor{blue!25}S    & \cellcolor{blue!25}S    & \cellcolor{blue!25}S    \\ \hline
\end{tabular}}
\caption{}
\end{subtable}
\begin{subtable}{\textwidth}
\resizebox{\textwidth}{!}{%
\begin{tabular}{|c|c|c|c|c|c|c|c|c|c|c|c|c|c|c|c|c|c|c|c|c|c|c|}
\hline
ACAS XU                                & \multicolumn{22}{c|}{Benchmark}                                                                                                                         \\ \hline
Tool                                   & 3\_6 & 3\_7 & 3\_8 & 3\_9 & 4\_1 & 4\_2 & 4\_3 & 4\_4 & 4\_5 & 4\_6 & 4\_7 & 4\_8 & 4\_9 & 5\_1 & 5\_2 & 5\_3 & 5\_4 & 5\_5 & 5\_6 & 5\_7 & 5\_8 & 5\_9 \\ \hline
Marabou                                & \cellcolor{yellow!50}TO   & \cellcolor{yellow!50}TO   & \cellcolor{yellow!50}TO   & \cellcolor{yellow!50}TO   & \cellcolor{blue!25}S    & \cellcolor{blue!25}S    & \cellcolor{blue!25}S    & \cellcolor{blue!25}S    & \cellcolor{blue!25}S    & \cellcolor{yellow!50}TO   & \cellcolor{yellow!50}TO   & \cellcolor{yellow!50}TO   & \cellcolor{yellow!50}TO   & \cellcolor{blue!25}S    & \cellcolor{blue!25}S    & \cellcolor{blue!25}S    & \cellcolor{blue!25}S    & \cellcolor{blue!25}S    & \cellcolor{blue!25}S    & \cellcolor{blue!25}S    & \cellcolor{yellow!50}TO   & \cellcolor{yellow!50}TO   \\ \hline
Neurify                                & \cellcolor{blue!25}S    & \cellcolor{blue!25}S    & \cellcolor{blue!25}S    & \cellcolor{blue!25}S    & \cellcolor{blue!25}S    & \cellcolor{blue!25}S    & \cellcolor{blue!25}S    & \cellcolor{blue!25}S    & \cellcolor{blue!25}S    & \cellcolor{blue!25}S    & \cellcolor{blue!25}S    & \cellcolor{blue!25}S    & \cellcolor{blue!25}S    & \cellcolor{blue!25}S    & \cellcolor{blue!25}S    & \cellcolor{blue!25}S    & \cellcolor{blue!25}S    & \cellcolor{blue!25}S    & \cellcolor{blue!25}S    & \cellcolor{blue!25}S    & \cellcolor{blue!25}S    & \cellcolor{blue!25}S    \\ \hline
Ours \textless{}28,   4\textgreater{}   & \cellcolor{blue!25}S    & \cellcolor{blue!25}S    & \cellcolor{blue!25}S    & \cellcolor{blue!25}S    & \cellcolor{blue!25}S    & \cellcolor{blue!25}S    & \cellcolor{blue!25}S    & \cellcolor{blue!25}S    & \cellcolor{blue!25}S    & \cellcolor{blue!25}S    & \cellcolor{blue!25}S    & \cellcolor{blue!25}S    & \cellcolor{blue!25}S    & \cellcolor{blue!25}S    & \cellcolor{blue!25}S    & \cellcolor{blue!25}S    & \cellcolor{blue!25}S    & \cellcolor{blue!25}S    & \cellcolor{blue!25}S    & \cellcolor{blue!25}S    & \cellcolor{blue!25}S    & \cellcolor{blue!25}S    \\ \hline
\end{tabular}}
\caption{}
\end{subtable}
\caption{Verification outcomes of Marabou, Neurify and our methodology on the AcasXu benchmarks for property 1. For reason of space, we split the 45 results in two subtables (a) and (b). Note that changing the parallelism of Neurify as in Figures \ref{fig:neurify-ours-27-4} and \ref{fig:neurify-st-ours-27-4} does not change its verification outcomes.}
\label{table:comparison-floating-point}
\end{table}

\begin{tcolorbox}
These results successfully answer \textbf{RQ3 - Comparison with SOTA}: We evaluated and compared our tool with other state-of-the-art tools, including SMT-based verification and symbolic intervals such as Marabou and Neurify, respectively. In terms of correctness, our approach can successfully verify all the benchmarks without timeout or crash. Furthermore, considering AcasXu property 1 benchmarks, our approach is significantly faster and solves more verification tasks than Marabou, a competitive opponent in SMT-based verification. 
\end{tcolorbox}

\subsection{Limitations}
\label{sec:limitations-v2}

We believe the work we present in this paper is an essential milestone for verifying fixed- and floating-point ANNs with arbitrary activation functions. This way, it can be considered a unified quantization framework, with the potential of broad model exploration and verification regarding data representation. However, we still want to highlight a few limitations of our verification approach that need to be addressed in future work.

First, we handle non-linear activation functions by replacing them with lookup tables (see Section~\ref{ssec:discretise}). This step is necessary for efficiency reasons but has a drawback: even with a proper resolution, a lookup table will always approximate an original function. Our experiments used lookup tables with a resolution of three decimal places and correctly validated all adversarial cases with MATLAB. However, we cannot exclude that our verification approach may produce incorrect adversarial examples or successful verification outcomes in other ANN verification scenarios, especially when a given lookup table's resolution does not match the adopted quantization granularity. This constitutes a potential threat to the validity of our method.

Second, the biggest challenge in ANN verification is scaling to large neural networks. In this regard, our Iris and Vocalic benchmarks are small to medium-sized regarding the number of neurons. Furthermore, the dataset themselves is small, which probably generated ANNs with low robustness to adversarial attacks. Both these factors contribute to keeping the dimensionality of resulting SMT formulae low and thus help our method achieve competitive verification times. However, a thorough investigation of which factors hamper verification performance and overcome them is still required.

Finally, quantized frameworks do not usually publish the code of their methods, which compromises any direct comparison attempt. Even so, the research presented here is itself SOTA and can pave the way for further research towards ANN deployment in restricted systems based on formal guarantees.

\section{Related work}
\label{sec:related_work}

This work's main contribution is providing a sound verification approach for checking the safety of MLPs with arbitrary activation functions and taking into account FWL effects in computations (weights, bias, and operations) due to fixed-point implementation, in addition to activation function discretization. SMT-based approaches~\cite{pulina2012challenging,huang2017safety,katz2017reluplex,Katz2019,Sena20} have been used for safety verification  of ANNs. Besides, the main advantage of those techniques lies in SMT solvers' soundness; however, there is an important drawback: the scalability is limited since they are sensitive to the ANN complexity. For this reason, most of them are unable to deal with large ANNs. 

Wang \etal{}~\cite{Wang2018} propose an efficient approach for checking different safety properties of large neural networks, aiming at finding adversarial cases. Their approach is based on two main ideas. First, symbolic linear relaxation combines symbolic interval analysis and linear relaxation to create an efficient propagation method for tighter estimations. Second, directed-constraint refinement, which identifies nodes whose output is overestimated and iteratively refines their output ranges. Those two techniques are implemented in a tool called Neurify that was validated against multiple ANN architectures. Furthermore, to scale up their verification framework, they have implemented their code using multi-threaded programming techniques. However, as the previous tools~\cite{pulina2012challenging,huang2017safety,katz2017reluplex}, Neurify only supports ReLU activation functions. Katz \etal{}~\cite{Katz2019} present Marabou that extends the Reluplex approach and uses lazy search to deal with nonlinearities of activation functions, allowing verification of ANNs with any piecewise-linear activation functions. 

Recently, set-theoretic methods for reachability-based verification have been proposed for verifying ANN-controlled closed-loop systems. In particular, Tran \etal{} propose the NNV tool~\cite{Tran2020}, which over-approximates the exact reachable set by approximating the exact reachable set after applying an activation function. It allows support to hyperbolic tangent and sigmoid activation functions. Other approaches~\cite{Huang2019,Ivanov2021} also employ set-theoretic methods and polynomial approximation of hyperbolic tangent and sigmoid, using Taylor's~\cite{Ivanov2021} or Bernstein's~\cite{Huang2019} polynomials. Our approach also allows verifying ANNs with non-linear activation functions. This approximation is based on lookup tables create with a suitable number of intervals (i.e., expected error) to avoid use of non-linear operators' in SMT solvers. This approach allows support to any piecewise continuous activation function.

Robustness and explainability are the core properties of the present study, and applying those properties to ANNs has shown impressive experimental results. Explainability showed a vital property to evaluate safety in ANNs: the core idea is to obtain an explanation for an adversarial case by observing the pattern activation behavior of a subset of neurons described by a given invariant. Gopinath \etal{} presented formal~\cite{Gopinath2019} and data-driven~\cite{Gopinath2018} techniques to extract properties from ANNs, which may be used as formal specifications for the ANNs. It is a crucial result to ensure explainable adversarial examples. 

Robustness is the ability to ensure safe outputs under the presence of disturbances and uncertainties, such as input noises and implementation issues~\cite{Giacobbe2020}. In this sense, Dey \etal{}~\cite{Dey2018} provide a parametric regularization methodology to improve the robustness of ANNs concerning additive noise. However, sensitivity to FWL effects is not considered in that approach. ANNs are usually designed to work in real arithmetic; however, it is already shown that safety violations may occur due to the floating-~\cite{Jia2020} and fixed-point~\cite{Giacobbe2020} implementations. In particular, Baranowski~\etal{} presented a practical SMT-based approach for verifying neural networks' properties considering fixed-point arithmetic. Their approach employs a realistic model of FWL effects that includes different rounding and overflow models. However, as shown by Henziger~\etal{}~\cite{Henzinger2020}, the scalability of this kind of approach is compromised due to the hardness of the verification of fixed-point implementations of ANNs. Therefore, a new method for verifying fixed-point implementations based on abstract interpretation is proposed in~\cite{Henzinger2020} to reduce complexity and increase scalability. However, that method can only verify ANNs with piecewise linear activation functions since it does not consider the propagation of FWL effects through generic non-linear functions. Our approach also considers FWL effects of fixed-point implementations of ANNs based on an efficient FWL implementation model that reduces complexity when verifying those ANNs. Our experiments and previous work on verification of fixed-point digital controllers~\cite{Chaves2019} indicated that scalability is not compromised by the use of this FWL implememntation model.

Our approach implemented on top of ESBMC has some similarities with other techniques described here, e.g., regarding the covering methods proposed by Sun \etal{}~\cite{Sun2019}, model checking to obtain adversarial cases proposed by Huang \etal{}\cite{huang2017safety}, and incremental verification of ANNs implemented in CUDA by Sena \etal{}~\cite{Sena20}. However, the main contribution concerns our requirements and how we handle, with invariant inference, actual implementations of ANNs with non-linear activation functions, also considering FWL effects. Moreover, the latter results in promptly deployable ANNs, which could be integrated into a unified design framework. Only ANNs' weights, bias descriptors, and desired input regarding a dataset are required to run our proposed safety verification. For tools such as DeepConcolic~\cite{Sun2019} and DLV~\cite{huang2017safety}, obtaining adversarial cases or safety guarantees in customized ANNs depends on the intrinsic characteristics of models. For instance, in their implementations, they do not support complex non-linear activation functions. Moreover, Sena \etal{}~\cite{Sena20} do not exploit invariant inference to prune the state space exploration, which is done in our proposed approach.

\section{Conclusions}
\label{sec:conclusion}

Verification of ANNs has recently attracted considerable attention, with notable approaches using optimization, reachability, and satisfiability methods. While the former two promise to scaling to large neural networks, they achieve such a goal by relaxing and approximating the verification problem. In contrast, satisfiability methods are exact by construction but are confronted with the full complexity of the original verification problem.

In this paper, we propose a satisfiability modulo theory (SMT) approach to address ANN verification. More specifically, we view the ANN not as an abstract mathematical model but as a concrete piece of software (i.e., source code), which performs a sequence of fixed- or floating-point arithmetic operations. We can borrow several techniques from software verification and seamlessly apply them to ANN verification with this view. In this regard, we center our verification framework around software model checking (SMC) and empirically show the importance of interval analysis, constant folding, tree balancing, and slicing in reducing the total verification time. Furthermore, we propose a tailored discretization technique for non-linear activation functions that allow us to verify ANNs beyond the piecewise-linear assumptions that many state-of-the-art methods are restricted to.

Besides, in our experimental evaluation, we uncovered an important relationship between the granularity of ANN quantization and verification time and the correctness of its properties. The more granular the quantization, the more significant the search space and thus the more prolonged the verification time. This is contrary to the main existing theoretical result in the literature, which states that verifying quantized ANNs is computationally harder than verifying real-valued ones. However, further research is needed to shed more light on this phenomenon. Regarding correctness, we verified that narrower bit widths can be used but must be verified before deployment to achieve the minimum format that still provides broadly correct results. However, when that minimum representation is obtained, more comprehensive formats will usually provide correct results, as the stationary response of a curve relating bit width and verification result.

We have also evaluated and compared our tool with Marabou and Neurify. Considering the ACASXu property one benchmarks~\cite{julian2016policy}, we have observed that our approach is significantly faster and solves more verification tasks than Marabou~\cite{Katz2019}, a competitive opponent in SMT-based verification. However, in many cases, Neurify~\cite{Wang2018} is faster than our tool since it deploys a multi-threaded algorithm to solve the verification tasks. However, note that neither Marabou nor Neurify can verify quantized neural networks as in our approach.

Finally, we believe that the problem of verifying ANNs is still open. More specifically, it is unclear which set of techniques yields the best performance when scaling to large networks. In this regard, our future work includes comparing our verification approach to other existing techniques and optimizing our verification performance even further. In addition, the results of our work can be regarded as the first steps towards an approach capable of revealing the most aggressive ANN representation that still provides correct operation, which aims at achieving maximum compression for a particular model.

\bibliographystyle{ACM-Reference-Format}
\bibliography{references}

\end{document}

%% file: packages_tecs_compatible.tex
\usepackage{subcaption}
\usepackage{listings}
\usepackage{comment}
\usepackage{url}
\usepackage{amsmath,amsfonts}
\usepackage{bbm}
\usepackage[noend]{algpseudocode}
\usepackage{algorithm}
\usepackage{tcolorbox}
\usepackage{tikz}
\usepackage{widetext}
\usetikzlibrary{matrix,chains,positioning,decorations.pathreplacing,arrows}
\usepackage{pgfplots}
\pgfplotsset{compat=1.10}
\usetikzlibrary{shapes.geometric,arrows,fit,matrix,positioning, shapes, calc}
\usepackage{ulem}
\usepackage{multirow}
\pgfarrowsdeclarecombine{twotriang}{twotriang}
{stealth'}{stealth'}{stealth'}{stealth'}

%% file: macros_tecs_compatible.tex
\newcommand{\Ignore}[1]{}

\newcommand{\ann}{ANN}

\newcommand{\etal}{{\it et al.}}

\newtheorem{theorem}{Theorem}[section]